\documentclass[11pt]{article}

\usepackage[utf8]{inputenc} % allow utf-8 input
\usepackage[T1]{fontenc}    % use 8-bit T1 fonts
\usepackage{url}            % simple URL typesetting
\usepackage{booktabs}       % professional-quality tables
\usepackage{amsfonts}       % blackboard math symbols
\usepackage{nicefrac}       % compact symbols for 1/2, etc.
\usepackage{microtype}      % microtypography
\usepackage{xcolor}         % colors
\usepackage{mathtools}
\usepackage{bm}
\usepackage{dsfont}

\newcommand{\newtext}[1]{{#1}}
\usepackage{amsthm}
\usepackage{amsmath}

\newtheorem{theorem}{Theorem}[section]
\newtheorem*{theorem*}{Theorem}

\newtheorem{proposition}[theorem]{Proposition}
\newtheorem*{proposition*}{Proposition}
\newtheorem{lemma}[theorem]{Lemma}
\newtheorem*{lemma*}{Lemma}
\newtheorem{corollary}[theorem]{Corollary}
\newtheorem*{conjecture*}{Conjecture}
\newtheorem{fact}[theorem]{Fact}
\newtheorem*{fact*}{Fact}

\newtheorem*{hypothesis*}{Hypothesis}

\newtheorem{claim}[theorem]{Claim}
\newtheorem*{claim*}{Claim}

\theoremstyle{definition}
\newtheorem{definition}[theorem]{Definition}

\theoremstyle{remark}

\newtheorem*{remark*}{Remark}

\newtheorem*{observation*}{Observation}

\newcommand{\eat}[1]{}

\newcommand{\R}{\mathbb{R}}
\newcommand{\N}{\mathbb{N}}

\newcommand{\calN}{\mathcal{N}}

\newcommand{\bbS}{\mathbb{S}}

\newcommand{\poly}{\mathrm{poly}}

 \newcommand{\set}[1]{\{ #1 \} }

\newcommand{\norm}[1]{\lVert #1 \rVert}

\newcommand{\Bignorm}[1]{\Big\lVert#1\Big\rVert}
\newcommand{\iprod}[1]{\langle#1\rangle}

\newcommand{\Esymb}{\mathbb{E}}
\newcommand{\Psymb}{\mathbb{P}}

\DeclareMathOperator*{\E}{\Esymb}

\DeclareMathOperator*{\ProbOp}{\Psymb}

\DeclareMathOperator*{\argmax}{argmax}

\renewcommand{\Pr}{\ProbOp}

\newcommand{\diag}{\text{diag}}

\newcommand{\sing}{s}
\newcommand{\flatten}{\text{flatten}}
\newcommand{\rank}{\mathsf{rank}}

\newcommand{\smin}{\sing_{\min}}

\newcommand{\eps}{\varepsilon}
\renewcommand{\epsilon}{\varepsilon}

\newcommand{\epsjenn}{\epsilon_{\ref{thm:jennrich}}}
\newcommand{\etajenn}{\eta_{\ref{thm:jennrich}}}

\newif\ifnotes\notesfalse

\ifnotes
\usepackage{color}
\definecolor{mygrey}{gray}{0.50}
\newcommand{\notename}[2]{{\textcolor{blue}{\footnotesize{\bf (#1:} {#2}{\bf ) }}}}

\else

\newcommand{\notename}[2]{{}}

\fi

\newcommand{\anote}[1]{{\notename{Aravindan}{#1}}}

\usepackage{thmtools}
\usepackage{thm-restate}

\usepackage{fullpage}
\usepackage[colorlinks=true,linkcolor=blue, backref=page]{hyperref}
\usepackage[ruled,vlined]{algorithm2e}

\title{Efficient Algorithms for Learning Depth-2 Neural Networks with General ReLU Activations}

\author{Pranjal Awasthi\\ \small{Google Research}\\ \small{pranjalawasthi@google.com}  \and Alex Tang\footnotemark[1]\\ \small{Northwestern University}\\ \small{alextang@u.northwestern.edu} \and Aravindan Vijayaraghavan\thanks{The last two authors are supported by the National Science Foundation (NSF) under Grant No.~CCF-1652491, CCF-1637585 and CCF 1934931. }
 \\ \small{Northwestern University} \\ \small{aravindv@northwestern.edu}}
\date{}

\begin{document}

\maketitle

\begin{abstract}
We present polynomial time and sample efficient algorithms for learning an unknown depth-2 feedforward neural network with {\em general} ReLU activations, under mild non-degeneracy assumptions. In particular, we consider learning an unknown network of the form $f(x) = {a}^{\mathsf{T}}\sigma({W}^\mathsf{T}x+b)$, where $x$ is drawn from the Gaussian distribution, and $\sigma(t) \coloneqq \max(t,0)$ is the ReLU activation. Prior works for learning networks with ReLU activations assume that the bias $b$ is zero. 
% have provided algorithms for learning the network parameters under the assumption that the bias ($b$) is zero, or that the activation functions satisfy certain smoothness assumptions. 
In order to deal with the presence of the bias terms, our proposed algorithm consists of robustly decomposing multiple higher order tensors arising from the Hermite expansion of the function $f(x)$. Using these ideas we also establish identifiability of the network parameters under minimal assumptions. 
\end{abstract}

% \anote{1. Abstract+intro - 2 pages. 
% 2. Prelimns - 1 page 
% 3. Related work - 0.75 page
% 4. Main result - 1.25\\
% 5. Non-robust - 2.5 pages
% 6. Robust analysis + Conclusion- 1.5 page. 
% }

% \anote{TODO for Alex: \\
% 0. Make notation consistent:
% 1. Use $\tilde{w}$ for all estimates instead of $\widehat{w}$. 
% 2. Use $\mathcal{N}$ instead of $N$ for Gaussian.
% 3. Model defined multiple times. Remove repeated mentions. Leave a note when you delete/ comment out some portion. 
% }
% !TEX root=main.tex

\section{Introduction} \label{sec:intro}
The empirical success of deep learning in recent years has led to a flurry of recent works exploring various theoretical aspects of deep learning such as learning, optimization and generalization. A fundamental question in the theory of deep learning is to identify conditions under which one can design provably time-efficient and sample-efficient learning algorithms for neural networks. Perhaps surprisingly, even for the simplest case of a depth-2 feedforward neural network, the learning question remains unresolved. In this work we make progress on this front by studying the problem of learning an unknown neural network of the form
\begin{equation}
    y = f(x) = a^{\mathsf{T}}\sigma(W^\mathsf{T}x+b).
\end{equation}
We are given access to a finite amount of samples of the form $(x_i,y_i)$ drawn i.i.d. from the data distribution, where each $x_i$ is comes from the standard Gaussian distribution $\mathcal{N}(0,I)$, and $y_i = f(x_i)$. The goal is to design an algorithm that outputs an approximation of the function $f$ up to an arbitrary error measured in the expected mean squared metric (squared $\ell_2$ loss).  An efficient learning algorithm has running time and sample complexity that are polynomial in the different problem parameters such as the input dimensionality, number of hidden units, and the desired error. 

Without any further assumptions on the depth-2 network, efficient learning algorithms are unlikely. %, even for a depth-2 network is hard as shown in a 
The recent work of \cite{diakonikolas2020algorithms} provides evidence by proving exponential statistical query lower bounds (even when $x$ is Gaussian) that rule out a broad class of algorithms. 
% Furthermore, the lower bound continues to hold even if the data distribution generating the input $x$ is the standard Gaussian distribution. 

Several recent works have designed efficient algorithms for depth-2 neural networks in the special setting when the bias term $b=0$. One prominent line of work \cite{ge2017learning, bakshi2019learning} give polynomial time algorithms under the non-degeneracy assumption that the matrix $W$ has full-rank. Another body of work relaxes the full-rank assumption by designing algorithm that incur an exponential dependence on the number of hidden units \cite{diakonikolas2020algorithms, chen2020learning, goel2019learning}, or a quasipolynomial dependence when the coefficients $\set{a_i: i \in [m]}$ are all non-negative~\cite{DKane20}. There is little existing literature on learning neural networks in the presence of the bias term. A notable exception is an approach based on computing ``score functions'' \cite{janzamin2015beating} that applies to certain activations with bias and requires various assumptions that are not satisfied by the ReLU function. The diminished expressivity of neural networks without the bias terms leads to the following compelling question:

%\vspace{2pt}
{\em Can we design polynomial time algorithms even in the presence of bias terms in the ReLU units?}
%{\em Can we achieve computational tractability, even in the presence of bias terms in the ReLU units?}
% while removing the assumption of no bias, that is inherent in most recent works?
%The primary technical challenge that we address in this work is removing the assumption of no bias, that is inherent in most recent works on the topic. 

%identify conditions under which efficient learnability of depth-2 networks is possible. In particular, these works design polynomial time learning algorithms under the assumption that the input distribution is Gaussian and that the matrix $W$ is full rank. 
%The results can be extended to the class of symmetric distributions under further assumptions \cite{ge2018learning}. 
%In addition to the assumptions on the data distribution and the non-degeneracy of $W$, all the works discussed above design algorithms under the assumption that the bias $b$ of the unknown ground truth network is zero. 
% There is little existing literature on learning neural networks in the presence of the bias term. A notable exception is an approach based on computing ``score functions'' \cite{janzamin2015beating} that applies to certain activations with bias and requires various assumptions that are not satisfied by the ReLU function. 
% Hence, dealing with the presence of the bias terms when learning depth-2 networks remains a challenge. 

We answer the question affirmatively by designing efficient algorithms for learning depth-2 neural networks with general ReLU activations, under the assumption that $W_{d \times m}$ has linearly independent columns (hence $m \le d$). In fact, our algorithms can be extended to work under much weaker assumptions on $W$, that allow for $m \gg d$ ($m \le O(d^{\ell})$ for any constant $\ell\ge 1$) in a natural smoothed analysis setting considered in prior works~\cite{ge2018learning} (see Theorem~\ref{thm:robust:higherorder} and Corollary~\ref{corr:smoothed}). 
\newtext{An important consequence of our techniques is the fact that the network parameters are identifiable up to signs, as long as no two columns of $W$ are parallel,  and all the $\set{a_i: i \in [m]}$ are non-zero. Furthermore we show that this ambiguity in recovering the signs is inherent unless stronger assumptions are made.}

% In this work we address the above question by designing an efficient learning algorithm for depth-2 feedforward networks in the presence of the bias terms. As in prior works we assume that the input data distribution is the standard Gaussian distribution and that the matrix $W$ has full rank. Furthermore, we extend our results to the much weaker setting where only linear independence among higher order tensors of the columns of $W$ is required. Such conditions are easily satisfied in a natural smoothed analysis model that has also been considered in prior works \cite{ge2018learning}.

\vspace{5pt}

\noindent \textbf{Conceptual and technical challenges with bias terms.} Similar to prior works \cite{ge2018learning, janzamin2015beating}, our techniques rely on the use of tensor decomposition algorithms to recover the parameters of the network. %However, the presence of the bias terms presents multiple challenges.
In the absence of any bias, it can be shown that the $4$th Hermite coefficient of the function $f(x)$ takes the form $\hat{f}_4 = \sum_{i=1} a_i w^{\otimes 4}_i$ where $w_i$ are the columns of $W$ and $a = (a_1, a_2, \dots, )$. When the columns of $W$ are linearly independent, existing algorithms for tensor decompositions in the full-rank setting can be used to recover the parameters~\cite{Harshman1970}.\footnote{Tensor decompositions will in fact recover each ReLU activation up to an ambiguity in the sign. However, in the full-rank setting, the correct sign can also be recovered (as we demonstrate later in Theorem~\ref{thm:full-rank}). } %In other words, all the parameters of the unknown can be recovered up to any desired accuracy. 
However, when bias terms are present, there are several challenges that we highlight below.

In the 
presence of biases the $k$th Hermite expansion of $f(x)$ takes the form $\hat{f}_k = \sum_{i=1}^m a_i g_k(b_i) w^{\otimes k}_i$, where $g_k$ is a function that may vanish on some of the unknown $b_i$ parameters. This creates a hurdle in recovering the corresponding $w_i$. A simple example where the above approach fails is when some of the $b_i = \pm 1 $, since the corresponding rank-1 terms vanish from the decomposition of $\hat{f}_4$. 
To overcome this obstacle, we first show give a precise expression for the function $g_k(b_i)$ involving the $(k-2)$th Hermite polynomial (see Lemma~\ref{lem:hermitecoeffs}). We then design an algorithm that decomposes multiple tensors obtained from Hermite coefficients %, and decomposes these tensors 
to recover the parameters. We use various properties of Hermite polynomials to analyze the algorithm e.g.,  the separation of roots of consecutive Hermite polynomials is used to argue that each $w_i$ is recovered from decomposing at least one of the tensors. 
% The analysis becomes particularly challenging when there are sampling errors due to using only a polynomial amount of data. 

% Hence one needs to consider the case when $g_k(b_i)$ is zero, and as a result the corresponding term cannot be recovered by decomposing $\hat{f}_k$. 
% In order to overcome this we prove, via utilizing the separation of the roots of consecutive Hermite polynomials, that $g_k(b_i)$ cannot be zero for two consecutive tensors of order $k$ and $k+1$. We then use the claim to analyze multiple higher order tensors and show that each $b_i$ of small magnitude can be recovered from at least one of them. 

Secondly, in the presence of the bias terms, recovery of all the parameters (even up to sign ambiguities) may not even be possible from polynomially many samples. For instance, consider a particular hidden node with output $\sigma(w^\top_i x + b_i)$. If $b_i$ is a large positive number then it behaves like a linear function (always active). Hence if multiple $b_i$s are large positive numbers then one can only hope to recover a linear combination of their corresponding weights and biases. On the other hand if $b_i$ is a large negative constant then the activation is $0$ except with probability exponentially small in $|b_i|$. We cannot afford a sample complexity that is exponential in the magnitude of the parameters. \newtext{Furthermore, when the columns of $W$ are not linearly independent, the tensor decomposition based method will only recover good approximations up to a sign ambiguity for the terms whose bias does not have very large magnitude i.e., we recover $\pm(w_i^\top x+b_i)$ if $|b_i|$ is not very large.}

\newtext{To handle the above issue we proceed in two stages. In the first stage we recover the network parameters (up to signs) of all the ``good'' terms, i.e., hidden units with biases of small magnitude. To handle the ``bad'' terms (large magnitude bias) we show that a linear functions is a good approximation to the residual function comprising of the bad terms. Based on the above, we show that one can solve a truncated linear regression problem to learn a function $g(x)$ that achieves low mean squared error with respect to the target $f(x)$. The output function $g(x)$ is also a depth-2 ReLU network with at most two additional hidden units than the target network.}

There are several other technical challenges that arise in the analysis sketched above, when there are sampling errors due to using only a polynomial amount of data (for example, the tensors obtained from $\hat{f}_k$ may have some rank-$1$ terms that are small but not negligible, that may affect the robust recovery guarantees for tensor decompositions). We obtain our robust guarantees by leveraging many useful properties of Hermite polynomials, and a careful analysis of how the errors propagate.

%As a result we need to establish that the case of large magnitude biases does not hurt the learning algorithm too much since the corresponding part of the network is a simple function (linear or zero). 

The rest of the paper is organized as follows. We present preliminaries in Section~\ref{sec:prelims} followed by related work in Section~\ref{sec:related}. We then formally present and discuss our main results in Section~\ref{sec:main-results}. 
In Section~\ref{sec:non-robust} we present our main algorithm and analysis in the population setting, i.e., under the assumption that one has access to infinite data from the distribution. We then present the finite sample extension of our algorithm in Section~\ref{sec:robust} that achieves polynomial runtime and sample complexity. 

% \begin{itemize}
%     \item Problem Setup and motivation.
%     \item Hardness without distributional and/or degeneracy assumptions.
%     \item Existing Poly time algorithms (for ReLU activations) and the assumptions they make. Also mention the FPT style guarantees (which also assume zero bias). And maybe just assert that handling general RELU seems like a challenge. 
%     \item Can be done for other activations that satisfy certain properties (Anandkumar paper).
%     \item Explaining why handling general RELU (removing the need for zero  bias assumptions) is a challenge. 
    
%     \item Our Contribution.
% \end{itemize}

\subsection{Model Setup and Preliminaries}
\label{sec:prelims}

We consider the supervised learning problem with input $x \in \mathbb{R}^d$ drawn from a standard $d$-dimensional Gaussian distribution $\calN(0,I_{d \times d})$ and labels $y$ generated by a neural network $ y = f(x) = a^{\mathsf{T}}\sigma(W^\mathsf{T}x+b)$,
%
% \begin{equation}
%     y = f(x) = a^{\mathsf{T}}\sigma(W^\mathsf{T}x+b)
% \end{equation}
%
where $a, b \in \mathbb{R}^m$, $W \in \mathbb{R}^{d \times m}$ and $\sigma$ is the element-wise ReLU activation function, i.e., $\sigma(t) = \max(t,0)$. 
%\anote{Add RELU}
We denote the column vectors of $W$ as $w_i \in \mathbb{R}^d$ with $i \in [m]$ and $a_i$ as the $i$'th element of vector $a$, similarly for $b$ and $x$. We pose a constraint on magnitudes of $a, b, W$ such that they are all $B$-bounded for some $1 \le B \leq \poly(m, d)$, i.e. $\norm{a}_\infty, \norm{b}_\infty, \norm{W}_\infty \le B$, and $\min_{i \in [m]} |a_i|\ge 1/B$.
Furthermore, we assume $\|w_i\|_2 = 1$ without loss of generality. If $w_i$ are not unit vectors, we can always scale $a_i$ and $b_i$ to $\|w_i\|a_i$ and $\frac{b_i}{\|w_i\|}$ respectively so that $w_i$ are normalized. We will denote by $\Phi()$ the cumulative density function~(CDF) of the standard Gaussian distribution. Finally, for a matrix $M$, we will use $\sing_k(M)$ to denote the $k$th largest singular value of $M$. 

For some parameters $n_1, n_2, \dots$ we will say that a quantity $q \le \poly(n_1, n_2, \dots)$ if and only if there exists constants $c_0>0, c_1>0, c_2>0, \dots$ such that  $q \le c_0 n_1^{c_1} n_2^{c_2} \dots $. If these constant depend on another parameter $\ell>0$ which is also a constant, then we will denote this by $\poly_\ell(n_1, n_2, \dots)$. We will say that an probabilistic event occurs with high probability if and if it occurs with probability $1-(mdB)^{-\omega(1)}$ i.e., the failure probability is smaller that {\em any} inverse polynomial in the parameters $m,d,B$. %(often this failure probability can be driven down to any desired inverse polynomial in the parameters).  
\newtext{Finally, we will use sign variables of the form $\xi_i \in \set{1,-1}$; they will typically capture an ambiguity in the sign of the parameters of the $i$th unit. }

%\anote{I removed the line on purpose. This should be stated as part of the theorem; in general we can take higher orders, and handle larger $m$.}

%\subsection{Preliminaries}

%\anote{Notation: Add that $\sing_k(M)$ stands for $k$th largest singular value of $M$. Edit all the occurrences of $\sigma_m$ for singular values.}

%\anote{Add something about RELU gates. The model and related notation.}

%\vspace{2pt}
%\noindent {\bf Hermite polynomials and Hermite coefficients.} 

\paragraph{Hermite Polynomials and Hermite Coefficients}

The $d$-dimensional or $d$-variate Hermite polynomials of the $k$th order evaluated at a point $x$ represented by the tensor $He_k(x) \in (\mathbb{R}^d)^{\otimes k}$ with $k \in \mathbb{N}$ (here $\mathbb{N}$ includes $0$) %\anote{I've been using $\mathbb{N}$ to include 0. Peano's axioms start with 0 :). } 
form a complete orthogonal basis for the weighted $L^2(\mathbb{R}^d, \gamma)$ space with inner product defined as $\mathbb{E}[f(x)g(x)] = \int_{\mathbb{R}^d} f(x)g(x) \gamma(x)dx$ for two functions $f, g$, where $\gamma(x) = \exp(-\|x\|^2/2)/\sqrt{2\pi}$. We can thus define the $k$'th Hermite coefficient of $f(x)$ by $\hat{f}_k = \mathbb{E}[f(x)He_k(x)]$, similar to how the Fourier coefficients are constructed. Throughout the context we wll also use $He_k^{\alpha}(x) \in \mathbb{R}$ to denote a specific entry of $He_k(x)$ with $\alpha \in [d]^k$ as a multi-index if $d > 1$. We now briefly introduce the definition of Hermite polynomials along with some useful facts below.
%More details and facts about Hermite polynomials %along with some useful facts are included in Appendix~\ref{app:hermite}.

\begin{definition}
Let $D_t^{(k)}$ be the total differential operator taken $k$ times with respect to $t$. For a function $g: \mathbb{R}^d \rightarrow \mathbb{R}$, $D_t^{(k)}g(t)|_{t = t_0} \in (\mathbb{R}^d)^{\otimes k}$, where the $\alpha$'th element of $D_t^{(k)}g(t)$, for $\alpha = (i_1, i_2, \dots, i_k)$, is $\frac{d}{dt_{i_1}} ... \frac{d}{dt_{i_k}}g(t)$ with $\alpha \in [d]^k$ being a multi-index. Note that the above is invariant to permutations, i.e., for any permutation $\alpha'$ of the indices in $\alpha$, the $\alpha'$th element of $D_t^{(k)}g(t)$ is the same as the $\alpha$th element.
\end{definition}

\begin{definition}
\label{def-a1}
Let $x \in \mathbb{R}^d$, $\gamma(x) = \exp(-\frac{\|x\|^2}{2})$, the (probabilist's) $k$'th $d$-dimensional Hermite polynomial $He_k(x) \in (\mathbb{R}^d)^{\otimes k}$ is given by

\begin{equation}
    He_{k}(x) = \frac{(-1)^k}{\gamma(x)} \cdot D_x^{(k)} \gamma(x)
\end{equation}
\end{definition}

A particularly useful fact for 1-dimensional Hermite polynomials is their relation with derivatives of a standard univariate Gaussian function.
\begin{fact}
\label{lem:gaussian-derivative}
The $k$'th order derivative of $\gamma(x)$ can be written in terms of 1-dimensional Hermite polynomials as
    
\begin{equation}
    \frac{d^k}{dx^k} \gamma(x) = (-1)^k \cdot He_k(x) \cdot \gamma(x) = He_k(-x) \gamma(x)
\end{equation}

\end{fact}
We will utilize this fact to express the Hermite coefficients of $f(x) = a^{\mathsf{T}}\sigma(W^\mathsf{T}x+b)$. Finally, for 1-dimensional Hermite polynomials $He_k(x)$, and a sign variable $\xi \in \set{\pm 1}$, we have $He_k(\xi x) = \xi^k He_k(x)$ from the odd/even function properties of Hermite polynomials. %this is an odd function for odd $k$, and an even function for even $k$. 

%Hence $\forall f \in L^2(\mathbb{R}^d, \gamma)$ we can express $f$ as a linear combination of the Hermite polynomials
%$f = \sum_{k \in \mathbb{N}} \widehat{f}_k \frac{He_k(x)}{\|He_k(x)\|_F}$,
%where $\widehat{f}_k = \mathbb{E}[f(x)He_k(x)]$ is the $k$'th Hermite coefficient of $f$, and $\|He_k(x)\|_F$ is the Frobenius norm of $He_k(x)$.

%A particularly useful fact for 1-dimensional Hermite polynomials is their relation with derivatives of a standard univariate Gaussian function.
%\begin{fact}
%\label{lem:gaussian-derivative}
%The $k$'th order derivative of $\gamma(x)$ can be written in terms of 1-dimensional Hermite polynomials as
    
%\begin{equation}
    %\frac{d^k}{dx^k} \gamma(x) = (-1)^k \cdot He_k(x) \cdot \gamma(x) = He_k(-x) \gamma(x)
%\end{equation}

%\end{fact}
%We will utilize this fact to express the Hermite coefficients of $f(x) = a^{\mathsf{T}}\sigma(W^\mathsf{T}x+b)$. 

%\paragraph{Tensor Decomposition}
%\anote{State the main tensor theorem we will use about robust efficient decomposition} 
%
%We will consider tensors of order %$3$ and above. 
%\noindent {\bf Tensor decomposition.}

\paragraph{Tensor Decomposition}

The tensor product $u \otimes v \otimes w \in \R^{d_1} \otimes \R^{d_2} \otimes \R^{d_3}$ of vectors $u \in \R^{d_1},v \in \R^{d_2},w\in \R^{d_3}$ is a rank-$1$ tensor. Similarly, we will use $u^{\otimes \ell} \in (\R^{d})^{\otimes \ell}$ to denote the tensor product of $u$ with itself $\ell$ times.
An order-$t$ tensor $T \in \R^{d_1}\otimes \R^{d_2}\otimes \dots \otimes \R^{d_t}$ is represented using a $t$-way array $\R^{d_1 \times d_2 \times \dots \times d_t}$ that has $t$ modes corresponding to the $t$ different indices. Given two matrices $U,V$ with $k$ columns each given by $U=(u_i : i \in [k])$ and $V=(v_u : i \in [k])$, the {\em Khatri-Rao product} $M=U \odot V$ is a matrix formed by the $i$th column being $u_i \otimes v_i$. We will also use $U^{\odot 2}= U \odot U$ (and similarly for higher orders). 
A claim about preserving the full-column-rank property (and analogously minimum singular value) under the Khatri-Rao product is included below. %the Appendix~\ref{app:tensordecomp}.
 
 {\em Flattening or Reshaping: } %An order-$t$ tensor $T \in (\R^{d})^{\otimes t}$ can be flattened to get a order-$3$ tensor $T'$ by combining different modes. In particular, 
 Given an order-$t$ tensor $T$, for $t_1, t_2, t_3\ge 0$ such that $t_1+t_2+t_3=t$  define $T'=\flatten(T,t_1,t_2,t_3)$ as the order-$3$ tensor $T' \in \R^{d^{t_1} \times d^{t_2} \times d^{t_3}}$, obtained by flattening and combining the first $t_1$ modes, the next $t_2$ and last $t_3$ modes respectively. %of $T$ to get the first mode of $T'$, the next $t_2$ modes of $T$ to get the second mode of $T'$ and the last $t_3$ modes of $T$ to get the 3rd mode for $T'$. 
 When $t_3=0$, the output is a matrix in $\R^{d^{t_1} \times d^{t_2}}$. 

Tensor decompositions of order $3$ and above, unlike matrix decompositions (which are of order $2$) are known to be unique under mild conditions.
While tensor decompositions are NP-hard in the worst-case, polynomial time algorithms for tensor decompositions are known under certain non-degeneracy conditions (see e.g.,~\cite{Anandkumarbook, AVbookchapter}). %These algorithmic guarantees and uniqueness hold even when there is some noise in the tensor, measured in Frobenius norm. 
%We use the following robust guarantee for Jennrich's algorithm~\cite{Harshman1970} (see also \cite{GVX14, Moitrabook} for robust analysis). 
In particular, Jennrich's algorithm~\cite{Harshman1970} provides polynomial time guarantees for recovering all the rank-$1$ terms of a decomposition of a tensor $T=\sum_{i=1}^k u_i \otimes v_i \otimes w_i$, when the $\set{u_i : i \in [k]}$ are linearly independent, the $\set{v_i: i \in [k]}$ are linearly independent, and no two of the vectors $\set{w_i: i \in [k]}$ are parallel.  
This algorithm and its guarantee can also be made robust to some noise (of an inverse polynomial magnitude), when measured in Frobenius norm. %See Appendix~\ref{app:tensordecomp} for details.
In this paper, the following claims are especially vital to formulate our main results. 

\begin{claim}[Implication of Lemma A.4 of \cite{BCV14}\footnote{In fact one can show that a stronger statement that a certain quantity called Kruskal-rank increases, see \cite{BCV14}. }]
\label{claim:khatri-rao}
Let $U \in \R^{d_1 \times k}, V \in \R^{d_2 \times k}$, and suppose the smallest column length $\min_{i \in [k]}\norm{v_i}_2 \ge \kappa$. Then the Khatri-Rao product $U \odot V$ has rank $k$ and satisfies 
$\sing_k(U \odot V) \ge \kappa \cdot \sing_k(U) /\sqrt{2k}$.
\end{claim}

A robust guarantee we will use for Jennrich's algorithm~\cite{Harshman1970} is given below (see also \cite{GVX14, Moitrabook} for robust analysis).
\begin{theorem}[Theorem 2.3 of \cite{BCMV14}] \label{thm:jennrich}
Suppose $\epsjenn>0$ we are given tensor $\widetilde{T}=T + E \in \R^{m \times n \times p}$, where $T$ has a decomposition
$T = \sum_{i = 1}^k u_i \otimes v_i \otimes w_i$ satisfying the following conditions:
\begin{enumerate}
\item Matrices $U=(u_i: i \in [k]), V=(v_i: i \in [k])$ have condition number (ratio of the maximum singular value $\sigma_1$ to the least singular value $\sigma_k$) at most $\kappa$, 
\item For all $i \neq j$, the submatrix $W_{\set{i,j}}$ has $\sing_2(W_{\set{i,j}}) \ge \delta$. %$\|\frac{w_i}{\|w_i\|} - \frac{w_j}{\|w_j\|} \|_2 \geq \delta$ and $\|\frac{w_i}{\|w_i\|} + \frac{w_j}{\|w_j\|} \|_2 \geq \delta$.
\item Each entry of $E$ is bounded by $ \etajenn(\eps,\kappa,\max\set{n,m,p}, \delta)=  \frac{\poly(\eps)}{\mathsf{poly}(\kappa, \max\set{n,m,p}, \tfrac{1}{\delta})}$. 
\end{enumerate}
Then there exists a polynomial time algorithm that on input $\widetilde{T}$ returns a decomposition $\set{(\widetilde{u}_i, \widetilde{v}_i, \widetilde{w}_i): i \in [k]}$ s.t. there is a permutation $p:[k] \to [k]$ with
\begin{equation}
    \forall i \in [k],~~ \|\widetilde{u}_i \otimes \widetilde{v}_i \otimes \widetilde{w}_i - u_{p(i)} \otimes v_{p(i)} \otimes w_{p(i)} \|_F  \le \epsjenn.
\end{equation} 
%{\em each rank one term} in the decomposition of $T$  (up to renaming), within an additive error of $\epsilon$.
\end{theorem}

\vspace{-10pt}

\section{Related Work}
\label{sec:related}
By now there is a vast literature exploring various aspects of deep learning from a theoretical perspective. Here we discuss the works most relevant in the context of our results. As discussed earlier, the recent works of \cite{ge2017learning, bakshi2019learning, ge2018learning} provide polynomial time algorithms for learning depth-2 feedforward ReLU networks under the assumption that the input distribution is Gaussian and that the matrix $W$ is full rank. Some of these works consider a setting where the output is also a high dimensional vector \cite{bakshi2019learning, ge2018learning}, and also consider learning beyond the Gaussian distribution.
However, these works do not extend to the case of non-zero bias. 

The work of \cite{janzamin2015beating} proposed a general approach based on tensor decompositions for learning an unknown depth-2 neural network that could also handle the presence of the bias terms. %They proposed an approach that also involves the use of tensor decompositions. 
The tensor used in the work of~\cite{janzamin2015beating} is formed by taking the weighted average of a ``score'' function evaluated on each data point. In this  way their approach generalizes to a large class of distributions provided one has access to the score function. However, for most data distributions computing the score function is a hard task itself. 
When the input distribution is Gaussian, then the score functions correspond to the Hermite coefficients of the target function $f(x)$ and can be evaluated efficiently. However, the analysis in~\cite{janzamin2015beating} does not extend to the case of ReLU activations for several reasons. %The authors in~\cite{janzamin2015beating} show the applicability of their 
Their technique needs certain smoothness and symmetry assumptions on the activations that do not hold for ReLU. These assumptions also ensure that all the terms in the appropriate tensor are non-zero. We do not make such assumptions, and tackle one of the main challenges by showing that one can indeed recover a good approximation to the network by analyzing multiple higher order tensors. Furthermore, the authors in~\cite{janzamin2015beating} assume that the biases, and the spectral norm of $W$ are both bounded by a constant. %$1$ in magnitude and that the spectral norm of the $W$ matrix is bounded by a constant. 
As a result they do not handle the case of biases of large magnitude where some of the ReLU units mostly function as linear functions (with high probability).  

There have also been works on designing learning algorithms for neural networks without assumptions on the linear independence of columns of $W$. % or linear independence of higher order tensors of the columns of $W$. 
These results incur an exponential dependence on either the input dimensionality or the number of parameters in the unknown network \cite{diakonikolas2020algorithms, chen2020learning}, or quasipolynomial dependence when the coefficients $\set{a_i: i \in [m]}$ are all non-negative~\cite{DKane20}. In particular, the result of \cite{chen2020learning} provides a learning algorithm for arbitrary depth neural networks under the Gaussian distribution with an ``FPT'' guarantee; its running time is polynomial in the dimension, but  exponential in the number of ReLU units. % in the network. 
Given the recent statistical query lower bounds on learning deep neural networks~\cite{diakonikolas2020algorithms, goel2020superpolynomial}, getting a fully polynomial time algorithm without any assumptions is a challenging open problem, even under Gaussian marginals.

Polynomial time algorithms with fewer assumptions and beyond depth-2 can be designed if the activation functions in the first hidden layer are sigmoid functions~\cite{goel2019learning}. Finally, there is also extensive literature on analyzing the convergence properties of gradient descent and stochastic gradient descent for neural networks. The results in this setting implicitly or explicitly assume that the target function is well approximated in the Neural Tangent Kernel (NTK) space of an unknown network. Under this assumption these results show that gradient descent on massively overparameterized neural networks can learn the target~\cite{andoni2014learning, daniely2017sgd, daniely2016toward, allen2019convergence, du2018gradient, arora2019fine, lee2019wide, chizat2018global, jacot2018neural}. 

\section{Main Results}
\label{sec:main-results}
%\anote{Use $\calN$ for normal distribution.}

%The target ReLU network is given by $f(x)=a^\top \sigma(W^\top x+b)= \sum_{i=1}^m a_i \sigma(w_i^\top x+b_i)$, where $a,b \in \R^m, W \in \R^{d \times m}$, and the columns of $W$ are unit vectors i.e., for all $i \in [m], \norm{w_i}_2 =1 $.  
%We are given i.i.d. samples of the $(x,f(x)) \in \R^d \times \R$ where $x\sim \calN(0,I_{d \times d})$. We will also say that the instance is $B$-bounded for some appropriate $B>0$ (think of $B\le \poly(m,d)$) if  the parameters satisfy $\norm{a}_\infty, \norm{b}_\infty, \norm{W}_\infty \le B$, and $\min_{i \in [m]} |a_i|\ge 1/B$. Finally, one other parameter that plays a key role is $\bneg=\max_{i \in [m]} \max\{-b_i, 0\}$

There are two related but different goals that we consider in learning the ReLU network:
% \vspace{-5pt}
% \begin{itemize}%[\itemsep=0pt, \topsep=0pt]
%  \setlength{\itemsep}{1pt}
%   \setlength{\parskip}{0pt}
%   \setlength{\parsep}{0pt}
%     \item {\em Achieves low error:} Output a ReLU network $\widetilde{f}(x)=\widetilde{a}^\top \sigma(\widetilde{W}^\top x+\widetilde{b})$ such that the $L_2$ error is at most $\epsilon$  for some $\epsilon>0$ i.e.,  $\E_{ x \sim N(0,I)} [(f(x) - \widetilde{f}(x))^2] \le \epsilon^2 $.% \footnote{An approximate $\widetilde{f}$ also implies a small expected squared loss under the data distribution (see Section~\ref{app:badterms}).%, since our $\epsilon=1/\poly(m,d)$ can be chosen sufficiently small. }      }
%     \item {\em Parameter recovery:} Output $\widetilde{W}, \widetilde{a}, \widetilde{b}$, such that each parameter is $\epsilon$-close (up to permuting the $m$ co-ordinates of $a,b \in \R^m$ and reordering the corresponding columns of $W$).    
% \end{itemize}
\vspace{-5pt}
\begin{itemize}%[\itemsep=0pt, \topsep=0pt]
 \setlength{\itemsep}{1pt}
  \setlength{\parskip}{0pt}
  \setlength{\parsep}{0pt}
    \item {\em Achieves low error:} Output a ReLU network $g(x)=a'^\top \sigma(W'^\top x+b')$ such that the $L_2$ error is at most $\epsilon$ for a given $\epsilon>0$ i.e.,  $\E_{ x \sim \calN(0,I_{d \times d})} [(f(x) - g(x))^2] \le \epsilon^2 $.% \footnote{An approximate $\widetilde{f}$ also implies a small expected squared loss under the data distribution (see Section~\ref{app:badterms}).%, since our $\epsilon=1/\poly(m,d)$ can be chosen sufficiently small. }      }
    \item {\em Parameter recovery:} Output $\widetilde{W}, \widetilde{a}, \widetilde{b}$, such that each parameter is $\epsilon$-close (up to permuting the $m$ co-ordinates of $a,b \in \R^m$ and reordering the corresponding columns of $W$).    
\end{itemize}

We remark that the second goal is harder and implies the first; in particular, when $\epsilon=0$, the second goal corresponds to identifiability of the model. However in some cases, parameter recovery may be impossible to achieve (see later for some examples)
%depending on the unknown parameter $b$, 
even though we can achieve the goal of achieving low error. As we have seen earlier, given $N$ samples if $b_i \gg \sqrt{\log N}$, then $\sigma(w_i^\top x+b_i)$ will be indistinguishable from the linear function $w_i^\top x+b$ w.h.p.; hence if there are multiple such $b_i \in [m]$ with large magnitude, the best we can hope to do is recover the sum of all those linear terms. Our first result shows that this is the only obstacle when we are in the {\em full-rank} or {\em undercomplete} setting i.e., $\set{w_i:i \in [m]}$ are linearly independent (in a robust sense). 

% \anote{This theorem needs to be rephrased.}
% \begin{theorem}
% \label{thm:full-rank}[Full-rank setting]
% %Suppose $W$ has full column rank. Then 
% Given $\eps \in (0,1)$ and $N \geq \poly(m,d,1/\epsilon,1/s_m(W),B)$ samples generated by a ReLU network $f(x) = a^\top \sigma(W^\top x + b)$ that is $B$-bounded, there exists an algorithm that runs in $\poly(N,m,d)$ time and with high probability recovers $\widetilde{a}_i$, $\widetilde{b}_i$, $\widetilde{w}_i$ for all $i \in G$, where $G = \set{i \in [m] : |b_i| < c\sqrt{\log(1/\eps mdB)}}$ for some constant $c>0$, such that $\norm{w_i - \widetilde{w}_i}_2+|a_i - \widetilde{a}_i|+|b_i - \widetilde{b}_i| < \eps$. Moreover, there is an algorithm that runs in $\poly(N,m,d)$ time and with high probability finds a ReLU network $g(x)=a'^\top \sigma(W'^\top x+{b'})$ such that the $L_2$ error $\E_{x \sim \calN(0,I_{d \times d})}[(f(x)-\widetilde{f}(x))^2] \le \eps^2$ 
% \end{theorem}

\begin{theorem}[Full-rank setting]
\label{thm:full-rank}
%Suppose $W$ has full column rank. Then 
Suppose $\eps \in (0,1)$ and $N \geq \poly(m,d,1/\epsilon,1/s_m(W),B)$ samples be generated by a ReLU network $f(x) = a^\top \sigma(W^\top x + b)$ that is $B$-bounded, and $|b_i| < c\sqrt{\log(1/\eps mdB)}$ for all $i \in [m]$. Then there exists an algorithm that runs in $\poly(N,m,d)$ time and with high probability recovers $\widetilde{a}_i$, $\widetilde{b}_i$, $\widetilde{w}_i$  such that $\norm{w_i - \widetilde{w}_i}_2+|a_i - \widetilde{a}_i|+|b_i - \widetilde{b}_i| < \eps$ for all $i \in [m]$. %Moreover, there is an algorithm that runs in $\poly(N,m,d)$ time and with high probability finds a ReLU network $g(x)=a'^\top \sigma(W'^\top x+{b'})$ such that the $L_2$ error $\E_{x \sim \calN(0,I_{d \times d})}[(f(x)-\widetilde{f}(x))^2] \le \eps^2$ 
\end{theorem}

The above theorem recovers all the parameters when the biases $\set{b_i: i \in [m]}$ of each ReLU unit does not have very large magnitude. 
Moreover even when there are $b_i$ of large magnitude, we can learn a depth-$2$ ReLU network $g$ that achieves low error, and simultaneously recover parameters for the terms that have a small magnitude of $b_i$~(up to a potential ambiguity in signs).
%(there could be some potential ambiguity in signs for the parameters).  
In fact, our algorithm and guarantees are more general, and can operate under the much milder condition that $\set{w_i^{\otimes \ell}: i \in [m]}$ are linearly independent for any constant $\ell\ge 1$; the setting when $\ell >1$ corresponds to what is often called the {\em overcomplete setting}.        
In what follows, for any constant $\ell \in \mathbb{N}$ we use $\poly_\ell(n_1, n_2,\dots)$ to denote a polynomial dependency on $n_1, n_2, \dots$, and potentially exponential dependence on  $\ell$. %; hence, this is  polynomial for constant $\ell>0$.
%\begin{restatable}{thm}{main}
\begin{theorem}
\label{thm:robust:higherorder}
     Suppose $\ell \in \N$ be a constant, and $\epsilon>0$. If we are given $N$ i.i.d. samples as described above from a ReLU network $f(x)=a^\top \sigma(W^\top x+b)$ that is $B$-bounded 
 %and satisfies 
%  \begin{equation}
% (a)~~     \sigma_m(W) \ge \frac{1}{\poly(m,d)}, ~\qquad(b)~~ |a_i| \ge a_{\min} ~\forall i \in [m], ~\qquad (c) ~b_i \ge -b_{\min} ~\forall i \in [m].
%  \end{equation}
%and if $W$ satisfies $s_m(W^{\odot \ell})\ge 1/\poly(m,d)$, 
%\anote{Check if the explicit condition on $W$ can be removed.}
then there is an algorithm that given $N\ge \poly_\ell(m,d,1/\epsilon,1/s_m(W^{\odot \ell}),B)$ runs in $\poly(N,m,d)$ time and with high probability finds a ReLU network $g(x)={a'}^\top \sigma({W'}^\top x+{b'})$ with at most $m+2$ hidden units, such that the $L_2$ error $\E_{x \sim \calN(0,I_{d \times d})}[(f(x)-{g}(x))^2] \le \eps^2$. 
%i.e., $\forall x \in \R^d, ~ |f(x)-\widetilde{f}(x)| \le \epsilon \norm{x}_2$. 
 Furthermore there are constants $c=c(\ell)>0, c'>0$ and signs $\xi_i \in \set{\pm 1}~\forall i \in [m]$, such that \newtext{in $\poly(N,m,d)$ time,} for all $i \in [m]$ with $|b_i| < c \sqrt{\log(1/(\epsilon\cdot mdB))}$, we can recover $(\widetilde{a}_i, \widetilde{w}_i, \widetilde{b}_i)$, such that $|a_i - \widetilde{a}_i|+ \norm{w_i - \xi_i \widetilde{w}_i}_2+|b_i - \xi_i \widetilde{b}_i| < c'\epsilon/(mB)$. %\anote{as before, need to check the $c \sqrt{\log(1/\epsilon)}$ condition.}
\end{theorem}
%\end{restatable}
% there exists an algorithm~(Algorithm~\ref{alg-regression}) that runs time polynomial in $N$ and with probability at least $1-\delta$ outputs a network $g(x)$ of the form $g(x) = \sum_{i=1}^{m'} a'_i \sigma(w'^\top_i \cdot x + b'_i) + w''^\top x + C$, where $m' \leq 8 |S|$, such that
% $$
% \Esymb_{x \sim N(0,I)} \big[f(x) - g(x) \big]^2 \leq \epsilon^2.
% $$

In the special case of $\ell=1$ in Theorem~\ref{thm:robust:higherorder}, we need the least singular value $\sing_m(W)>0$ (this necessitates that $m\le d$). This corresponds to the {\em full-rank setting} considered in Theorem~\ref{thm:full-rank}. 
  In contrast to the full-rank setting, for $\ell>1$ we only require that the set of vectors $w_1^{\otimes \ell}, w_2^{\otimes \ell}, \dots, w_m^{\otimes \ell}$ are linearly independent (in a robust sense), which one can expect for much larger values of $m$ typically. The following corollary formalizes this in the smoothed analysis framework of Spielman and Teng~\cite{ST04}, which is a popular paradigm for reasoning about non-worst-case instances~\cite{bwcabook}. Combining the above theorem with existing results on smoothed analysis~\cite{BCPV} implies polynomial time learning guarantees for non-degenerate instances with $m = O(d^{\ell})$ for any constant $\ell>0$. Below, $\widehat{W}$ denotes the columns of $W$ are $\tau$-smoothed i.e., randomly perturbed with standard Gaussian of average length $\tau$ that is at least inverse polynomial (See Section~\ref{app:smoothed} for the formal smoothed analysis model and result).

\begin{corollary}[Smoothed Analysis]\label{corr:smoothed}
Suppose $\ell \in \N$ and $\epsilon>0$ are constants in the smoothed analysis model with smoothing parameter $\tau>0$, and also assume the ReLU network $f(x)=a^\top \sigma(\widehat{W}^\top x+b)$ is $B$-bounded with $m \le 0.99 \binom{d+\ell-1}{\ell}$. 
 %and satisfies 
%  \begin{equation}
% (a)~~     \sigma_m(W) \ge \frac{1}{\poly(m,d)}, ~\qquad(b)~~ |a_i| \ge a_{\min} ~\forall i \in [m], ~\qquad (c) ~b_i \ge -b_{\min} ~\forall i \in [m].
%  \end{equation}
 Then there is an algorithm that given $N\ge \poly_\ell(m,d,1/\epsilon,B, 1/\tau)$ samples runs in $\poly(N,m,d)$ time and with high probability finds a ReLU network $g(x)={a'}^\top \sigma({W'}^\top x+{b'})$ \newtext{with at most $m+2$ hidden units}, such that the $L_2$ error $\E_{x \sim \calN(0,I_{d \times d})}[(f(x)-{g}(x))^2] \le \eps^2$. 
%i.e., $\forall x \in \R^d, ~ |f(x)-\widetilde{f}(x)| \le \epsilon \norm{x}_2$. 
 Furthermore there are constants $c, c'>0$ and signs $\xi_i \in \set{\pm 1}~\forall i \in [m]$, such that \newtext{ in $\poly(N,m,d)$ time,} for all $i \in [m]$ with $|b_i| < c \sqrt{\log(1/(\epsilon\cdot mdB))}$, \newtext{we can recover} $(\widetilde{a}_i, \widetilde{w}_i, \widetilde{b}_i)$, such that $|a_i - \widetilde{a}_i|+ \norm{w_i - \xi_i \widetilde{w}_i}_2+|b_i - \xi_i \widetilde{b}_i| < c'\epsilon/(mB)$.
\end{corollary}

While our algorithm and the analysis give guarantees that are robust to sampling errors and inverse polynomial error, even the non-robust analysis has implications and, implies identifiability of the model (up to ambiguity in the signs) as long as no two rows of $W$ are parallel. Note that in general identifiability may not imply any finite sample complexity bounds.   

\begin{theorem}[Partial Identifiability] \label{thm:identifiability}
Suppose we are given samples from a ReLU network $f(x)=a^\top \sigma(W^\top x+b)$ where $\min_{i \in [m]}|a_i|>0$ and no two columns of $W$ are parallel to each other. Then given samples,  the model parameters $(a_i, b_i ,w_i: i \in [m])$ are identified up to ambiguity in the signs and reordering indices i.e., we can recover $\set{(a_i, \xi_i b_i, \xi_i w_i): i \in [m]}$ for some $\xi_i \in \set{+1, -1}~ \forall i \in [m]$.\\
Moreover given any $(\xi_i \in \set{+1,-1}: i \in [m])$ such that 
\begin{equation} \label{eq:nonidentifiability}
    \sum_{i=1}^m a_i \Phi(\xi_i b_i) \xi_i w_i = \sum_{i=1}^m a_i \Phi(b_i) w_i, ~~\text{ and }~ \sum_{i=1}^m a_i  b_i\Phi(\xi_i b_i) =\sum_{i=1}^m a_i  b_i\Phi(b_i) ,
\end{equation} 
we have that the set of parameters $((a_i, \xi_i b_i , \xi_i w_i : i \in [m])$ also gives rise to the same distribution.  
\end{theorem}

%\anote{Can we give a simple example?}
The above theorem shows that under a very mild assumption on $W$, the parameters can be identified up to signs. However, this ambiguity in the signs may be unavoidable -- the second part of the Theorem~\ref{thm:identifiability} shows that any combination of signs that match the zeroth and first Hermite coefficient gives rise to a valid solution (this corresponds to the $d+1$ equations in \eqref{eq:nonidentifiability}). Even in the case when all the $b_i =0$, we have non-identifiability due to ambiguities in signs whenever the $\set{w_i: i \in [m]}$ are not linearly independent for an appropriate setting of the $\set{a_i}$; see Claim~\ref{claim:nonidentifiability} for a formal statement. On the other hand, Theorem~\ref{thm:non-robust:fullrank} gives unique identifiability result  in the full-rank setting (as there is only one setting of the signs that match the first Hermite coefficient in the full-rank setting).  

Our results rely on the precise expressions for higher order Hermite coefficients of $f(x)$ given below.  
\begin{lemma}\label{lem:hermitecoeffs}
Let $\hat{f}_k = \mathbb{E}_{x \sim \mathcal{N}(0,I)}[f(x)He_k(x)]$ (with $k \in \mathbb{N}$) be the $k$'th Hermite coefficient (this is an order-$k$ tensor) of $f(x) = a^\top \sigma(W^\top x+b)$. Then
\begin{align}
    \hat{f}_0 &= \sum_{i=1}^m a_i \Big( b_i\Phi(b_i) + \frac{\exp(-\frac{b_i^2}{2})}{\sqrt{2\pi}} \Big)
    ~,~
    \hat{f}_1 = \sum_{i=1}^m a_i \Phi(b_i) w_i\\
  \forall k \geq 2,~~  \hat{f}_k
    &= \sum_{i=1}^m (-1)^k \cdot a_i \cdot He_{k-2}(b_i) \cdot \frac{\exp(\frac{-b_i^2}{2})}{\sqrt{2\pi}} \cdot w_i^{\otimes k}
\end{align}
\end{lemma}

%\tnote{Moved the proof to appendix}
%\anote{I added something about Stein's lemma?}
We prove this by considering higher order derivatives and using properties of Hermite polynomials. A key property we use here is that the $k$'th derivative of a standard Gaussian function is itself multiplied by the $k$'th Hermite polynomial (with sign flipped for odd $k$). This significantly simplifies the expression for the coefficient $g(b_i)$ of $w_i^{\otimes k}$.% is exactly the $k-2$'th derivative of a Gaussian in terms of $b_i$.

We remark that the above lemma may also be used to give an expression for the training objective for depth-2 ReLU networks, analogous to the result of \cite{ge2017learning} for ReLU activations with no bias, that provides an expression as a combination of tensor decomposition problems of increasing order.
%See Appendix~\ref{app:lossobjective} for details. 
%Next we show that the population loss can be written as a sum of tensor decomposition problems. A similar result we shown in the work of \cite{ge2017learning} for the case when the bias terms are zero. We extend this to the more general setting. 
The authors in \cite{ge2017learning} crucially use the form of the decomposition to design a new regularized objective on which the convergence of gradient descent can be analyzed. The decomposition presented below for general ReLU activations opens the door for analyzing gradient descent in the non-zero bias setting.

\begin{proposition} \label{prop:lossobj}
    Let $\widetilde{f}(x)=\widetilde{a}^\top \sigma(\widetilde{W}^\top x +\widetilde{b})$ be the model trained using samples generated by the ground-truth ReLU network $f(x) = a^\top \sigma(W^\top x + b)$. 
    %In addition, define $\alpha \bowtie \alpha'$ as $\alpha$ being a permutation of $\alpha'$ and $c_k^{\alpha} = \sqrt{\mathbb{E}[He_k^{\alpha}(x)^2]}$. 
%    \anote{The permutation definition and the $c_k$ definition doesn't need to be in the statement, and can go into the proof. }
    Then the statistical risk with respect to the $\ell_2$ loss function can be expressed as follows

    \begin{align*}
        L(\widetilde{a}, \widetilde{b}, \widetilde{W}) &= 
        \sum_{k \in \mathbb{N}}
        \frac{1}{k!} \Big\| T_k - \hat{f}_k \Big\|_F^2 \\
        ~~\text{ where }~~ T_0 &= \sum_{i=1}^m \widetilde{a}_i(\widetilde{b}_i\Phi(\widetilde{b}_i) + \frac{\exp(-\widetilde{b}_i^2/2)}{\sqrt{2\pi}}), 
        ~\text{ and }~ T_1 = \sum_{i=1}^m \widetilde{a}_i\Phi(\widetilde{b}_i)\widetilde{w}_i\\
        \forall k \geq 2,~~ T_k &= \sum_{i=1}^m (-1)^k \cdot \widetilde{a}_i \cdot He_{k-2}(\widetilde{b_i}) \cdot \frac{\exp(-\widetilde{b_i}^2/2)}{\sqrt{2\pi}} \cdot \widetilde{w}_i^{\otimes k}
    \end{align*}
\end{proposition}

Please refer to Appendix \ref{app:sec:hermite} for the proof. 

Observe that by setting $b_i = \widetilde{b}_i = 0$ in our above expression, when $k \geq 2$, we immediately recover the objective function given in Theorem 2.1 of \cite{ge2017learning} as
\begin{equation}
    \sum_{\substack{k \geq 2}} \frac{He_{k-2}(0)^2}{2\pi k!} \Big\| \sum_{i=1}^m \widetilde{a}_i \widetilde{w}_i^{\otimes k} - \sum_{i=1}^m a_i w_i^{\otimes k} \Big\|_F^2
    =
    \sum_{\substack{k \geq 2 \\ k ~ \text{is even}}} \frac{((k-3)!!)^2}{2\pi k!} \Big\| \sum_{i=1}^m \widetilde{a}_i \widetilde{w}_i^{\otimes k} - \sum_{i=1}^m a_i w_i^{\otimes k} \Big\|_F^2
\end{equation}

for $k = 0, 1$, we also have
\begin{equation}
    \Big| T_0 - \hat{f}_0 \Big|^2 + \Big\| T_1 - \hat{f}_1 \Big\|_2^2
    =
    \frac{1}{2\pi}\Big| \sum_{i=1}^m \widetilde{a}_i - \sum_{i=1}^m a_i \Big|^2
    +
    \frac{1}{4}\Big\| \sum_{i=1}^m \widetilde{a}_i\widetilde{w}_i - \sum_{i=1}^m a_iw_i \Big\|_2^2
\end{equation}

Furthermore, note that the Hermite coefficients of the ReLU function $\sigma(x)$ are $\hat{\sigma}_0 = 1/\sqrt{2\pi}$, $\hat{\sigma}_1 = 1/2$ and $\hat{\sigma}_k = (-1)^{\frac{k-2}{2}}(k-3)!!/\sqrt{2\pi k!}$ for $k \geq 2$ and $k$ being even; otherwise $\hat{\sigma}_k = 0$.

\section{Non-robust Algorithm and Analysis}
\label{sec:non-robust}
Our algorithms for learning the parameters of $f(x)=a^\top \sigma(W^\top x+b)$ decompose tensors obtained from the Hermite coefficients $\set{\hat{f}_t \in (\R^{d})^{\otimes t}}$ of the 
function $f$. %With polynomial samples, we can estimate these tensors up to inverse polynomial accuracy. 
In this section, we design an algorithm assuming that we have access to all the necessary Hermite coefficients exactly (no noise or sampling errors). This will illustrate the basic algorithmic ideas and the identifiability result. However with polynomial samples, we can only hope to estimate these quantities up to inverse polynomial accuracy. In Section~\ref{sec:robust} we describe how we deal with the challenges that arise from errors. %, and how we overcome them.   

%\anote{Perhaps we should change all the singular values to $\sing_m$ and $\sing_{min}$, since $\sigma$ is used for activations. }

% The following algorithm gives polynomial time guarantees when the weight matrix $W \in \R^{d \times m}$ satisfy $\sing_m(W) > 0$, and we are given the Hermite coefficients $\hat{f}_2, \hat{f}_3,\hat{f}_4$. They both use Jennrich's algorithm as a black-box (see Theorem~\ref{thm:jennrich} and Section~\ref{sec:jennrich}). %we present Algorithm \ref{alg1} that recovers $a$, $b$ and $W$ up to permutation.
% \tnote{Change (i) no two ... are linearly independent ... to (i) ... are linearly dependent ...}
% \tnote{Removed (iii) since $B$-bounded already implies $\min_{i \in [m]} |a_i| \geq 1/B$} \anote{Thanks for catching (i). Put back (iii) since we don't assume B-bounded for Theorem 4.1}

Our first result is a polynomial time algorithm in the full-rank setting that recovers all the parameters exactly. 
\begin{theorem}[Full-rank non-robust setting]\label{thm:non-robust:fullrank}
    Suppose the parameters $\set{(a_i, b_i, w_i): i \in [m]}$  satisfies:~ {\em (i)} $a_i \ne 0$ for all $i \in [m]$, ~ {\em (ii)} $\set{w_i: i \in [m]}$ are linearly independent.
    Then given  $\set{\hat{f}_t : 0\le t \le 4}$ exactly, Algorithm~\ref{alg1} recovers (with probability $1$) the unknown parameters $a$, $b$ and $W$ in $\poly(m, d)$ time.
\end{theorem}
%\anote{Just started editing. Not yet done.}
See Theorem~\ref{thm:full-rank} for the analogous theorem in the presence of errors in estimating the Hermite coefficients $\set{\hat{f}_t}$. Our algorithm for recovering the parameters estimates different Hermite coefficient tensors $\set{\hat{f}_t: 0 \le t \le 4}$ and uses tensor decomposition algorithms on these tensors to first find the $\set{w_i: i \in [m]}$ up to some ambiguity in signs. We can also recover all the coefficients $\set{b_i: i \in [m]}$ up to signs (corresponding to the signs of $w_i$ ), and all the $\set{a_i: i \in [m]}$ (no sign ambiguities). This portion of the algorithm extends to higher order $\ell$, under a weaker assumption on the matrix $W$.    

\begin{theorem}\label{thm:non-robust:higher}
    Suppose the parameters $\set{(a_i, b_i, w_i): i \in [m]}$  satisfies:~ {\em (i)} no two $\set{w_i: i \in [m]}$ are linearly dependent and,~ {\em (ii)} for a constant $\ell \in \mathbb{N}$,  $\set{w_i^{\otimes \ell}: i \in [m]}$ are linearly independent,
    ~{\em (iii)} $a_i \ne 0$ for all $i \in [m]$. 
    Then given  $\set{\hat{f}_t : 0\le t \le 2\ell+2}$ exactly, Algorithm~\ref{alg1}  in $\poly_\ell(m,d)$ time outputs (with probability $1$) $\set{\hat{a}_i, \hat{w}_i, \hat{b}_i: i \in [m]}$ such that we can recover the parameters up to a reordering of the indices $[m]$ and up to signs i.e.,  for some $\set{\xi_i \in \set{-1,1}: i \in [m]}$  we have $\hat{a}_i = a_i$, $\hat{w}_i = \xi_i w_i$ and $\hat{b}_i = \xi_i b_i$. \newtext{Furthermore, given {exact} statistical query access to the distribution $\calN(0, I_{d \times d})$,\footnote{This means that for any function $h(x,y)$ that can be computed in polynomial time, one can obtain $\Esymb_{x,y}[h(x,y)]$ exactly.} there exists an algorithm that runs in time $\poly(m,d,B)$ and outputs a function $g(x)$ such that $\Esymb_{x \sim \calN(0,I_{d \times d})}\big(f(x)-g(x)\big)^2 = 0$.}
    %Moreover, we can also recover in $\poly_\ell(m,d)$ time recover parameters $\set{ a_i , \xi_i w_i, \xi_i b_i: i \in [m]}$ such that for all $i$ with $b_i \ne 0$, we have $\xi_i =1$, and $\xi_i \in \set{+1, -1}$ when $b_i=0$.
\end{theorem}
%\anote{Pranjal, can you add something for fitting with regression?}

We now describe the algorithm for general $\ell\ge 1$ (this specializes to the full-rank setting for $\ell=1$).

\begin{algorithm}[H]
\label{alg1}
%\vspace{-10pt}

%\vspace{-5pt}
\SetAlgoLined
\textbf{Input:} $\hat{f}_{\ell}, \hat{f}_{\ell+1}, \hat{f}_{\ell+2}, \hat{f}_{\ell+3}, \hat{f}_{2\ell+1}, \hat{f}_{2\ell+2}$\;
1. Let $T'=\flatten(\hat{f}_{2\ell+1},\ell,\ell,1) \in \R^{d^{\ell} \times d^{\ell} \times d }$ and $T''=\flatten(\hat{f}_{2\ell+2},\ell,\ell,2) \in \R^{d^{\ell} \times d^{\ell} \times d^2}$ be order-3 tensors obtained by flattening $\hat{f}_{2\ell+1}$ and $\hat{f}_{2\ell+2}$. 

2. Set $k'=\text{rank}(\flatten(\hat{f}_{2\ell+1},\ell, \ell+1,0))$. Run Jennrich's algorithm~\cite{Harshman1970}  on $T'$ to recover rank-1 terms $\set{\alpha'_i u_i^{\otimes \ell} \otimes u_i^{\otimes \ell} \otimes u_i ~|~ i \in [k'] }$, where $\forall i \in [k'],~ u_i \in \bbS^{d-1}$ and $\alpha'_i \in \R$. %and use PCA on each rank-$1$ term to find $\set{u_i: i \in [k']}$. 

3. Set $k''=\text{rank}(\flatten(\hat{f}_{2\ell+2},\ell, \ell+1,0))$. Run Jennrich's algorithm~\cite{Harshman1970} on $T''$ to recover rank-1 terms $\set{\alpha''_i v_i^{\otimes \ell} \otimes v_i^{\otimes \ell} \otimes v_i^{\otimes 2} ~|~ i \in [k''] }$, where $\forall i \in [k''],~ v_i \in \bbS^{d-1}$ and $\alpha''_i \in \R$. % and use PCA on each rank-$1$ term to find $\set{v_i: i \in [k'']}$.

4. Remove duplicates and negations (i.e., antipodal pairs of the form $v$ and $-v$) from $\set{u_1, u_2, \dots, u_{k'}} \cup \set{v_1, v_2, \dots,  v_{k''}}$ to get $\widetilde{w}_1, \widetilde{w}_2, \dots, \widetilde{w}_m$. %Let their corresponding $\alpha$ values be $\alpha_1, \dots, \alpha_m \in \R$. Let $S' \subseteq [m]$ be those indices with their $\tilde{w}_i$ recovered from step 2, and $S'' \subseteq [m]$ be the remaining indices. 

% Run Jennrich's Algorithm on $\hat{f}_3 = \sum_{i=1}^m \alpha_i w^{\otimes 3}_i$ with constraint $\|w_i\| = 1$ to recover $\alpha_i, w_i$\;
% \eIf{
%     Number of non-zero eigenvalues output by Jennrich's Algorithm $\ne m$
% }{
%     $\forall i,j \in [d], ~ \hat{f}_{4 ~ ij}^{\prime} = \mathsf{vec}(\hat{f}_{4 ~ ij})$\;
%     Run Jennrich's Algorithm on $\hat{f}_4^{\prime} = \sum_{i=1}^m \alpha_i^{\prime} w^{\prime \otimes 2}_i \otimes \mathsf{vec}(w^{\prime \otimes 2}_i)$ with constraint $\|w_i^{\prime}\| = 1$ to recover $\alpha^{\prime}_i, w^{\prime}_i$\;
%     Add all newly recovered $w^\prime_i, \alpha_i^\prime$ to the set of recovered $w_i, \alpha_i$\;
% }{
%     Pass\;
% }

5. Run subroutine \textsc{RecoverScalars}$(m,\ell, \set{\widetilde{w}_i: i \in [m]}, \hat{f}_{\ell}, \hat{f}_{\ell+1}, \hat{f}_{\ell+2}, \hat{f}_{\ell+3})$ to get $\set{\widetilde{a}_i, \widetilde{b}_i: i \in [m]}$.  

6. If $\ell=1$ (full-rank setting), run Algorithm 3 ({\sc FixSigns}) on parameters $m, \hat{f}_1$ and $(\widetilde{a}_i, \widetilde{b}_i, \widetilde{w}_i: i \in [m] )$ to get $(a'_i=\widetilde{a}_i,b'_i, w'_i: i \in [m])$. 
%%%Expanded the subroutine here itself%%%%

% 5. For each $j \in \set{\ell, \ell+1, \ell+2, \ell+3}$, solve the following system of linear equations to recover unknowns $\set{\zeta_j(i) : i \in [m]}$ 
%  $$\sum_{i=1}^m \zeta_j(i) \mathsf{vec}(\tilde{w}_i^{\otimes j})=\mathsf{vec}(\hat{f}_j)$$ 

% 6. For each $i \in [m]$, use \eqref{eq:ab_assgt} in Lemma~\ref{lem:recovscalars} for $k=\ell$ and $c_{j-2}=\zeta_j(i)$ for $j \in \set{\ell, \ell+1, \ell+2, \ell+3}$, to find $(\tilde{a}_i : i \in [m])$ and  $(\tilde{b}_i: i \in [m])$.  

\KwResult{Output $\set{\widetilde{w}_i, \widetilde{a}_i, \widetilde{b}_i : i \in [m]}$}

\caption{Algorithm for order $\ell$: recover $a$, $b$, $W$ given $\set{\hat{f}_{t}: 0\le t \le 2\ell+2}$}
\end{algorithm}
%
%\vspace{-10pt}
%
%
Subroutine Algorithm~\ref{alg2} finds the unknown parameters $a_1, \dots, a_m \in \R$ and $b_1, \dots, b_m \in \R$  given $w_1, w_2, \dots, w_m$. While Algorithm~\ref{alg1} changes a little when we have errors in the estimates, the subroutine Algorithm~\ref{alg2} remains the same even for the robust version of the algorithm. 

%\tnote{when $\ell=1$, we should be able to recover $b_i = -\zeta_3(i)/\zeta_2(i)$ up to sign. $\zeta_1(i)$ shouldn't be needed in this step.}
%
%\vspace{-10pt}
\begin{algorithm}
\label{alg2}

\SetAlgoLined
%\textbf{Input:} $\ell, (\tilde{w}_i: i \in [m]), \hat{f}_{\ell}, \hat{f}_{\ell+1}, \hat{f}_{\ell+2}, \hat{f}_{\ell+3}$\;
\textbf{Input:} $m,\ell, (\widetilde{w}_i: i \in [m])$, and tensors $T_{\ell}, T_{\ell+1}, T_{\ell+2}, T_{\ell+3}$ which are tensors of orders $\ell, \ell+1, \ell+2, \ell+3$ respectively\;

\eIf{$\ell=1$}{
1. For $j \in \set{2, 3}$, solve the system of linear equation 
$\sum_{i=1}^m \zeta_j(i) \mathsf{vec}(\widetilde{w}_i^{\otimes j})=\mathsf{vec}(T_j)$ to recover unknowns $\set{\zeta_j(i) ~|~ i \in [m]}$\;
2. For each $i \in [m]$, set $b_i = -\frac{\zeta_3(i)}{\zeta_2(i)}$, and $a_i=\zeta_2(i)\cdot \sqrt{2\pi} e^{b_i^2/2}$.  
}
{1. For $j \in \set{\ell, \ell+1, \ell+2, \ell+3}$, solve the system of linear equation 
$\sum_{i=1}^m \zeta_j(i) \mathsf{vec}(\widetilde{w}_i^{\otimes j})=\mathsf{vec}(T_j)$ to recover unknowns $\set{\zeta_j(i) ~|~ i \in [m]}$\;

\vspace{5pt}
2. For each $i \in [m]$, $q_i:= \argmax_{\substack{j \in \set{\ell+1, \ell+2}}} |\zeta_{j}(i)|$, set $\widetilde{b}_i=  -\frac{\zeta_{q+1}(i)+q 
  \cdot \zeta_{q-1}(i)}{\gamma_{q}(i)}$, and  $\widetilde{a}_i= \sqrt{2\pi}(-1)^{q} \zeta_{q}(i)e^{b_i^2/2}/He_{q}(b_i)$, as   
  %for $k=\ell$ and $\gamma_{j}=\zeta_j(i)$ for $j \in \set{\ell, \ell+1, \ell+2, \ell+3}$, to find $(\tilde{a}_i : i \in [m])$ and  $(\tilde{b}_i: i \in [m])$.  
  described in \eqref{eq:ab_assgt} (Lemma~\ref{lem:recovscalars}).
}
\KwResult{Output $(\widetilde{a}_i, \widetilde{b}_i : i \in [m])$}

\caption{{\bf Subroutine} {\sc RecoverScalars} to recover $b_i$ (up to signs) and $a_i$ given $w_i$ (up to signs) for all $i \in [m]$.}
\end{algorithm}

The above two algorithms together recover for all $i\in [m]$, the $a_i$ and up to a sign the $w_i$ and $b_i$. 
In the special case of $\ell=1$ which we refer to as the {\em full-rank setting}, we can also recover the correct signs, and hence recover all the parameters.  

%
%\vspace{-10pt}
\begin{algorithm}[H]
\label{alg:fixsigns}

\SetAlgoLined
%\textbf{Input:} $\ell, (\tilde{w}_i: i \in [m]), \hat{f}_{\ell}, \hat{f}_{\ell+1}, \hat{f}_{\ell+2}, \hat{f}_{\ell+3}$\;
\textbf{Input:} $m$, $\hat{f}_{1} \in \R^d$ and estimates $\tilde{a}_i, \tilde{b}_i, \tilde{w}_i$ for each $i \in [m]$\;

%1. Set $\tilde{\alpha}_i \coloneqq -  \widetilde{a}_i \Phi(\widetilde{b}_i)$. 

1. Solve the system of linear equation 
$\sum_{i=1}^m z_i \tilde{a}_i \tilde{w}_i =\hat{f}_1$ to recover unknowns $\set{z_i ~|~ i \in [m]}$\;
2. Set $\tilde{\xi}_i = \text{sign}(z_i)$ for each $i \in [m]$.  

\KwResult{Output $(\tilde{a}_i, \tilde{\xi}_i \tilde{b}_i, \tilde{\xi}_i \tilde{w}_i : i \in [m])$}

\caption{{\bf Algorithm} {\sc Fix signs in full-rank setting}.}
\end{algorithm}

Algorithm~\ref{alg1} decomposes two different tensors obtained from consecutive Hermite coefficients $\hat{f}_{2\ell+1}, \hat{f}_{2\ell+2}$ to obtain the $\set{w_i}$ up to signs. We use two different tensors because the bias $b_i$ could make the coefficient of the $i$th term in the decomposition $0$ (e.g., $He_{2\ell-1}(b_i) = 0$ for $\hat{f}_{2\ell+1}$); hence $w_i$ cannot be recovered by decomposing $\hat{f}_{2\ell+1}$. %More specifically, consider when $k = 3$, where $0$ is the only root of $He_1(b_i) = b_i = 0$. 
Hence $\hat{f}_{2\ell+1}$ can degenerate to a rank $m'< m$ tensor, and Jennrich's algorithm will return only $m'<m$ eigenvectors that correspond to non-zero eigenvalues. %\cite{LRA93}. 
%Consequently, when this is the case, we must consider higher-order coefficients of the Hermite expansion to recover those $w_i$.
%To address this issue, we first introduce a key property regarding the roots of Hermite polynomials we will utilize later. 

The following lemma addresses this issue by showing that two consecutive Hermite polynomials can not both take small values at any point $x \in \R$. This implies a separation between roots of consecutive Hermite polynomials or , and establishes a ``robust'' version that will be useful in Section~\ref{sec:robust}. Moreover, this lemma also shows that when $|x|$ is not close to $0$, at least one out of every two consecutive {\em odd} Hermite polynomials takes a value of large magnitude at $x$.  

%the robustness analysis of  

\begin{lemma}[Separation of Roots]
\label{lem:hermite_roots}
%If $|He_k(x)| < 1/\mathsf{poly}(x)$, then $He_{k+1}(x) \geq c_k$, where $c_k$ is some constant depending only on $k$.
For all $k \in \mathbb{N}, x \in \R$,  $\max\{|He_k(x)|, |He_{k+1}(x)|\} \ge \sqrt{k!/2}$. %for all $x \in \R$.
%Moreover $\forall x \in \R$ we have $\max\{|He_k(x)|, |He_{k+2}(x)|\} \ge \tfrac{1}{4}\min\{|x|,1\}\cdot  \sqrt{k!}/(k+1)$. 
\end{lemma}
%\vspace{-4pt}
\begin{proof}
    % \anote{Suppose for contradiction that $\max\{|He_{k}(x)|, |He_{k+1}(x)|\}<c_k$ ...}
    % \\
    % \tnote{Assume the above inequality holds, then it is equivalent to say that both $|He_k(x)| < c_k$ and $|He_{k+1}(x)| < c_k$ holds. However since $He_k(x)$, $He_{k+1}(x)$ are degree-$k$ and $k+1$ polynomials respectively with leading coefficient 1, $\forall c_k > 0$ $\exists x \in R$ s.t. $He_k(x) > c_k$ and $He_{k+1}(x) > c_k$, which contradicts the assumption. Thus the lemma follows.}
    % \\
    
    First, $\forall k \in \mathbb{N}$, by Turán's inequality~\cite{turan1950zeros} we have
    \begin{equation} \label{eq:turan}
        He_{k+1}^2(x) - He_k(x)He_{k+2}(x) = k! \cdot \sum_{i=0}^k \frac{He_i(x)^2}{i!} > 0
    \end{equation}
    Set $\epsilon = \sqrt{k!/2}$ and assume for contradiction that $|He_{k+1}(x)|, |He_{k+2}(x)| < \epsilon$. The LHS of \eqref{eq:turan} is at most $\epsilon^2 + \epsilon|He_k(x)|$, and the RHS of \eqref{eq:turan} is at least
    \begin{equation}\label{eq:turan2}
        k!\cdot (1 + x^2 + ... + \frac{He_k(x)^2}{k!}) > k! \cdot (1+\frac{He_k(x)^2}{k!})
    \end{equation}
    Therefore, if $|He_k(x)| = t$, combining both sides we get $\epsilon^2 + \epsilon t \geq k! + t^2$. This implies on the one hand that $\epsilon t \ge k!/2$, and on the other hand that $\epsilon t \ge t^2$. %Choose $\epsilon^\prime = \frac{k!}{2}$, then $\epsilon t \geq k!/2 + t^2$, indicating $\epsilon t \geq k!/2$ and $\epsilon t \geq t^2$, where the latter establishes $\epsilon \geq t$, thus $\epsilon \geq \sqrt{k!/2}$ must hold as well, which 
    However for our choices of $\eps=\sqrt{k!/2}$, no value of $t$ is feasible. This yields the required contradiction for the first claim. % our bounds on $|He_{k+1}(x)|, |He_{k+2}(x)|$. 
%    contradicts our assumption. Hence the lemma follows.
%\anote{Thanks for filling in the proof Alex!}

% For the second claim, suppose $|He_{k}(x)| < \min\set{|x|,1} \sqrt{k!}/(4(k+1))$ (else we are already done), then $|He_{k+1}(x)| \ge \sqrt{k!}/2$. We now use \eqref{eq:turan} along with the following recurrence relation involving Hermite polynomials
% \begin{align*}
%   He_{k+2}(x) &= x He_{k+1}(x) - (k+1) He_k(x)\\
%   |He_{k+2}(x)| & \ge  |x| |He_{k+1}(x)| -  (k+1) |He_{k}(x)| \ge \frac{|x| \sqrt{k!}}{2}
%   -  \frac{|x|\sqrt{k!}}{4} \ge \frac{|x|\sqrt{k!}}{4(k+1)}.
% \end{align*}
% This proves the second claim. 
\end{proof}

%Lemma~\ref{lem:hermite_roots} shows that to and vice versa; moreover the above lemma establishes a ``robust'' version that will be useful in the robust analysis. 

The following claim shows that Jennrich's algorithm for decomposing a tensor successfully recovers all the rank-$1$ terms whose appropriate $He_k(b_i) \ne 0$. This claim along with Lemma~\ref{lem:hermite_roots} shows that Steps 2-3 of Algorithm~\ref{alg1} successfully recovers all the $\set{w_i: i \in [m]}$ up to signs.
\begin{claim}\label{claim:nonoise:jennrich}
Let $\ell_1, \ell_2 \ge \ell$ and $T \in \R^{d^{\ell_1} \times d^{\ell_2} \times d^{\ell_3}}$ have a decomposition $T=\sum_{i=1}^m \alpha_i w_i^{\otimes \ell_1} \otimes w_i^{\otimes \ell_2} \otimes w_i^{\otimes \ell_3}$, with $\set{w_i^{\otimes \ell}: i \in [m]}$ being linearly independent. Consider matrix $M=\flatten(T,\ell_1, \ell_2+\ell_3,0) \in \R^{d^{\ell_1} \times d^{\ell_2+\ell_3}}$, and let $r:= rank(M)$. Then Jennrich's algorithm applied with rank $r$ runs in $\poly_{\ell_1+\ell_2+\ell_3}(m,d)$ time  recovers (w.p. $1$) the rank-$1$ terms corresponding to $\set{i \in [m]: |\alpha_i|>0}$. Moreover for each $i$ with $|\alpha_i|>0$, we have $\tilde{w}_i = \xi_i w_i$ for some $\xi_i \in \set{+1,-1}$. 
\end{claim}
\begin{proof}
    Let $Q=\set{i \in [m]: |\alpha_i| > 0}$.
    Firstly $r=rank(M)=|Q|$, since $M$ has a decomposition 
    $$M=\sum_{i \in Q} \alpha_i \big(w_i^{\otimes \ell_1}\big) \big( w_i^{\otimes \ell_2+\ell_3} \big)^\top = M_1 \diag(\alpha_Q) M_2^\top  $$
    where $M_1, M_2, \diag(\alpha_Q)$ all have full column rank $|Q|$.
    Secondly, from assumption (ii) of Theorem~\ref{thm:non-robust:higher} and Claim~\ref{claim:khatri-rao} applied with  $\ell_1, \ell_2 \ge \ell$, we have that $\set{w_i^{\otimes \ell'}: i \in [m]}$ are linearly independent for every $\ell'\ge \ell$. Hence the
     the factor matrices  $U=(w_{i}^{\otimes \ell_1}: i \in Q)$ and $V=(w_i^{\otimes \ell_2}: i \in Q)$ also have full column rank. 
    %This is because even the full set of columns $(w_{i}^{\ell}: i \in [m])$ are linearly independent by assumption; hence from Claim~\ref{claim:khatri-rao} applied with  $\ell_1, \ell_2 \ge \ell$, we see that $U$ and $V$ have rank $|Q|$. 
    Similarly from (iii) no two vectors in $\set{\alpha_i w_i^{\otimes \ell_3}: i \in Q}$ are parallel. Hence, they satisfy the conditions of Jennrich's algorithm. Since there is no error in the tensor, Jennrich's algorithm (Theorem~\ref{thm:jennrich}) succeeds with probability $1$ (see \cite{AVbookchapter}). Finally since each rank-1 term is recovered exactly when $\alpha_i \ne 0$, the vector in $\R^d$ obtained from the term will correspond to either $w_i$ or $-w_i$ as required. 
\end{proof}

%For $k=3$, this implies by looking at $\hat{f}_4$, any $w_i$ not recovered through $\hat{f}_3$ can be captured by decomposing $\hat{f}_4$. 
The above claim was useful in recovering $w_i$ up to a sign ambiguity. 
The following lemma is useful for recovering $a_i$ parameters (no sign ambiguities) and the $b_i$ parameters up to sign ambiguity, once we have recovered the $w_i$ up to sign ambiguity. It  uses  various properties of Hermite polynomials along with Lemma~\ref{lem:hermitecoeffs} and Lemma~\ref{lem:hermite_roots}.

\begin{lemma}\label{lem:recovscalars}
Suppose $k \in \N, k \ge 2$.  Suppose for some unknowns $\beta,z \in \R$ with $\beta\ne 0$, we are given values of $\gamma_j=(-1)^j \xi^j \beta He_j(z)~ \forall j \in \set{k, k+1, k+2, k+3}$ for some $\xi \in \set{+1, -1}$. Then $z, \beta$ are uniquely determined by
\begin{equation}
\text{For } q:= \argmax_{\substack{j \in \set{k+1, k+2}}} |\gamma_{j}|, ~~ \xi z=  -  \frac{\gamma_{q+1}+q 
  \cdot \gamma_{q-1}}{\gamma_{q}}, ~~~~ \beta= (-1)^{q} \frac{ \gamma_{q}}{He_{q}(\xi z)}  \label{eq:ab_assgt}
\end{equation}

% \begin{equation}\label{eq:ab_assgt}
%   z=  -\frac{\gamma_{k+2}+(k+1)\gamma_{k}}{\gamma_{k+1}}, ~\beta= \frac{\gamma_{k+1}}{He_{k+1}(z)} ~\text{ if }\gamma_{k+1} \ne 0, \text{ and otherwise }
%     z=  -\frac{\gamma_{k+3}}{\gamma_{k+2}}, ~ \beta= \frac{\gamma_{k+2}}{He_{k+2}(z)} . 
% \end{equation}
\end{lemma}
\vspace{-5pt}

\begin{proof}
    We use the following fact about Hermite polynomials:
    \begin{equation} \label{eq:recurrence}
        He_{r+1}(z)=z He_{r}(z)-r \cdot He_{r-1}(z).
    \end{equation}
    From Lemma~\ref{lem:hermite_roots}, we know that $\max\set{|He_{k+1}(z)|, |He_{k+2}(z)|} >0$ and hence $\gamma_q \ne 0$. 
    Substituting in the recurrence \eqref{eq:recurrence} with $r=q$,  
    \begin{align*}
        z&= \frac{He_{q+1}(z) + q \cdot He_{q-1}(z)}{He_{q}(z)} = \frac{\beta He_{q+1}(z) + q\cdot \beta He_{q-1}(z)}{ \beta He_{q}(z)} =
         -\xi \Big(\frac{\gamma_{q+1}+q \cdot \gamma_{q-1}}{\gamma_{q}}\Big),
    \end{align*}
    where we used the fact that the Hermite polynomials are odd functions for odd $q$ and even polynomials for even $q$. The $\beta$ value is also recovered since $\gamma_q=(-1)^q \beta (\xi^q He_q(z)) = (-1)^q \beta \cdot He_q( \xi z)$. 
    
    % If $He_{k+1}(z) =0$, then from Lemma~\ref{lem:hermite_roots}, we have that $He_{k+2}(z) \ne 0$; hence $\gamma_{k+2} \ne 0$. Then substituting \eqref{eq:recurrence} with $r=k+2$ and since $\beta\ne 0$
    % \begin{align*}
    %     z&= \frac{\beta He_{k+3}(z) + (k+2) \beta He_{k+1}(z)}{ \beta He_{k+2}(z)} =
    %      -\frac{\gamma_{k+3}+(k+2) \gamma_{k+1}}{\gamma_{k+2}} = -\frac{\gamma_{k+3}}{\gamma_{k+2}}.
    % \end{align*}
    
\end{proof}

%\anote{Check that the robust version is also updated.}
A robust version of this lemma (see Section~\ref{app:ab_recovery}) will be important in the robust analysis of Section~\ref{sec:robust}. 
The following claim applies the above lemma for each $i \in [m]$ with $\xi_i z=b_i$ and $\beta=a_i e^{-b_i^2/2}/\sqrt{2\pi}$, to show that Step 5 of the algorithm recovers the correct $\set{(a_i, \xi_i b_i)}_{i \in [m]}$ given the $\set{\xi_i w_i}_{i \in [m]}$. 
%\anote{Modify to include the proof modification for $\ell=1$.}
\begin{claim}\label{claim:nonrobust:ab_recovery}
Given $\set{\widetilde{w}_i=\xi_i w_i : i \in [m]}$ where $\xi_i\in \set{+1,1}~ \forall i \in [m]$, Step 5 of Alg.~\ref{alg1} recovers $\set{(a_i, \xi_i b_i, \xi_i w_i): i \in [m]}$.  
\end{claim}
\begin{proof}
    We first prove for $\ell \ge 2$. For each of the $j \in \set{\ell, \ell+1, \ell+2, \ell+2}$,  we have from Lemma~\ref{lem:hermitecoeffs} and the Hermite polynomials $He_j$ being odd functions for odd $j$ and even functions for even $j$,  
    $$\hat{f}_j = \sum_{i=1}^m (-1)^j \cdot a_i \cdot  \cdot \frac{He_{j-2}(b_i) \exp(-b_i^2/2)}{\sqrt{2\pi}} \cdot w_i^{\otimes j} =  \sum_{i=1}^m (-1)^j \cdot a_i \cdot  \frac{He_{j-2}(\xi_i b_i) \exp(-b_i^2/2)}{\sqrt{2\pi}} \cdot  \widetilde{w}_i^{\otimes j}.$$
    Moreover the vectors $\set{\xi_i^j w_i^{\otimes j}: i \in [m]}$  are linearly independent by assumption (and from Claim~\ref{claim:khatri-rao} for $j > \ell$). Hence the linear system for each $j$ has a unique solution 
    $$\forall \ell \le j \le \ell+3,  ~\forall i \in [m], ~\text{ we have } \zeta_j(i)=a_i \cdot \frac{(-1)^j}{\sqrt{2\pi}}\exp(-b_i^2/2) ~He_{j-2}(\xi_i b_i). $$ 
     Lemma~\ref{lem:recovscalars} applied with $\beta=\frac{1}{\sqrt{2\pi}} a_i  e^{-b_i^2/2}$  and $\gamma_{j-2}=(-1)^j He_{j-2}(\xi_i b_i)$ (note $\beta \ne 0$) proves that Alg.~\ref{alg2} recovers $a_i, \xi_i b_i$. 
     
     For $\ell=1$, we note $\forall z\in \R,~ He_{0}(z)=1$, and $He_1(z)=z$. From Lemma~\ref{lem:hermitecoeffs}, we see that one set of solutions to the linear system is 
     $$\forall i \in [m],~~\zeta_2(i)=a_i \frac{\exp(-b_i^2/2)}{\sqrt{2\pi}}, ~~\text{ and } ~~ \zeta_3(i)=-a_i \xi_i b_i \cdot \frac{\exp(-b_i^2/2)}{\sqrt{2\pi}}. $$ 
     Moreover the vectors $\set{\xi_i w_i: i \in [m]}$ are linearly independent. Hence, $\zeta_2, \zeta_3$ are the unique solutions to the system. Hence Algorithm~\ref{alg2} recovers $a_i, \xi_i b_i$ as claimed. 
\end{proof}

We now complete the proof of the non-robust analysis for any constant $\ell \ge 0$. 
\begin{proof}[Proof of Theorem~\ref{thm:non-robust:higher}]
The proof follows by combining Claim~\ref{claim:nonoise:jennrich} and Claim~\ref{claim:nonrobust:ab_recovery}, along with Lemma~\ref{lem:hermite_roots}. 
    Let $Q_1 = \set{i: |He_{2\ell-1}(b_i)|>0 }$ and $Q_2 = \set{i: |He_{2\ell+1}(b_i)|>0 }$.     
    From Claim~\ref{claim:nonoise:jennrich}, Step 2 recovers all the rank-$1$ terms in $Q_1$ with probability $1$; hence we obtain in particular $\set{\xi_i w_i \in \mathbb{S}^{d-1}~|~ i \in Q_1} $ for some signs $\xi_i \in \set{+1,-1}$. Similarly, in Step 3 we recover w.p. $1$, the  $\set{\xi_i w_i \in \mathbb{S}^{d-1}~|~ i \in Q_2}$ for some $\xi_i \in \set{1,-1}$.   
    
From Lemma~\ref{lem:hermite_roots}, we know that no $x \in \R$ is a simultaneous root of $He_{2\ell+1}(x), He_{2\ell+2}(x)$. Hence $Q_1 \cup Q_2 = \set{1, 2, \dots, m}$. Thus we obtain $\set{\xi_i w_i \in \mathbb{S}^{d-1} : i \in [m]}$ in Step 4 for some signs $\xi_i \in \set{1,-1}$ for all $i \in [m]$. 
Finally using Claim~\ref{claim:nonrobust:ab_recovery}, we recover for each $i \in [m]$, the $a_i, \xi_i b_i \in \R$ corresponding to $\xi_i w_i$. %This concludes the proof.     

\newtext{Next, in order to recover a function $g(x)$ of zero $L_2$ error we set up a linear regression problem. Given $x \in \mathbb{R}^d$ consider a $2m$ dimensional feature space $\phi(x)$ where $\phi(x)_{2i} = a_i\sigma(\xi_i w^\top_i x + \xi_i b_i)$ and $\phi(x)_{2i+1} = a_i\sigma(-\xi_i w^\top_i x - \xi_i b_i)$. Then is is easy to see that the target network $f(x)$ can be equivalently written as $f(x) = {\beta^*}^\top \phi(x)$ for some vector $\beta^*$. Hence we can recover another vector $\beta$ of zero $L_2$ error by solving ordinary least squares, i.e, $\beta = \Esymb[\phi(x)^\top \phi(x)]^\dagger \Esymb[\phi(x) y]$.\footnote{We remark that using Claim~\ref{claim:consolidate}, we can further consolidate the terms to get a ReLU network with at most $m+2$ hidden units. } Notice that both the expectations can be calculated exactly given exact statistical query access to the data distribution. In Section~\ref{sec:regression} we provide a more general analysis of the above argument with finite sample analysis that will let us approximate $f(x)$ up to arbitrary accuracy in the presence of sampling errors.
}    
\end{proof}

We now complete the proof of recovery in the full-rank setting. 
\begin{proof}[Proof of Theorem~\ref{thm:non-robust:fullrank}] We first apply Theorem~\ref{thm:non-robust:higher} (and its above proof) with $\ell=1$. We note that the conditions are satisfied since $\smin(W)>0$ and all the $a_i \ne 0$. Theorem~\ref{thm:non-robust:higher} guarantees that the first 5 steps of Algorithm~\ref{alg1} recovers (with probability $1$) for each $i \in [m]$, $a_i, \widetilde{b}_i = \xi_i b_i$ and $\widetilde{w}_i =\xi_i w_i$ for some $\xi_i \in \set{+1, 1}$. From Lemma~\ref{lem:hermitecoeffs}, we have that 
$$ \hat{f}_1 = \sum_{i=1}^m a_i \Phi(b_i) w_i = \sum_{i=1}^m z^*_i a_i \widetilde{w}_i, ~~\text{for}~~ z^*_i = \xi_i \Phi(b_i) ~\forall i \in [m],$$ and  $\Phi(b_i)$ is the Gaussian CDF and restricted to $(0,1)$. Moreover the $\set{w_i: i \in [m]}$ are linearly independent. Hence there is a unique solution $(z^*_i : i \in [m])$ to the system of linear equations in the unknowns $(z_i : i \in [m])$ in step 6 of Algorithm~\ref{alg1}, and $\widetilde{\xi}_i = \xi_i$ as required. Hence $(a_i , \tilde{\xi}_i \widetilde{b}_i, \tilde{\xi}_i \widetilde{w}_i: i \in [m])$ are the true parameters of the network (up to reordering indices).  
\end{proof}

\noindent \textbf{Proof of Identifiability (Theorem~\ref{thm:identifiability}):}
%\begin{proof}
Theorem~\ref{thm:identifiability} follows by verifying that the conditions of Theorem~\ref{thm:non-robust:higher} hold for $\ell=m$. %We now verify that the conditions of Theorem~\ref{thm:non-robust:higher} hold for $\ell=m$. 
Conditions (i) and (iii) follow from the conditions of Theorem~\ref{thm:identifiability}. 

We now verify condition (ii). 
For a matrix $U$ with columns $\set{u_i: i \in [m]}$, the $\mathsf{krank}(U)$ (denoting the Kruskal-rank) is at least $k$ iff every $k$ of the $m$ columns of $U$ are linearly independent. Note that $\mathsf{krank}(U) \le \rank(U) \le m$. The $\mathsf{krank}$ increases under the Khatri-Rao product.    
\begin{fact} [Lemma A.4 of \cite{BCV14}] For two matrices $U, V$ with $m$ columns,  $\mathsf{krank}(U \odot V) = \min(\mathsf{krank}(U) + \mathsf{krank}(V) - 1, m)$.
\end{fact}

Let $U=W^\top$ (with $i$th column $w_i$). Note that since no two columns are parallel, $\mathsf{krank}(U)\ge 2$. 
By applying the above fact on matrix $M=U^{\odot m}$ with $i$th column $w_i^{\otimes m}$, we get that $\mathsf{krank}(M) = m$, as required. Hence, Theorem~\ref{thm:non-robust:higher} can be applied to recover for all $i \in [m]$, all the unknown $a_i$, and up to ambiguities in signs given by (unknown) $\xi_i \in \set{1,-1}$ the $b_i$ and $w_i$ as well (we recover $\xi_i b_i, \xi_i w_i$). 

For the second half of the claim, let $\xi_i \in \set{1,-1}~ \forall i \in [m]$ be any combination of signs.  Consider the solution $a'_i = a_i, w'_i = \xi_i w_i, b'_i = \xi_i b_i$, and let $g(x)$ represent the corresponding ReLU function given by these parameters. Note that the Hermite polynomial $He_t( \xi z)=\xi^t He_t(z)$ for all $\xi \in \set{\pm 1}$ and $z \in \R$ . Hence, the Hermite  coefficients of order at least $2$ are equal for $f$ and $g$ i.e., for all $t \ge 2$ 
$$\hat{g}_t = \sum_{i=1}^m (-1)^t a_i~ He_{t-2}(\xi_i b_i)\cdot  \frac{e^{-b_i^2/2}}{\sqrt{2\pi}} \cdot (\xi_i w_i)^{\otimes t} = \sum_{i=1}^m (-1)^t a_i~ He_{t-2}(b_i) \cdot \frac{e^{-b_i^2/2}}{\sqrt{2\pi}} \cdot w_i^{\otimes t} = \hat{f}_t. $$

Condition~\ref{eq:nonidentifiability} also implies that zeroth and first Hermite coefficients of $f,g$ are also equal. All the Hermite coefficients are hence equal (and the functions are squared-integrable w.r.t. the Gaussian measure for bounded $a_i$). Thus the two functions $f$ and $g$ being identical follows since Hermite polynomials form a complete orthogonal system. 

\qed

%\end{proof}
%%%%%%%%%%%%%%%%%%%%%%%%%%%%%%%%%%%%%%%%%%%%%%%%%%%%%%%%

\paragraph{Non-identifiability of signs when $\set{w_i: i \in [m]}$ are not linearly independent.}
\label{sec:nonidentifiability}

Theorem~\ref{thm:non-robust:fullrank} shows that when the $\set{w_i: i \in [m]}$ are linearly independent, the model is identifiable. The following claim shows that even in the special setting when the biases $b_i=0~\forall i \in [m]$, whenever the $(w_i : i \in [m])$ are linearly dependent, the model is non-identifiable (for appropriate $(a_i: i \in [m])$) because of ambiguities in the signs (Theorem~\ref{thm:non-robust:higher} also shows it is identifiable up to this sign ambiguity). 

\begin{claim}\label{claim:nonidentifiability}
Suppose $w_1, \dots, w_m$ are linearly dependent. Then there exists $a_1, \dots, a_m$ (not all $0$) and signs $\xi_1, \xi_2, \dots, \xi_m \in \set{\pm 1}$ with not all $+1$ such that the ReLU networks $f$ and $g$ defined as 
$$f(x)\coloneqq \sum_{i=1}^m a_i \sigma(w_i^\top x ), ~~g(x) \coloneqq \sum_{i=1}^m  a_i \sigma(\xi_i  w_i^\top x ) \text{ satisfy } f(x)=g(x) ~\forall x \in \R^d .$$ 
\end{claim}
\begin{proof}
Since $\set{w_i: i \in [m]}$ are linearly dependent, there exists $(\beta_i: i \in [m])$ which are not all $0$ such that $\sum_{i=1}^n \beta_i w_i =0$. Define $a_i = \beta_i$ for each $i \in [m]$, and let $\xi_i = -1$ if $\beta_i \ne 0$ and $\xi_i=1$ otherwise. Let $f(x)=\sum_{i=1}^m a_i \sigma(w_i^\top x)$ and $g(x)=\sum_{i=1}^m a_i \sigma(\xi_i w_i^\top x)$.  

From Lemma~\ref{lem:hermitecoeffs}, it is easy to verify that all the even Hermite coefficients are equal, and the odd Hermite coefficients for $\ell \ge 3$ are all $0$ since $b_i =0$. Moreover the $\ell=1$ order Hermite coefficients are equal since 
$$ \sum_{i=1}^m a_i w_i -  \sum_{i=1}^m a_i \xi_i w_i = \sum_{i=1}^m a_i (1-\xi_i)w_i = \sum_{i=1}^m 2\beta_i w_i =0 .$$ 
All the Hermite coefficients of $f$ and $g$ are equal (and the functions are also squared-integrable w.r.t. the Gaussian measure when the $a_i$ are bounded). As the Hermite polynomials form a complete orthogonal basis, the two ReLU network functions $f(x)$ and $g(x)$ are also equal.   This concludes the proof. 
\end{proof}

% !TEX root=main.tex

\section{Robustness Analysis}
\label{sec:robust}
% \paragraph{Challenges when there are sampling errors.}
% \anote{Maybe such a paragraph should be in the introduction and main results.}

In this section, we prove Theorem~\ref{thm:full-rank} and Theorem~\ref{thm:robust:higherorder} which give polynomial time and sample complexity bounds for our algorithms. 
In the previous section we showed that given oracle access to $\set{\hat{f}_k}$, we can recover the exact network parameters $a, b, W$ (or at least up to signs). In reality, we can only access polynomially many samples in polynomial time, and we will have sampling errors when estimating $\set{\hat{f}_k}$. Therefore, given data generated from the target network $(x_1, y_1), ..., (x_N, y_N)$, we will approximate $\hat{f}_k$ through the empirical estimator
\begin{equation}
    T_k = \frac{1}{N} \sum_{i=1}^N y_i He_k(x_i)
\end{equation}
Observe that $T_k$ is an unbiased estimator for $\hat{f}_k$.
We first show using standard concentration bounds that for any $\eta>0$, with $N\ge \poly(d^{k}, m, B, 1/\eta)$ samples, the empirical estimates with high probability satisfies $\forall \ell \in [k], ~ \norm{\xi_\ell}_F \coloneqq \|T_\ell - \hat{f}_\ell\|_F \leq \eta$ (see Appendix~\ref{app:samplingerrors}).  Hence for any constant $k$, with polynomial samples, we can obtain with high probability, estimates for the tensors $\set{T_0, T_1, \dots, T_k}$ that are accurate up to any desired inverse-polynomial error.

The main algorithm in the robust setting is Algorithm~\ref{alg3} described below, which approximately recovers the parameters for the activation units (up to signs) that do not have large positive bias. The guarantees are given in the following Theorem~\ref{thm:robust:uptosigns}. %We modify the algorithm (Alg.~\ref{alg1}) appropriately to handle the errors in the estimates $\set{T_k}$. 

\begin{theorem}
\label{thm:robust:uptosigns}
     Suppose $\ell \in \N$ be a constant, and $\epsilon>0$. If we are given $N$ i.i.d. samples as described above from a ReLU network $f(x)=a^\top \sigma(W^\top x+b)$ that is $B$-bounded 
 %and satisfies 
%  \begin{equation}
% (a)~~     \sigma_m(W) \ge \frac{1}{\poly(m,d)}, ~\qquad(b)~~ |a_i| \ge a_{\min} ~\forall i \in [m], ~\qquad (c) ~b_i \ge -b_{\min} ~\forall i \in [m].
%  \end{equation}
%and if $W$ satisfies $s_m(W^{\odot \ell})\ge 1/\poly(m,d)$, 
%\anote{Check if the explicit condition on $W$ can be removed.}
Then there are constants $c=c(\ell)>0, c'>0$, signs $\xi_i \in \set{\pm 1}~\forall i \in [m]$ and a permutation $\pi: [m] \to [m]$ such that Algorithm~\ref{alg3}  given $N\ge \poly_\ell(m,d,1/\epsilon,1/s_m(W^{\odot \ell}),B)$ runs in $\poly_\ell(N,m,d)$ time and with high probability 
%finds a ReLU network $\widetilde{f}(x)=\widetilde{a}^\top \sigma(\widetilde{W}^\top x+\widetilde{b})$ that $\epsilon$-approximates $f$. 
%i.e., $\forall x \in \R^d, ~ |f(x)-\widetilde{f}(x)| \le \epsilon \norm{x}_2$. 
% Furthermore there are constants $c=c(\ell)>0, c'>0$ 
outputs $\set{\widetilde{a}_i,\widetilde{b}_i, \widetilde{w}_i: i \in [m'] }$ such that for all $i \in [m]$ with $|b_i| < c \sqrt{\log(1/(\epsilon\cdot mdB))}$ we have that $\norm{w_i - \xi_{\pi(i)} \widetilde{w}_{\pi(i)}}_2+|a_i - \widetilde{a}_{\pi(i)}|+|b_i - \xi_{\pi(i)} \widetilde{b}_{\pi(i)}| \le \epsilon$. %\anote{as before, need to check the $c \sqrt{\log(1/\epsilon)}$ condition.}
\end{theorem}

In fact, the analysis just assumes that $\norm{T_k - \hat{f}_k}_F$ are upper bounded up to an amount that is inverse polynomial in the different parameters (this could also include other sources of error) to approximate the 2-layer ReLU network that approximates $f$ up to desired inverse polynomial error $\eps$.

\begin{algorithm}
\label{alg3}

\SetAlgoLined
\textbf{Input:} Estimates $T_0, \dots, T_{2\ell+2}$ for $ \hat{f}_0, \hat{f}_1, \dots, \hat{f}_{2\ell+2}$\;
\textbf{Parameters:} $\eta_0, \eta_1, \eta_2, \eta_3>0$.\;

%0. Set $\eta_1:=$, ...

1. Let order-3 tensors $T'=\flatten(T_{2\ell+1},\ell,\ell,1) \in \R^{d^{\ell} \times d^{\ell} \times d }$ and let $T''=\flatten(T_{2\ell+2},\ell,\ell,2) \in \R^{d^{\ell} \times d^{\ell} \times d^2}$. 

2. Set $k'=\max_{r \le m} \sing_r(\flatten(T_{2\ell+1},\ell, \ell+1,0))>\eta_1$. Run Jennrich's algorithm  on $T'$ to recover rank-1 terms $\set{\alpha'_i u_i^{\otimes \ell} \otimes u_i^{\otimes \ell} \otimes u_i ~|~ i \in [k'] }$, where $\forall i \in [k'],~ u_i \in \bbS^{d-1}$ and $\alpha'_i \in \R$. %, and use PCA on each rank-$1$ term to find $\set{u_i: i \in [k']}$

3. Set $k''=\max_{r \le m} \sing_r(\flatten(T_{2\ell+2},\ell, \ell+2,0))>\eta_1$. Run Jennrich's algorithm on $T''$ to recover rank-1 terms $\set{\alpha''_i v_i^{\otimes \ell} \otimes v_i^{\otimes \ell} \otimes v_i^{\otimes 2} ~|~ i \in [k''] }$, where $\forall i \in [k''],~ v_i \in \bbS^{d-1}$ and $\alpha''_i \in \R$. % and use PCA on each rank-$1$ term to find $\set{v_i: i \in [k'']}$

4. Remove all the rank-$1$ terms in steps 2 and 3 with Frobenius norm $<\eta_2$ i.e., $\alpha''_i$ or $\alpha'_i<\eta_2$. Also remove all duplicates from $\set{u_1, u_2, \dots, u_{k'}} \cup \set{v_1, v_2, \dots, v_{k''}}$ even up to signs i.e., remove iteratively from the above set vectors $v$ if either of $+v, -v$ are within $\eta_3$ in $\ell_2$ distance of the other vectors in the set, to get $\widetilde{w}_1, \widetilde{w}_2, \dots, \widetilde{w}_{m'}$.

5. Run the subroutine  \textsc{RecoverScalars}$(\ell, \set{\widetilde{w}_i: i \in [m']},T_\ell, T_{\ell+1}, T_{\ell+2}, T_{\ell+3})$ (i.e., Alg.~\ref{alg2}) to get $\set{\widetilde{a}_i, \widetilde{b}_i: i \in [m']}$.  

%6. If $\ell=1$ (full-rank setting), run Algorithm 3 ({\sc FixSigns}) on parameters $m, T_1$ and $(\widetilde{a}_i, \widetilde{b}_i, \widetilde{w}_i: i \in [m] )$ to get $(a'_i=\widetilde{a}_i,b'_i, w'_i: i \in [m])$. 
%6. Set $\widetilde{w}_0 \coloneqq T_1 - \sum_{i=1}^{m'} \widetilde{a}_i\Phi(\widetilde{b}_i)\widetilde{w}_i$, and $\widetilde{b}_0 \coloneqq T_0 - \sum_{i=1}^{m'} \widetilde{a}_i(\widetilde{b}_i\Phi(\widetilde{b}_i) + \exp(-\widetilde{b}_i^2/2)/\sqrt{2\pi})$.

\KwResult{Output $\set{\widetilde{w}_i, \widetilde{a}_i, \widetilde{b}_i :1 \le i \le m' }$.} %and $\set{\widetilde{w}_0, \widetilde{b}_0}$}

\caption{for order $\ell$: recover $\set{a_i}$, and (up to signs) $\set{b_i, w_i}$ given estimates $\set{T_{0}, \dots, T_{2\ell+2}}$.}

\end{algorithm}

The following algorithm (Algorithm~\ref{alg-regression}) shows how to find a depth-2 ReLU network that fits the data i.e., achieves arbitrarily small $L_2$ error. The algorithm uses Algorithm~\ref{alg3} as a black-box to first approximately recover the unknown parameters of the activation units (with not very large bias) up to signs, and then setup an appropriate linear regression problem to find a network that fits the data.

\begin{algorithm}
\label{alg-regression}

\SetAlgoLined
 \textbf{Input:} $N$ i.i.d. samples of the form $(x_i,y_i)$\; 
%  Estimates $T_0, \dots, T_{2\ell+2}$ for $ \hat{f}_0, \hat{f}_1, \dots, \hat{f}_{2\ell+2}$\;
\textbf{Parameters:} $\epsilon, \eta_0, \eta_1, \eta_2, \eta_3>0$.\;

%0. Set $\eta_1:=$, ...

1. Construct estimates $T_0, \dots, T_{2\ell+2}$ for $ \hat{f}_0, \hat{f}_1, \dots, \hat{f}_{2\ell+2}$ using the first $\frac{N}{2}$ samples.

2. Let $S = \{(\tilde{w}_i, \tilde{a}_i, \tilde{b}_i)\}$ be the output of Algorithm~\ref{alg3} on inputs $T_0, \dots, T_{2\ell+2}$ when run with parameters $\eta_0, \eta_1, \eta_2, \eta_3>0$.

3. For each $(x_i, y_i)$ and $i \in [N/2+1, N]$ construct the feature mapping $\phi(x_i) = (Z(x_i), Z'(x_i))$ as described in the proof of Lemma~\ref{lem:regression}.

4. Set $\tau = 20 m (8|S| + d) B \sqrt{\log(\frac{m d B|S|}{\epsilon})}$ and find, via projected gradient descent, a vector $\hat{\beta}$ such that
$$
\hat{L}_\tau(\hat{\beta}) \leq \min_{\beta: \|\beta\| \leq \sqrt{8 |S|} + m(1+B)} \hat{L}_\tau(\beta) + \frac{\epsilon^2}{100}.
$$
Here $\hat{L}_\tau(\beta)$ is defined as
$$
\hat{L}_\tau(\beta) = \frac{2}{N} \sum_{i=\frac{N}{2}+1}^N (y_i - \beta^\top \phi(x_i))^2 \mathds{1}\big(\|\phi(x_i)\| < \tau \big).
$$

\KwResult{The function $g(x) = \hat{\beta}^\top \phi(x)$.} %and $\set{\widetilde{w}_0, \widetilde{b}_0}$}

\caption{Outputs a function $g(x)$ that approximates the target network $f(x)$ in mean squared error.}

\end{algorithm}

%\vspace{-5pt}

The error parameters $\eta_0, \eta_1, \eta_2, \eta_3$ can be set with appropriate polynomial dependencies on $\eps, d^{\ell}, m, B, \sing_m(W^{\odot \ell})$ to obtain the recovery guarantees in Theorem~\ref{thm:robust:uptosigns} and Theorem~\ref{thm:robust:higherorder}. See Section~\ref{app:robust:final} for details.

%\anote{Write the algorithm?}

\paragraph{ Overview of Analysis.} The error in the tensors $T_k$ introduces additional challenges that we described in Section~\ref{sec:intro}. The analysis is technical and long, but we now briefly describe the main components. %how these are tackled. 

\noindent \emph{(i)} Recall from Section~\ref{sec:intro}, that when there are errors, it may not even be possible to recover the parameters of some ReLU units! In particular when the bias $b_i$ is large in magnitude, the ReLU unit will be indistinguishable from a simple linear function. % (if $b_i$ is negative with large magnitude the ReLU unit will always output $0$ w.h.p, and if $b_i$ is a large positive number it will always output $a_i w_i^\top x + b$ w.h.p). 
It will contribute negligibly to any of the higher order Hermite coefficients, and hence will be impossible to recover them individually (especially if there are multiple such units).
For a desired recovery error $\eps>0$, the $m$ hidden ReLU units are split into groups (for analysis)
$$    G=\set{i \in [m] ~|~ |b_i| < O\big(\sqrt{\log(1/(\eps  \cdot m d B))} \big)}, \text{ and } P=\set{1, 2, \dots, m} \setminus G.$$
We aim to recover all of the parameters of the units corresponding to $G$ up to signs. For the terms in $P$, we will show the existence of a linear function that approximates the total contribution from all the terms in $P$. 

\noindent \emph{(ii)} The tensor decomposition steps (steps 2-3) are simpler in the no-noise setting: the parameter $w_i$ of the $i$th ReLU unit can be recovered (up to sign ambiguity) as long its bias $b_i$ is not a root of $He_{2\ell-1}$. When there is noise,  there could be terms $i \in [m]$ for which $b_i$ are not roots of $He_{k-2}(x)$, and yet their signal can get swamped by the sampling error in the tensor. We can only hope to recover those $k \le m$ components whose corresponding coefficient is above some chosen threshold $\eta_1$ (the other terms are considered as part of the error tensor). However a technical issue that arises is that the robust recovery guarantees for tensor decomposition algorithms lose polynomial factors in different parameters including the least singular value ($\sing_k(\cdot)$) of the factor matrices. Hence, for each of step 2 and 3, we argue that recovery is possible only if the coefficient of the corresponding term is significantly large, and this may give reasonable estimates for only a subset of these $m$ terms with coefficients $>\eps$. 
 
\noindent \emph{(iii)} When decomposing two consecutive tensors $T_{2\ell+1}$ and $T_{2\ell+2}$, we use Lemma~\ref{lem:hermite_roots} to argue that each $i \in G$ will have a large coefficient in at least one of these two tensors. Hence we can stitch together estimates $\set{\widetilde{w}_i: i \in G}$ which are accurate up to a sign and small error. 
This will in turn be used to recover $a_i, b_i$ for $i \in G$, with properties of Hermite polynomials used to ensure that the errors do not propagate badly.  
 
\noindent \emph{(iv)} We argue that the other ReLU units in $P= [m]\setminus G$ can be approximated altogether using a linear function. This is obtained by subtracting from estimates $T_0, T_1$ with the corresponding terms from $G$. 

\noindent \emph{(v)} \newtext{ The above arguments let us compute good approximations to the parameters for the units in $G$, but only up to signs. In order to use this to learn a good predictor for $f(x)$ we consider solving a truncated linear regression problem in an expanded feature space. At a high level, for each $i \in G$, given estimates $(\widetilde{a}_i, \widetilde{w}_i, \widetilde{b}_i)$ we consider an expanded feature representation for this unit into an $8$-dimensional vector where each coordinate is of the form $\xi_{i_3} \widetilde{a}_{i} \sigma(\xi_{i_1}\widetilde{w}_i \cdot x + \xi_{i_2}\widetilde{b}_i)$ for $\xi_{i_1},\xi_{i_2}, \xi_{i_3} \in \{-1,+1\}$.\footnote{While this portion of the algorithm works more generally with ambiguities in the sign of $a_i$, $b_i$, and $w_i$, in our case, the sign ambiguity of $w_i$ and $b_i$ are coordinated, and the sign of the $a_i$ are also recovered correctly; hence a $2$-dimensional vector suffices in this case. Moreover the terms can be consolidated to get an equivalent ReLU network with at most $|G|+2$ hidden units (see Claim~\ref{claim:consolidate} and Lemma~\ref{lem:regression}).} Repeating this for every $i \in G$ it is easy to see that there is a linear function in the expanded space that approximates the part of the function $f(x)$ that depends on units in $G$. Combining with the previous argument that the units in $P = [m] \setminus G$ can be approximated by a linear function in the original feature space, we deduce that there is an $O(d+m)$ dimensional feature space where $f(x)$ admits a good linear approximation. We then solve a truncated least squares problem in this space to obtain our final function $g(x)$ that approximates $f(x)$ in $L_2$ error.}

%This also Finally, we also show in Section~\ref{app:badterms} that recovering the parameters up to inverse polynomial error also implies $\epsilon$-approximation to the ground-truth ReLU network.

In the following sections, we will state the main claims and intermediate steps that prove the robustness of the algorithm and establish Theorems~\ref{thm:robust:uptosigns}, \ref{thm:full-rank}, \ref{thm:robust:higherorder}.

\subsection{Estimating the Hermite Coefficients}

First, we derive concentration bounds on $\xi_k$, which will be followed by error bounds of the recovered parameters $\widetilde{a}, \widetilde{b}, \widetilde{W}$ in terms of $\xi_k$.

\begin{lemma}\label{thm:samplingerror}
    For any $\eta>0$, if $T_k$ is estimated from $N \ge c_k d^{k} m^2 B^4 \cdot \mathrm{polylog}(mdB/\eta)/ (\eta^2)$ samples,  then     for some constant $c_k>0$ that depends only on $k$, we have with probability at least $1-(mdB)^{-\log(md)}$,
    \begin{equation}
        \|T_k - \hat{f}_k\|_F \leq \eta.
    \end{equation}
%    for some constant $C$ such that $\mathbb{P}\set{\|x\| < 2\sqrt{d}} > 1 - o(1)$. 
\end{lemma}

%Although a simpler analysis using Chebyshev inequality suffices for Theorem~\ref{thm:nonoise:higherorder}, one can actually show a much 
The above statement follows by applying Rosenthal inequality~\cite{Rosenthal, Pinelis} along with Markov's inequality. Please refer to Section \ref{app:robust} for the proof of the theorem.

\subsection{Recovering the Parameters under Errors}

Suppose $\eps>0$ is the desired recovery error. 
The $m$ hidden units are split into groups
\begin{equation}\label{eq:def:G}
    G=\Big\{i \in [m] ~|~ |b_i| < c_\ell \sqrt{\log\Big(\frac{1}{\eps m d B \sing_m(W^{\odot \ell})}\Big)}\Big\}, \text{ and } P=\set{1, 2, \dots, m} \setminus G.
    \end{equation}
where $c_\ell$ is an appropriate constant that depends only on the constant $\ell>0$. Note that under the assumption that $\sing_m(W^{\odot \ell}) \ge 1/\poly(m,d,B)$ in Theorem~\ref{thm:robust:uptosigns}, this reduces to
$$    G=\Big\{i \in [m] ~|~ |b_i| < c'_\ell \sqrt{\log(1/\eps m d B)}\Big\}.$$
We aim to recover all of the parameters of the units corresponding to $G$. For the terms in $P$, we will learn a linear function that approximates the total contribution from all the terms in $P$.

\paragraph{Recovery of Weight Vectors $w_i$ for the Terms in $G$.}

We first state the following important lemma showing that Jennrich's algorithm run with an appropriate choice of rank $k$ will recover each large term up to a sign ambiguity.

\begin{lemma}
\label{app:lem:robust-recovery-jennrich}
 Suppose $\epsilon_2 \in (0,  \tfrac{1}{4})$, and $\ell_1, \ell_2 \ge \ell, \ell_3>0$ be constants for some fixed $\ell$, and  $T = \mathsf{flatten}(\hat{f}_{\ell_1+\ell_2+\ell_3}, \ell_1, \ell_2, \ell_3)$ have decomposition $T=\sum_{i=1}^m \lambda_i (u_i \otimes v_i \otimes z_i)$ with $\lambda_i \in \mathbb{R}$ and unit vectors $u_i=w_i^{\otimes \ell_1} \in \mathbb{R}^{d^{\ell_1}}, v_i=w_i^{\otimes \ell_2} \in \mathbb{R}^{d^{\ell_2}}, z_i=w_i^{\otimes \ell_3} \in \mathbb{R}^{d^{\ell_3}}$. There exists $\eta_1=\poly(\eps_2, \sing_m(W^{\odot \ell}))/\poly(m,d,B)>0$ and $\eps_1 \coloneqq \max\set{2\eps_2, 1/\poly(1/\eps_2, 1/\sing_m(W^{\odot \ell}), d^{\ell_1+\ell_2+\ell_3},B)}$  such that 
    if %and constants $c_1, c_2>0$ such that given $\widetilde{T}$ satisfying %with $E=\widetilde{T}- T$ satisfying 
    \begin{align}
%        \sing_m(W) & \le \frac{1}{\kappa}, \qquad \text{and}\\
        \|T-\widetilde{T}\|_{F} &\leq \eta'_1 \coloneqq \min\Big\{\poly(\eps_2, \sing_m(W^{\odot \ell}))/\poly(m,d,B,1/\eta_1), \frac{\eta_1}{2}\Big\},  %\frac{\epsilon^{c_1} \cdot \norm{T}_F}{\mathsf{poly}(\sing_m(W^{\odot \ell}), (md)^{\ell_1+ \ell_2+ \ell_3},B)} 
    \end{align}
then Jennrich's algorithm runs with rank $k'\coloneqq \argmax_{r \le m} \sing_r(\flatten(\widetilde{T}, \ell_1, \ell_2+\ell_3, 0))>\eta_1$ and w.h.p. outputs\footnote{
Note that one can also choose to pad the output with zeros to output $m$ sets of parameters instead of $k'$ if required. 
} $\set{\widetilde{\lambda}_i, \widetilde{w}_i}_{i \in [k']}$ such that there exists a permutation $\pi:[m] \to [m]$ and signs $\xi_i \in \set{1,-1}~\forall i \in [m]$  %and $\eps_1>0$ with  $\eps_1 \coloneqq 1/\poly(1/\eps_2, 1/\sing_m(W^{\odot \ell}), d^{\ell_1+\ell_2+\ell_3},B)$  satisfying $\eps_1> 2 \eps_2$, 
satisfying:
\begin{align}
(i)~~ \forall i \in [m],& ~~ |\lambda_i-\widetilde{\lambda}_{\pi(i)}|\le \eps_2^2, \text{ and }\\
    (ii)~~\forall i \in [m],& \text{ s.t. } |\lambda_i|>\epsilon_1, ~\text{we have}~ \norm{w_i^{\otimes t} - \xi_{\pi(i)}^t \widetilde{w}_{\pi(i)}^{\otimes t}}_2 \le \eps_2,~~~ \forall t \in [2\ell]. 
\end{align}
\end{lemma}

A direct application of Lemma \ref{app:lem:robust-recovery-jennrich} establishes the following claim, showing that we can recover all the weight vectors $w_i$ for each term $i \in G$ up to a sign ambiguity.  

\begin{lemma}\label{lem:robust:W_recovery}
For any $\eps_2>0$, there exists an $\eta_2' = \frac{\poly(\eps_2, \sing_m(W^{\odot \ell}))}{\poly_\ell(m,d,B)}>0 $ such that if the estimates $\norm{T_k - \hat{f}_k}_F \le \eta_2'$ for all $k \in \set{0,1,\dots, 2\ell+2}$, then steps 1-4 of Algorithm \ref{alg3} finds a set $\set{\widetilde{w}_i : i \in [m']}$ such that  there exists a one-to-one map $\pi: [m'] \to [m]$ satisfying {\em (i)} every $i \in G$ has a pre-image in $\pi$ (i.e., every term in $G$ is recovered), and for appropriate signs $\set{\xi_i \in \set{1,-1}: i \in [m']}$, 
\begin{equation}\label{eq:wrecovery}
    \forall i \in [m'], \forall t \in [2\ell], ~~ \norm{\xi_i^t \widetilde{w}_{i}^{\otimes t} - w_{\pi(i)}^{\otimes t}}_F \le \eps_2.  
\end{equation}
In particular $\forall i \in [m']$, we have $\norm{\xi_i \widetilde{w}_{i} - w_{\pi(i)}}_2 \le \eps_2$.  
\end{lemma}

\paragraph{Recovering Error for the Parameters $a_i, b_i$ with $i \in G$.}

The following lemmas now proves the recovery for each $i \in G$, the $a_i$ (no sign ambiguity) and the $b_i$ up to the same sign ambiguity as in $w_i$ (and in fact, this holds for all the terms output in steps 1-5 of Algorithm \ref{alg3}).%.  all the $\set{a_i, b_i : i \in G}$ 

\begin{lemma}\label{lem:robust:ab_recovery}
%\begin{lemma}\label{lem:robust:ab_recovery}
For $\eps>0$ in the definition of $G$ in \eqref{eq:def:G}, there exists $\eta_3'=\frac{\poly(\eps, \sing_m(W^{\otimes \ell}))}{\poly_\ell(m,d,B)}>0$, and $\eps_3'= \frac{\poly(\eps, \sing_m(W^{\otimes \ell}))}{\poly_\ell(m,d,B)}>0$ 
such that for some $\xi_i \in \set{\pm 1}~\forall i \in [m']$
$$ \text{if}~~\norm{T_k - \hat{f}_k}_F \le \eta_3'~~ \forall k \le 2\ell+2, ~~~~\text{ and }   ~~~\norm{\widetilde{w}_i^{\otimes t} - \xi^t w_i^{\otimes t}}_F \le \eps_3', \forall i \in [m'], \forall \ell \le t \le \ell+3.$$ 
then steps 5-6 of the algorithm finds $(\widetilde{a}_i, \widetilde{b}_i: i \in [m'])$ such that  
\begin{equation}\label{eq:wrecovery}
|\widetilde{a}_i - a_i | \le \eps, \text{ and } |\widetilde{b}_i - \xi_i b_i| \le \eps.  
\end{equation}

%\end{lemma}
\end{lemma}
Note that in Lemma~\ref{lem:robust:W_recovery} we showed that $G$ is contained in the $m'$ terms output in steps 1-4 (and hence step 5 as well). 
%%%% Robust version of Lemmas
\newcommand{\tempa}{\alpha}
The above uses the following two lemmas which gives a robust version of Lemma~\ref{lem:recovscalars} when there are errors in the estimates. 
We remark that $\beta=\tempa e^{-z^2/2}$ in the notation of Lemma~\ref{lem:recovscalars}. %This lemma is used to estimate the $a_i, b_i$ once we have recovered the $w_i$.  
% \begin{lemma}[Robust version of Lemma~\ref{lem:recovscalars}]\label{lem:robust:recovscalars}
% Suppose $k \in \N, k \ge 2, B \ge 1$, and for some unknowns $\tempa,z \in \R$.  There exists a constant $c_k=c(k)\ge 1$ such that for any $\eps \in (0,\tfrac{1}{4})$ satisfying (i) $|\tempa| \in [\tfrac{1}{B}, B]$ and $|z| \le B$, and (ii) $|z| <  2\sqrt{\log(c_k/ \eps^{1/4} (1+B)^3))}$ \anote{simplify this?}, if we are given values $\gamma_k,\gamma_{k+1},\gamma_{k+2}, \gamma_{k+3}$  s.t. $$ \Big| \gamma_j-  \tempa \cdot \frac{(-1)^j}{\sqrt{2\pi}} e^{-z^2/2} He_j(z) \Big| \le \eps'=\frac{\eps^{4}}{2(1+B)^2} ~~~ \forall j \in \set{k, k+1, k+2, k+3},$$
% then we can get estimates $\widetilde{z}, \widetilde{\tempa}$ as follows:   
% \begin{align}
%   \text{For } q:= \argmax_{\substack{j \in \set{k+1, k+2}}} |\gamma_{j}|, ~\text{ set }~ \widetilde{z}=  -\frac{\gamma_{q+1}+q 
%   \cdot \gamma_{q-1}}{\gamma_{q}}, ~~~~\widetilde{\tempa}&= (-1)^{q} \frac{\sqrt{2\pi} \gamma_{q}}{e^{-\frac{\widetilde{z}^2}{2}}He_{q}(\widetilde{z})}  \nonumber\\
%     \big| \widetilde{z}-z \big| \le \frac{\eps|z|}{B+1}  \le \eps, ~~\text{ and }~~\big| \widetilde{\tempa}-\tempa \big| &\le \eps. %\label{eq:ab_assgt}
% \end{align}
% \end{lemma}

\begin{lemma}[Robust version of Lemma~\ref{lem:recovscalars} for $k \ge 2$]\label{lem:robust:recovscalars}
Suppose $k \in \N, k \ge 2, B \ge 1$, and $\tempa,z \in \R$ be unknown parameters.  There exists a constant $c_k=c(k)\ge 1$ such that for any $\eps \in (0,\tfrac{1}{4})$ satisfying (i) $|\tempa| \in [\tfrac{1}{B}, B]$ and $|z| \le B$, and (ii) $|z| <  2\sqrt{\log(c_k/ \eps^{1/4} (1+B)^3))}$ \anote{simplify this?}, if we are given values $\gamma_k,\gamma_{k+1},\gamma_{k+2}, \gamma_{k+3}$  s.t. for some $\xi \in \set{\pm 1}$, $$ \Big| \gamma_j-  \tempa \cdot \frac{(-1)^j}{\sqrt{2\pi}} e^{-z^2/2}  He_j(\xi z) \Big| \le \eps'=\frac{\eps^{4}}{2(1+B)^2} ~~~ \forall j \in \set{k, k+1, k+2, k+3},$$
then the estimates $\widetilde{z}, \widetilde{\tempa}$ obtained as:   
\begin{align}
\widetilde{z}=  -\frac{\gamma_{q+1}+q 
  \cdot \gamma_{q-1}}{\gamma_{q}} ~~\text{where } q:= \argmax_{\substack{j \in \set{k+1, k+2}}} |\gamma_{j}|, ~~\text{and}~~\widetilde{\tempa}&= (-1)^{q} \frac{\sqrt{2\pi} \gamma_{q}}{e^{-\frac{\widetilde{z}^2}{2}}He_{q}(\widetilde{z})}  \nonumber\\
%\end{align}
%Then $\widetilde{z}$ and $\widetilde{\alpha}$ satisfy
%\begin{align}
\text{satisfy}~~~
\big| \widetilde{z}- \xi z \big| \le \frac{\eps|z|}{B+1}  \le \eps, ~~\text{ and }~~\big| \widetilde{\tempa}-\tempa \big| &\le \eps. %\label{eq:ab_assgt}
\end{align}
\end{lemma}

The simpler variant of the above lemma (Lemma~\ref{lem:robust:recovscalars}) for $k=1$ which is used in the full-rank setting, follows a very similar analysis and is stated below. 
\begin{lemma}[Robust version of Lemma~\ref{lem:recovscalars} for $k =1$]\label{lem:robust:recovscalars:1}
Suppose $B \ge 1$, and $\tempa,z \in \R$ be unknowns.  There exists a constant $c\ge 1$ such that for any $\eps \in (0,\tfrac{1}{4})$ satisfying (i) $|\tempa| \in [\tfrac{1}{B}, B]$ and $|z| \le B$, and (ii) $|z| <  2\sqrt{\log(c/ \eps^{1/4} (1+B)^3))}$, if we are given values $\gamma_0,\gamma_1$  s.t. for some $\xi \in \set{\pm 1}$, 
$$\Big| \gamma_j-  \tempa \cdot \frac{(-1)^j}{\sqrt{2\pi}} e^{-z^2/2}  He_j(\xi z) \Big| \le \eps'=\frac{\eps^{4}}{2(1+B)^2} ~~~ \forall j \in \set{0,1},$$
%$$ \Big| \gamma_0-  \tempa \cdot \frac{e^{-z^2/2}}{\sqrt{2\pi}}   \Big| + \Big| \gamma_1 +  \tempa \cdot \frac{e^{-z^2/2}}{\sqrt{2\pi}}  He_j(\xi z) \Big| \le \eps'=\frac{\eps^{4}}{2(1+B)^2} \le \eps'=\frac{\eps^{4}}{2(1+B)^2},$$
then the estimates $\widetilde{z}, \widetilde{\tempa}$ obtained as:   
\begin{align}
\widetilde{z}&=  -\frac{\gamma_{1}}{\gamma_{0}} , ~~\text{ and }~~\widetilde{\tempa}=  \frac{\sqrt{2\pi} \gamma_{0}}{e^{-\frac{\widetilde{z}^2}{2}}} \nonumber \\
\text{satisfy}~~~    \big| \widetilde{z}- \xi z \big| \le \frac{\eps|z|}{B+1}  &\le \eps, ~~\text{ and }~~\big| \widetilde{\tempa}-\tempa \big| \le \eps. %\label{eq:ab_assgt}
\end{align}
\end{lemma}

\subsection{Learning Guarantees via Linear Regression}
\label{sec:regression}
In the previous sections we designed algorithms based on tensor decompositions that, given i.i.d. samples from a network $f(x) = \sum_{i=1}^m a_i \sigma(w^\top_i  x + b_i)$, can recover 
approximations (up to signs) for ``good units'', i.e, $G = \{i \in [m]:  |b_i| < O\big( \sqrt{\log (\frac{1}{\epsilon m d B})} \big) \}$. In this section we will show how to use these approximations to perform improper learning of the target network $f(x)$ via a simple linear regression subroutine. Our algorithm will output a functions of the form $g(x) = \sum_{i=1}^{m'} a'_i \sigma(w'^\top_i  x + b'_i) + w''^\top x + C$, where $m' \leq 8m$. In particular we will prove the following.
\begin{lemma}
\label{lem:regression}
Let $\epsilon > 0$ and $f(x) = \sum_{i=1}^m a_i \sigma(w^\top_i  x + b_i)$ be an unknown target network. Let $S$ be a given set of tuples of the form $(\widetilde{w}_i, \widetilde{b}_i, \widetilde{a}_i)$ with $\|\tilde{w}_i\| = 1$, such that for each $i \in G$, there exists $j \in S$, and $\xi_{j_1}, \xi_{j_2}, \xi_{j_3} \in \{-1,+1\}$,  such that $\|w_i - \xi_{j_1}\widetilde{w}_j \| \leq O(\frac{\epsilon}{m d B})$, $|b_i - \xi_{j_2}\widetilde{b}_j | \leq O(\frac{\epsilon}{m d B})$, and $|a_i - \xi_{j_3}\widetilde{a}_j | \leq O(\frac{\epsilon}{m d B})$. Then for any $\delta \in (0,1)$, given $N = \poly(m, d, B, \frac{1}{\epsilon}, \log(\frac 1 \delta))$ i.i.d. samples of the form $(x, y = f(x))$ where $x \sim N(0,I)$, there exists an algorithm~(Algorithm~\ref{alg-regression}) that runs time polynomial in $N$ and with probability at least $1-\delta$ outputs a network $g(x)$ of the form $g(x) = \sum_{i=1}^{m'} a'_i \sigma(w'^\top_i  x + b'_i) + w''^\top x + C$, where $m' \leq 8 |S|$, such that
$$
\Esymb_{x \sim \calN(0,I_{d \times d})} \big(f(x) - g(x) \big)^2 \leq \epsilon^2.
$$
Furthermore, when $\xi_{j_1} = \xi_{j_2}$ and $\xi_{j_3}=+1$ for all $j \in S$  (i.e., the sign ambiguity  of $w_i$ and $b_i$ are the same, and there is no ambiguity in the sign of $a_i$ for all $i \in G$), then the number of hidden units in $g(x)$ is at most $|S|+2$.
\end{lemma}

While the above lemma is more general, when it is applied in the context of Theorem~\ref{thm:robust:higherorder} it satisfies the conditions of the ``furthermore'' portion of the lemma.  
Our algorithm for recovering $g(x)$ will set up a linear regression instance in an appropriate feature space. In order to do this we will need the lemma stated below that shows that there is a good linear approximation for the units not in $G$, i.e., $P = [m]\setminus G$. 
\begin{lemma}[Approximating $f_P$]\label{lem:approxfP}
Let $c > 2$ be a fixed constant. Consider $f_P(x) = \sum_{i=1}^m a_i \sigma(w^\top_i x + b_i) \mathds{1}\big(|b_i| \geq c \sqrt{\log(\frac{1}{\epsilon m d B})} \big)$. Then there exists a function $g_P(x) = \beta^\top_P x + C_P$ where $\|\beta_P\| \leq m B$ and $|C_P| \leq mB^2$ such that for a constant $c'>0$ that depends on $c$, 
$$
\Esymb_{x \sim \calN(0,I)}[f_P(x) - g_P(x)]^2 = c' \epsilon^2.
$$
%where $c'$ is a constant that depends on $c$.
\end{lemma}

% \begin{algorithm}
% \label{alg-regression}

% \SetAlgoLined
%  \textbf{Input:} $N$ i.i.d. samples of the form $(x_i,y_i)$\; 
% %  Estimates $T_0, \dots, T_{2\ell+2}$ for $ \hat{f}_0, \hat{f}_1, \dots, \hat{f}_{2\ell+2}$\;
% \textbf{Parameters:} $\epsilon, \eta_0, \eta_1, \eta_2, \eta_3>0$.\;

% %0. Set $\eta_1:=$, ...

% 1. Construct estimates $T_0, \dots, T_{2\ell+2}$ for $ \hat{f}_0, \hat{f}_1, \dots, \hat{f}_{2\ell+2}$ using the first $\frac{N}{2}$ samples.

% 2. Let $S = \{(\tilde{w}_i, \tilde{a}_i, \tilde{b}_i)\}$ be the output of Algorithm~\ref{alg3} on inputs $T_0, \dots, T_{2\ell+2}$ when run with parameters $\eta_0, \eta_1, \eta_2, \eta_3>0$.

% 3. For each $(x_i, y_i)$ and $i \in [N/2+1, N]$ construct the feature mapping $\phi(x_i) = (Z(x_i), Z'(x_i))$ as described in the proof of Lemma~\ref{lem:regression}.

% 4. Set $\tau = 20 m (8|S| + d) B \sqrt{\log(\frac{m d B|S|}{\epsilon})}$ and find, via projected gradient descent, a vector $\hat{\beta}$ such that
% $$
% \hat{L}_\tau(\hat{\beta}) \leq \min_{\beta: \|\beta\| \leq \sqrt{8 |S|} + m(1+B)} \hat{L}_\tau(\beta) + \frac{\epsilon^2}{100}.
% $$
% Here $\hat{L}_\tau(\beta)$ is defined as
% $$
% \hat{L}_\tau(\beta) = \frac{2}{N} \sum_{i=\frac{N}{2}+1}^N (y_i - \beta^\top \phi(x_i))^2 \mathds{1}\big(\|\phi(x_i)\| < \tau \big).
% $$

% \KwResult{The function $g(x) = \hat{\beta}^\top \phi(x)$.} %and $\set{\widetilde{w}_0, \widetilde{b}_0}$}

% \caption{Outputs a function $g(x)$ that approximates the target network $f(x)$ in mean squared error.}

% \end{algorithm}
We first establish the main result assuming the lemma above and provide a proof of the lemma at the end of the subsection.
\begin{proof}[Proof of Lemma~\ref{lem:regression}]
In order to find the approximate network $g(x)$ we will set up a linear regression problem in an appropriate feature space. We begin by describing the construction of the feature space and showing that there does indeed exist a linear function in the space that approximates $f(x)$. We first focus on the terms in the set $G$, i.e., 
$$
f_G(x) = \sum_{i \in G} a_i \sigma(w^\top_i  x + b_i).
$$
In order to approximate $f_G(x)$ we create for each $(\widetilde{w}_j, \widetilde{b}_j, \widetilde{a}_j) \in S$, eight features $Z_{j,1}, \dots, Z_{j,8}$ where each feature is of the form $\xi_{j_3} \widetilde{a}_j \sigma(\xi_{j_1} \widetilde{w}^\top_j  x + \xi_{j_2} \widetilde{b}_j)$ for $\xi_{j_1}, \xi_{j_2}, \xi_{j_3} \in \{-1,+1\}$. Consider a particular $i \in G$. Since the set $S$ consists of a good approximation $(\widetilde{w}_j, \widetilde{b}_j, \widetilde{a}_j)$ for the unit $i$, it is easy to see that one of the eight features corresponding to $Z_{j,:}$ approximates the $i$th unit well (by matching the signs appropriately). In other words we have that there exists $r \in [8]$ such that
\begin{align}
|Z_{j,r} - a_i \sigma(w^\top_i  x + b_i)| &= |{a'_j} \sigma({w'_j}^\top  x + {b'}_j) - a_i \sigma(w^\top_i \cdot x + b_i)|\\
&\leq |(a'_j-a_i) \sigma(w^\top_i  x + b_i)| + |a'_j \big( \sigma({w'}^\top_j  x + b'_j) - \sigma(w^\top_i  x + b_i)\big)|\\
&\leq O\Big(\frac{\epsilon}{m d B}\Big) (B + \|x\|) + O\Big(\frac{\epsilon}{m d B}\Big) + O\Big(\frac{\epsilon}{m d B}\Big) \|x\|.
\end{align}
Noting that $\Esymb[\|x\|^2] = d$ we get that there exists a vector $\beta^*_1$ with $\|\beta^*_1\|_2 \leq \sqrt{8 |S|}$ in the feature space $Z(x)$ defined as above such that
\begin{align}
\Esymb\big[f_1(x) - {\beta^*_1}^\top  Z(x)\big]^2 \leq \frac{\epsilon^2}{100}.
\end{align}
To approximate terms not in $G$, i.e., $f_P(x) = \sum_{i \notin G} a_i \sigma(w^\top_i  x + b_i)$, we use Lemma~\ref{lem:approxfP} to get that there exists a vector $\beta^*_2$ with $\|\beta^*_2\| \leq m(1+B)$ in the $Z' = (x,1)$ feature space such that
\begin{align}
\Esymb\big[f_P(x) - {\beta^*_2}^\top  Z'(x)\big]^2 \leq \frac{\epsilon^2}{100}.
\end{align}
Combining the above and noting that $y = f(x) = f_G(x) + f_P(x)$, we get that there exists a vector $\beta^*$ in the $\phi(x) = (Z(x), Z'(x))$ feature space with $\|\beta^*\| \leq \sqrt{8 |S|} + m(1+B)$such that
\begin{align}
\Esymb\big[y - {\beta^*}^\top  \phi(x)\big]^2 \leq \frac{\epsilon^2}{20}.
\end{align}

In order to approximate $\beta^*$ we solve a truncated least squares problem. In particular, define the truncated squared loss $L_\tau(\beta) = \Esymb [(y - \beta^\top  \phi(x))^2 \mathds{1}({\|\phi(x)\| < \tau})]$. Furthermore we define the empirical counter part $\hat{L}_\tau(\beta)$ based on $N$ i.i.d. samples drawn from the distribution of $\phi(x)$.
 For an appropriate value of $\tau$ we will output $\hat{\beta}$ such that
\begin{align}
\label{eq:approximate-regression-guarantee}
\hat{L}_\tau(\hat{\beta}) &\leq \min_{\beta: \|\beta\| \leq \sqrt{8 |S|} + m(1+B)} \hat{L}_\tau(\beta) + \frac{\epsilon^2}{100}.
\end{align}
In particular we will set $\tau = 20 m (8|S| + d) B \sqrt{\log(\frac{m d B|S|}{\epsilon})}$.  Notice that the empirical truncated loss above is convex and for the chosen value of $\tau$, has gradients bounded in norm by $\poly(m,d,B,|S|, \frac{1}{\epsilon})$. Hence we can use the projected gradient descent algorithm~\cite{boyd2004convex} to obtain a $\hat{\beta}$ that achieves the above guarantee in $N \cdot \poly(m,d,B,|S|, \frac{1}{\epsilon})$ time. Furthermore using standard uniform convergence bounds for bounded loss functions~\cite{mohri2018foundations} we get that if $N = \poly(m,d,B, |S|, \frac{1}{\epsilon}, \log(\frac{1}{\delta}))$ then with probability at least $1-\delta$ we have
\begin{align}
L_\tau(\hat{\beta}) &\leq \min_{\beta: \|\beta\| \leq \sqrt{8 |S|} + m(1+B)} {L}_\tau(\beta) + \frac{\epsilon^2}{50}\\
&\leq {L}_\tau(\beta) + \frac{\epsilon^2}{50}.
\end{align}

Finally, it remains to relate the truncated loss $L_\tau (\beta)$ to the true loss $L(\beta) = \Esymb [y - \beta^\top  \phi(x)]^2$. We have that for any $\beta$ such that $\|\beta\|_2 \leq \sqrt{8 |S|} + m(1+B)$,
\begin{align}
|L_\tau (\beta) - L(\beta)| &= \Esymb[(y -  \beta^\top  \phi(x))^2 \mathds{1}(\|\phi(x) \| \geq \tau)].
\end{align}
Next notice that if $\|\phi(x)\| \geq 2^j \tau$ then we must have that either $|\tilde{a}_j\sigma(\tilde{w}^\top_j \cdot x + \tilde{b}_j)| \geq \frac{2^j \tau}{8|S| + d}$ or that for some $i \in [d]$, $|x_i| \geq \frac{2^j \tau}{8|S| + d}$. For our choice of $\tau$, this probability is bounded by $(8|S| + d) e^{-2^{2j}\Omega(\log(\frac{m d B|S|}{\epsilon}))}$. Hence we get that
\begin{align}
|L_\tau (\beta) - L(\beta)| &= \Esymb[(y -  \beta^\top  \phi(x))^2 \mathds{1}(\|\phi(x) \geq \tau\|)]\\
&= \sum_{j=0}^\infty \Esymb[(y -  \beta^\top  \phi(x))^2 \mathds{1}(\|\phi(x)\| \in [2^j \tau, 2^{j+1}\tau))]\\
&\leq \sum_{j=0}^\infty O(2^{2j} m^2 (8|S| + d)^2 \tau^2)(8|S| + d) e^{-2^{2j}\Omega(\log(\frac{m d B|S|}{\epsilon}))}\\
&\leq \frac{\epsilon^2}{50}.
\end{align}
Hence, the output network $g(x) = \hat{\beta} \cdot \phi(x) = \sum_{i=1}^{m'} a'_i \sigma(w'^\top_i x + b'_i) + w''^\top x + C$ satisfies with probability at least $1-\delta$ that
$$
\Esymb_{x \sim \calN(0,I_{d \times d})} \big(f(x) - g(x) \big)^2 \leq \epsilon^2.
$$
Notice that since a linear function can be simulated via two ReLU units (see Claim~\ref{claim:consolidate}), our output function $g(x)$ is indeed a depth-2 neural network with $m'+2 \leq 8m$ hidden units. 

Furthermore, while the statement of Lemma~\ref{lem:regression} assumes that the signs of units in $G$ are completely unknown, the output of the tensor decomposition procedure from Theorem~\ref{thm:robust:higherorder} in fact recovers, for each $i \in G$, the signs of $a_i$ exactly and the signs of the corresponding $(w_i, b_i)$ are either both correct or both incorrect. Hence when applying Lemma~\ref{lem:regression} to our application we only need to create two features for each unit in $G$. In other words we can output a network of the form $g(x) = {w^{''}}^\top x + C + \sum_{i=1}^{m'} a'_i \sigma(w'^\top_i x + b'_i) + a''_i \sigma(-w'^\top_i x - b'_i)$, where $m' \leq m$. Finally, from Claim~\ref{claim:consolidate}, the above network can be written as a depth-2 network with ReLU activations and at most $m+2$ hidden units. 
\end{proof}

We end the subsection with the proof of Lemma~\ref{lem:approxfP}.
\begin{proof}[Proof of Lemma~\ref{lem:approxfP}]
Consider a particular unit $i$ such that $b_i > c \sqrt{\log(\frac{1}{\epsilon m d B})}$. Then notice that $z_i = w^\top_i x + b_i \sim \calN(b_i, 1)$. By using standard properties of the Gaussian cdf, we get that by approximating $\sigma(z_i)$ by the linear term $z_i$ we incur the error
\begin{align}
    \Esymb_{z_i \sim \calN(b_i, 1)} (z_i - \sigma(z_i))^2 &= \frac{1}{\sqrt{2\pi}} \int_{-\infty}^0 z^2_i e^{-(z_i-b_i)^2/2} d z_i\\
    &\leq O(b^2_i) e^{-b^2_i/2} \leq O\Big(\frac{\epsilon^2}{m^2}\Big).
\end{align}
Similarly for a unit with $b_i < -c \sqrt{\log(\frac{1}{\epsilon m d B})}$, by approximating $\sigma(z_i)$ with the constant zero function we incur the error
\begin{align}
    \Esymb_{z_i \sim \calN(b_i, 1)} (\sigma(z_i))^2 &= \frac{1}{\sqrt{2\pi}} \int_{0}^\infty z^2_i e^{-(z_i-b_i)^2/2} d z_i\\
    &\leq O(b^2_i) e^{-b^2_i/2} \leq O\Big(\frac{\epsilon^2}{m^2}\Big).
\end{align}

Hence, each unit $i \in P$ with $z_i = w^\top_i x + b_i$ has a good linear approximation $\widetilde{z}_i =\widetilde{w}^\top_i x + \widetilde{b}_i$ of low error. Combining the above we get that
\begin{align}
\Esymb_{x \sim \calN(0, I_{d \times d})} \Big(\sum_{i \in P} a_i(z_i - \widetilde{z}_i) \Big)^2 &\leq c' \epsilon^2,
\end{align}
for a constant $c'$ that depends on $c$. Furthermore it is easy to see that the linear approximation $\sum_{i \in P} a_i\widetilde{z}_i$ is of the form $\beta^\top_P x + C_P$ where $\|\beta_P\| \leq \sum_{i\in P} |a_i| \|w_i\| \leq m B$ and $|C_P| \leq \sum_{i \in P} |a_i b_i| \leq m B^2$.
\end{proof}

\subsection{Wrapping up the proofs} \label{app:robust:final}

With the lemmas above, we can now complete the proof of Theorem~\ref{thm:full-rank}, Theorem~\ref{thm:robust:uptosigns} and Theorem~\ref{thm:robust:higherorder}.

\paragraph{Proof of Theorem \ref{thm:robust:uptosigns}}

We first set the parameters according to the polynomial bounds from the different lemmas in this section. 

For the final error $\eps$ in approximating $f$, we will set $\eps' \coloneqq \eps/(4mB)$. 
Also set $\eps'_3, \eta'_3$ according to Lemma~\ref{lem:robust:ab_recovery} with the $\eps$ in Lemma~\ref{lem:robust:ab_recovery} set to $\eps'$. 
Then set $\eps_2=\eps'_3$, and also set $\eps_1 = \sqrt{\eps_2}$. Now we can set the algorithm parameters $\eta_3 \coloneqq 2 \eps_2$, and $\eta_2 \coloneqq 4\eps_1$, and $\eta_0=\min\set{\eta_3', \eta'_2}$, where $\eta'_2$ is given by Lemma~\ref{lem:robust:W_recovery}.  Moreover $\eta_1$ (and $\eta_1'$) are set according to Lemma~\ref{app:lem:robust-recovery-jennrich}.

First by using Lemma \ref{thm:samplingerror} we see that with $\poly_\ell(d,m,B, 1/\sing_m(W^{\odot \ell}), 1/\eps)$ we can estimate all the Hermite coefficients up to $2\ell+2$ up to $\eta_0$ error in Frobenius norm. Then, for our setting of parameters  we have from Lemma~\ref{lem:robust:W_recovery} that for every $\widetilde{w}_i$ for $i \in [m']$ output by steps 1-4 of Algorithm \ref{alg3}, we have that there exists a $w_i$ (up to relabeling $i$) such that $\norm{w_i^{\otimes t}- \widetilde{w}_i^{\otimes t}}_F \le \eps_2=\eps'_3 < \eps'$. Then we can apply Lemma~\ref{lem:robust:ab_recovery} to conclude that for all such terms $i \in [m']$ that are output we get estimates $\widetilde{a}_i, \widetilde{b}_i$ with $|\widetilde{a}_i - a_i|+ |\widetilde{b}_i - b_i| \le \eps'$. Moreover using Lemma~\ref{lem:robust:W_recovery} and Lemma~\ref{lem:robust:ab_recovery} also show that every $i \in G$ is also one of the $m'$ terms that are output. Hence for each $i \in \widetilde{G}$, we have recovered each parameter up to error $\eps'$. This completes the proof.

% By Lemma \ref{lem:paramstoapprox} we can reconstruct $\widetilde{f}_{\widetilde{G}}(x) = \sum_{i \in \widetilde{G}} \widetilde{a}_i \sigma(\widetilde{w}_i^\top x + \widetilde{b}_i)$ such that $|f_{\widetilde{G}}(x) - \widetilde{f}_{\widetilde{G}}(x)| \leq \eps \|x\| + \eps$.

% On the other hand, by Lemma \ref{lem:approxfP} we use $T_0, T_1, \set{\widetilde{a}_i, \widetilde{b}_i, \widetilde{w}_i}_{i \in \widetilde{G}}$ to recover the parameters $\widetilde{w}_0, \widetilde{b}_0$ for $\widetilde{f}_P(x) = \sum_{i \in P} \widetilde{a}_i \sigma(\widetilde{w}_i^\top x + \widetilde{b}_i) =_{\text{w.h.p.}} \widetilde{w}_0^\top x + \widetilde{b}_0$ and approximate $f_P(x)$ as $|f_P(x) - \widetilde{f}_P(x)| \leq \xi \|x\| + \xi$, where $\xi$ is given by Lemma \ref{lem:approxfP} as well.

%Finally, let $\widetilde{\eps} = \eps + \xi$, then by triangle inequality $|f(x) - \widetilde{f}(x)| \leq |f_{\widetilde{G}}(x) - \widetilde{f}_{\widetilde{G}}(x)| + |f_P(x) - \widetilde{f}_P(x)| \leq (\eps + \xi)\|x\| + (\eps + \xi) = \widetilde{\eps}\|x\| + \widetilde{\eps}$. This concludes the proof.

\paragraph{Proof of the full-rank setting: Theorem~\ref{thm:full-rank}}

The guarantees for Theorem~\ref{thm:full-rank} hold for the following Algorithm~\ref{alg:full-rank-robust}, which is a robust variant of Algorithm~\ref{alg1} in the special case of $\ell=1$. It first uses Algorithm~\ref{alg3} to approximately recover for each $i \in [m]$,  the $a_i$, and up to an ambiguity in a sign (captured by unknown $\xi_i \in \set{1,-1}$) close estimates of $w_i$ and $b_i$. Then it runs Algorithm~\ref{alg:fixsigns} to disambiguate the sign by recovering $\xi_i$.    

\begin{algorithm}
\label{alg:full-rank-robust}

\SetAlgoLined
\textbf{Input:} Estimates $T_0, T_1, T_2, T_3, T_4$\;
\textbf{Parameters:} $\eta_0, \eta_1, \eta_2, \eta_3>0$.\;

%1. Let order-3 tensors $T'=\flatten(T_3,1,1,1) \in \R^{d \times d \times d}$ and let $T''=\flatten(T_4,1,1,2) \in \R^{d \times d \times d^2}$. 

%2. Set $k'=\max_{r \le m} \sing_r(\flatten(T_3,1,2,0)) > \eta_1$. Run Jennrich's algorithm  on $T'$ to recover rank-1 terms $\set{\alpha'_i u_i \otimes u_i \otimes u_i ~|~ i \in [k'] }$, where $\forall i \in [k'],~ u_i \in \bbS^{d-1}$ and $\alpha'_i \in \R$

%3. Set $k''=\max_{r \le m} \sing_r(\flatten(T_4,1,3,0)) > \eta_1$. Run Jennrich's algorithm on $T''$ to recover rank-1 terms $\set{\alpha''_i v_i \otimes v_i \otimes v_i^{\otimes 2} ~|~ i \in [k''] }$, where $\forall i \in [k''],~ v_i \in \bbS^{d-1}$ and $\alpha''_i \in \R$

%4. Remove all the rank-$1$ terms in steps 2 and 3 with Frobenius norm $<\eta_2$ i.e., $\alpha''_i$ or $\alpha'_i<\eta_2$. Also remove all duplicates from $\set{u_1, u_2, \dots, u_{k'}} \cup \set{v_1, v_2, \dots, v_{k''}}$ even up to signs i.e., remove iteratively from the above set vectors $v$ if either of $+v, -v$ are within $\eta_3$ in $\ell_2$ distance of the other vectors in the set, to get $\widetilde{w}_1, \widetilde{w}_2, \dots, \widetilde{w}_{m'}$. Note that $\widetilde{w}_i$ at this step is only recovered up to sign.

%5. Run the subroutine  \textsc{RecoverScalars}$(\ell, \set{\widetilde{w}_i: i \in [m']},T_1, T_2, T_3, T_4)$ (i.e., Alg.~\ref{alg2}) to get $\set{\widetilde{a}_i, \widetilde{b}_i: i \in [m']}$, where $\widetilde{b}_i$ is recovered only up to sign. 

1. Run Algorithm 4 on parameters $(\eta_0, \eta_1, \eta_2, \eta_3)$ with inputs $T_0, T_1, T_2, T_3, T_4$ to receive results $(\widetilde{w}_i, \widetilde{a}_i, \widetilde{b}_i)_{i \in [m]}$. Note that $\widetilde{b}_i, \widetilde{w}_i$ are only recovered up to signs.

2. Run Algorithm 3 ({\sc FixSigns}) on parameters $m, T_1$ and $(\widetilde{a}_i, \widetilde{b}_i, \widetilde{w}_i: i \in [m] )$ to recover $(\widetilde{a}_i,\widetilde{b}_i, \widetilde{w}_i: i \in [m])$.

\KwResult{Output $\set{\widetilde{w}_i, \widetilde{a}_i, \widetilde{b}_i :1 \le i \le m }$.} %and $\set{\widetilde{w}_0, \widetilde{b}_0}$}

\caption{Robust full-rank algorithm: recover $\set{a_i, b_i, w_i}$ given estimates $\set{T_0,T_1, \dots, T_4}$.}

\end{algorithm}

%Note that in Algorithm \ref{alg2}, the obtained $\zeta_2(i)$ for all $i \in [m]$ is guaranteed to be correct. However, $\zeta_3(i)$ is only correct up to sign, implying only $a_i$ will be correctly recovered. Hence in Algorithm \ref{alg:fixsigns} we leverage the fact that $W$ is full-rank to solve a linear system $\hat{f}_1 = \sum_{i=1}^m \widetilde{a}_i\widetilde{z}_i\widetilde{w}_i$ for the unknowns $\widetilde{z}_i$. Since  $\hat{f}_1 = \sum_{i=1}^m a_i\Phi(b_i)w_i$ and $\Phi(b_i)$ is non-negative, the sign of $\widetilde{z}_i$ will therefore be the correct sign of $\widetilde{w}_i$, thus recovering the sign of $\widetilde{b}_i$ as well.

\begin{proof}[Proof of Theorem \ref{thm:full-rank}]

We first set the parameters of Algorithm~\ref{alg:full-rank-robust} as dictated by Theorem~\ref{thm:robust:uptosigns} (and its proof) in the special case of $\ell=1$. Let $\eps_0>0$ be chosen so that $$\eps_0< \frac{\min\set{\Phi(-c \sqrt{\log(1/\eps m dB)}), \eps} \cdot  \sing_m(W)}{8\sqrt{m} B^2 }, $$ and $\eta_0$ to be the smaller of $\eps_0/ ((1+B)\sqrt{m})$, and whatever is specified Theorem~\ref{thm:robust:uptosigns} for $\eps=\eps_0$.  Note that $\Phi(-c \sqrt{\log(1/(\eps m d B)}) \ge \Omega\big((\eps m d B)^{c^2/2}\min\set{1,\epsilon mdB}\big) $.

We draw $N = \poly(d,m,B,1/s_m(W), 1/\eps_0)$ i.i.d. samples and run Algorithm~\ref{alg3} with the parameters $\eta_1, \eta_2, \eta_3$ as described in the proof of Theorem~\ref{thm:robust:uptosigns}. From the assumptions of Theorem~\ref{thm:full-rank}, we have that each $i \in [m]$ belongs to the ``good set'' $G$ as well. Hence, from the guarantee of Theorem~\ref{thm:robust:uptosigns} we will obtain w.h.p. for each $i \in [m]$ estimates $\widetilde{a}_i, \widetilde{b_i}, \widetilde{w}_i$ (up to relabeling the indices $[m]$) such that up an unknown sign $\xi_i \in \set{1,-1}$ we have 
\begin{equation}\label{eq:fullrank:1}
    |a_i - \widetilde{a}_i| + | \widetilde{b_i} - \xi_i b_i| + \norm{ \widetilde{w}_i - \xi_i w_i}_2 \le \eps_0 < \eps. 
\end{equation} 

Now consider the (ideal) linear system in the unknowns $\set{z_i: i \in [m]}$ given by $\widehat{f}_1 = \sum_{i=1}^m z_i (a_i \xi_i w_i) $; it has $d$ equations in $m\le d$ unknowns. Let $M \coloneqq W \diag((\xi_i a_i : i \in [m]))$ be a $d \times m$ matrix representing the above linear system as $Mz=\widehat{f}_1$. From Lemma~\ref{lem:hermitecoeffs}, $z^*_i = \xi_i \Phi(b_i)$ is a solution. Moreover $M$ is well-conditioned: since $|a_i| \in [1/B, B]$, we have $\sing_1(M) \le B \sing_1(W) \le B \sqrt{m}$, while $\sing_m(M) \ge \sing_m(W)/B$ (from the assumption on $W$). Hence, this is a well-conditioned linear system with a unique solution $z^*$. 

Algorithm~\ref{alg:fixsigns} solves the linear system $\widetilde{M}z = T_1$, where $\widetilde{M}= \widetilde{W} \diag(\widetilde{a})$; here each column of $\widetilde{M}$ is close to its corresponding column of $M$, while the sample estimate $T_1$ for $\widehat{f}_1$ satisfies $\norm{T_1 - \widehat{f}_1}_2 \le \eta_0$. Let $\widetilde{z}$ be a solution to the system $\widetilde{M}z = T_1$. 

Observe that if $\norm{z -\widetilde{z}}_\infty \le \|\widetilde{z} - z\|_2$ is at most $\min_i |z^*_i|$, then Algorithm 3 recovers the signs correctly, since  $\widetilde{z}$ will not flip in sign. To calculate this perturbation first observe that $i$th column of $E=\widetilde{M}-M$ has length at most 
\begin{align*}
    \norm{\widetilde{a}_i \xi_i \widetilde{w}_i - a_i w_i }_2 &\le |\widetilde{a}_i - a_i| \norm{\xi_i \widetilde{w}_i }_2 + a_i \norm{\xi_i \widetilde{w}_i - w_i }_2 \le \eps_0 + B \eps_0 \le \eps_0 (1+B). 
\end{align*}
Hence $\norm{E}_2=\sing_1(E) \le \eps_0 (1+B) \sqrt{m}$. Moreover by Weyl's inequality $\sing_{m}(\widetilde{M}) \ge \sing_m(M) - \norm{E} \ge \tfrac{1}{B}\sing_m(W) - \eps_0 (1+B) \ge \sing_m(W)/(2B)$ due to our choice of parameter $\eps_0$. 
From standard perturbation bounds for linear systems, we have
 \begin{align*}
        \norm{\widetilde{z}-z^*}_2 & \le \Big(\sing_{m}(\widetilde{M}) \Big)^{-1} \Big( \norm{T_1 - \widehat{f}_1}_2 + \sing_1(M - \widetilde{M}) \norm{z^*}_2\Big)\\
        &\le \frac{2B}{\sing_m(W)}\Big( \eta_0 + \eps_0(1+B)\sqrt{m}\Big) \le \frac{4\eps_0 B^2 m}{\sing_m(W)} \\
        & \le \Phi( - c \sqrt{\log(1/(\eps m d B))}) \le \frac{1}{2}\min_{i \in [m]} |z^*_i|
\end{align*} 
as required, due to our choice of $\eps_0$.  Hence the signs are also recovered accurately. This along with \eqref{eq:fullrank:1} concludes the proof.

\end{proof}

\paragraph{Proof of Theorem \ref{thm:robust:higherorder}} In order to establish Theorem~\ref{thm:robust:higherorder} we draw $N = \poly_\ell(d,m,B,1/s_m(W^{\odot \ell}), 1/\epsilon)$ i.i.d. samples and run Algorithm~\ref{alg-regression} with the parameters $\eta_0, \eta_1, \eta_2, \eta_3$ as described in the proof of Theorem~\ref{thm:robust:uptosigns}. From the guarantee of Theorem~\ref{thm:robust:uptosigns} we will obtain w.h.p., up to signs, approximations for all units in $G$ up to an error of $O(\frac{\epsilon}{m d B})$. Furthermore, given these approximations the guarantee of Lemma~\ref{lem:regression} tells us that w.h.p. the function $g(x)$ output by Algorithm~\ref{alg-regression} will satisfy $\Esymb_{x \sim \calN(0, I_{d \times d})} \big(f(x) - g(x) \big)^2 \leq \epsilon^2$.

\section{Smoothed Analysis} \label{app:smoothed} \label{app:smoothedresult}

%\anote{Just describe the setting, and result. Could be in appendix.\\ } 

We use the smoothed analysis framework of Spielman and Teng~\cite{ST04}, which is a beyond-worst-case-analysis paradigm that has been used to explain the practical success of various algorithms. In smoothed analysis, the performance of the algorithm is measured on a small random perturbation of the input instance. We use the model studied in the context of parameter estimation and tensor decomposition problems to obtain polynomial time guarantees under non-degeneracy conditions~\cite{BCMV14, AVbookchapter}. 
The smoothed analysis model for the depth-2 neural RELU network setting is as follows:
%which was initially introduced in \cite{ST09} to bridge the gap between theoretical worst-case analyses and typical real-world phenomena of the performance of an algorithm. It adopts an adversarial approach to measure how well an algorithm performs by the following procedure
%\anote{Edited the model. See if it makes sense? } 
\begin{enumerate}
    \item An adversary chooses set of parameters $a,b \in \mathbb{R}^m$ and $W \in \R^{d \times m}$.
    \item The weight matrix $\widehat{W} \in \R^{d \times m}$ is obtained by a small {\em random} i.i.d.  perturbation as $\widehat{W}_{ij}=W_{ij}+\xi_{i,j} ~\forall i \in [d], j \in [m]$ where $\xi_{i,j} \sim N(0,\tau^2/d)$. (Note that the average squared pertubation in each column is $\tau^2$) \footnote{Think of $\tau$ as a fairly small but inverse polynomial quantity $1/\poly(n,d)$.}.
    %every  $\xi \sim N(0, \frac{\tau^2}{d}I)$ for some constant $\tau$ and returns $\widetilde{a} = a + \xi$
    \item Each sample $(x,f(x))$ is drawn i.i.d. with $x \sim N(0, I_{d \times d})$ and $f(x)=a^\top \sigma(\widetilde{W}^\top x+b)$. 
    
\end{enumerate}
The goal is to design an algorithm that with high probability, estimates the parameters $a,b,\widehat{W}$ up to some desired accuracy $\epsilon$ in time $\poly(m,d,1/\epsilon,1/\tau)$. 
%
%\paragraph{Proof of Corollary~\ref{corr:smoothed}} 
%
We now prove the following corollary of Theorem~\ref{thm:robust:higherorder}.

\paragraph{Corollary \ref{corr:smoothed}}
\itshape
Suppose $\ell \in \N$ and $\epsilon>0$ are constants in the smoothed analysis model with smoothing parameter $\tau>0$, and also assume the ReLU network $f(x)=a^\top \sigma(\widehat{W}^\top x+b)$ is $B$-bounded with $m \le 0.99 \binom{d+\ell-1}{\ell}$. 
 %and satisfies 
%  \begin{equation}
% (a)~~     \sigma_m(W) \ge \frac{1}{\poly(m,d)}, ~\qquad(b)~~ |a_i| \ge a_{\min} ~\forall i \in [m], ~\qquad (c) ~b_i \ge -b_{\min} ~\forall i \in [m].
%  \end{equation}
 Then there is an algorithm that given $N\ge \poly_\ell(m,d,1/\epsilon,B, 1/\tau)$ samples runs in $\poly(N,m,d)$ time and with high probability finds a ReLU network $g(x)={a'}^\top \sigma({W'}^\top x+{b'})$ with at most $m+2$ hidden units such that the $L_2$ error $\E_{x \sim \calN(0,I_{d \times d})}[(f(x)-{g}(x))^2] \le \eps^2$. 
%i.e., $\forall x \in \R^d, ~ |f(x)-\widetilde{f}(x)| \le \epsilon \norm{x}_2$. 
 Furthermore there are constants $c, c'>0$ and signs $\xi_i \in \set{\pm 1}~\forall i \in [m]$, such that \newtext{in $\poly(N,m,d)$ time,} for all $i \in [m]$ with $|b_i| < c \sqrt{\log(1/(\epsilon\cdot mdB))}$, \newtext{we can recover} $(\widetilde{a}_i, \widetilde{w}_i, \widetilde{b}_i)$, such that $|a_i - \widetilde{a}_i|+ \norm{w_i - \xi_i \widetilde{w}_i}_2+|b_i - \xi_i \widetilde{b}_i| < c'\epsilon/(mB)$.
\upshape
%\end{corollary}

\begin{proof}

The proof of the corollary follows by combining Theorem~\ref{thm:robust:higherorder} 
 with existing results on smoothed analysis~\cite{BCPV} on the least singular value $\sing_m(\widehat{W}^{\odot \ell})$. We apply Theorem 2.1 of \cite{BCPV} with $\rho=\tau$, $U$ being the identity matrix to derive that for any $\delta>0$ and  $m \le (1-\delta) \binom{d+\ell-1}{\ell}$, we get with probability at least $1-m\exp(- \Omega_\ell(\delta n))$ that 
 $$\sing_m(\widehat{W}^{\odot \ell}) \ge \frac{c_\ell}{\sqrt{m}} \Big(\frac{\tau}{d} \Big)^{\ell} .$$
 
 We then just apply Theorem~\ref{thm:robust:higherorder} to conclude the proof.

  \end{proof}

\section{Conclusion}
\label{sec:conclusion}

In this paper, we designed polynomial time algorithms for  learning depth-2 neural networks with general ReLU activations (with non-zero bias terms), and gave provable guarantees under mild non-degeneracy conditions.  %efficiently and showed that not only can we closely approximate the target network, under mild assumptions most well-behaved parameters can be identified and recovered as well. 
The results of this work are theoretical in nature, in trying to understand whether efficient algorithms exist for learning ReLU networks; hence we believe they do not have any adverse societal impact. %\anote{Societal impact line?}
We addressed multiple challenges for learning such ReLU network with non-zero bias terms throughout our analyses, that may be more broadly useful in handling bias terms in the ReLU activations. We also proved identifiability under minimal assumptions and adopted the framework of smoothed analysis to establish beyond-worst-case guarantees.  The major open direction is to provide similar guarantees for networks of higher depth.  

\section{Acknowledgement}
We thank Ainesh Bakshi for pointing us to an error in the previous version of the paper.

\bibliographystyle{alpha}
\bibliography{library.bib}

\newpage

\appendix

\section{Expressions for the Hermite Coefficients} \label{app:sec:hermite}

\paragraph{Lemma \ref{lem:hermitecoeffs}}
\itshape
The $k$'th Hermite expansion of $f(x) = a^{\mathsf{T}}\sigma(W^\top x+b)$, $\hat{f}_k$, when $k = 0, 1$, is
\begin{equation}
    \hat{f}_0 = \sum_{i=1}^m a_i[b_i\Phi(b_i) + \frac{\exp(-\frac{b_i^2}{2})}{\sqrt{2\pi}}]
    ~,~
    \hat{f}_1 = \sum_{i=1}^m a_i \Phi(b_i) w_i
\end{equation}

when $k \ge 2$, the coefficients are
\begin{equation}
    \hat{f}_k
    %= \mathbb{E}[f(x)He_k(x)]
    = \sum_{i=1}^m (-1)^k \cdot a_i \cdot He_{k-2}(b_i) \cdot \frac{\exp(\frac{-b_i^2}{2})}{\sqrt{2\pi}} \cdot w_i^{\otimes k}
\end{equation}
where the expectation is taken over $x \sim\calN(0, I_d)$ and $\hat{f}_k$ is a $k$'th-order tensor.
\upshape

\begin{proof}
Note that Hermite polynomials can be written in terms of their generating function \cite{O14}
\begin{equation}
    He_k(x) = D^{(k)}_t \exp(t^{\mathsf{T}}x - \|t\|^2/2) |_{t=\mathbf{0}}
\end{equation}

Hence we can write $\hat{f}_k$ as
\begin{equation}
    \hat{f}_k = \mathbb{E}[f(x)He_k(x)]
    = \sum_{i=1}^m a_i \int_{w_ix + b_i \geq 0} (w_i^\mathsf{T}x+b_i) \cdot D^{(k)}_t \exp(t^{\mathsf{T}}x - \frac{\|t\|^2}{2}) ~ d\mu |_{t=\mathbf{0}}
\end{equation}

where $d\mu = \frac{\exp(-\|x\|^2/2)}{(\sqrt{2\pi})^d} dx$ is the Gaussian probability measure. Moving $d\mu$ into the exponential term, we get
\begin{equation}
    \hat{f}_k = \sum_{i=1}^m \frac{a_i}{(\sqrt{2\pi})^d} \int_{w_ix + b_i \geq 0} (w_i^\mathsf{T}x+b_i) \cdot D^{(k)}_t \exp(-\frac{\|x\|^2}{2}+t^{\mathsf{T}}x - \frac{\|t\|^2}{2}) ~ dx |_{t=\mathbf{0}}
\end{equation}
\begin{equation}
    = \sum_{i=1}^m \frac{a_i}{(\sqrt{2\pi})^d} \int_{w_ix + b_i \geq 0} (w_i^\mathsf{T}x+b_i) \cdot D^{(k)}_t \exp(-\frac{\|x-t\|^2}{2}) ~ dx |_{t=\mathbf{0}}
\end{equation}
\begin{equation}
    = D^{(k)}_t [ \sum_{i=1}^m \frac{a_i}{(\sqrt{2\pi})^d} \int_{w_ix + b_i \geq 0} (w_i^\mathsf{T}x+b_i) \cdot \exp(-\frac{\|x-t\|^2}{2}) ~ dx ]_{t=\mathbf{0}}
\end{equation}

Denote
\begin{equation}
    I_i(t) = \frac{1}{(\sqrt{2\pi})^d} \int_{w_ix + b_i \geq 0} (w_i^\mathsf{T}x+b_i) \cdot \exp(-\frac{\|x-t\|^2}{2}) ~ dx
\end{equation}

then
\begin{equation}
    \hat{f}_k = D^{(k)}_t \sum_{i=1}^m a_i I_i(t)|_{t=\mathbf{0}}
\end{equation}

Now, let $y_i = w_i^\mathsf{T}x$ and $t_{w_i} = t^\mathsf{T}w_i$. To evaluate $I_i(t)$ in terms of $y_i$, it suffices to only consider the projection of $t$ onto $w_i$, $t_{w_i}$ with the remaining parts being integrated out. Hence, we can rewrite $I_i(t)$ as
\begin{equation}
    I_i(t) = \frac{1}{(\sqrt{2\pi})^d} \int_{x^\prime \in \mathbb{R}^{d-1}} \exp(-\frac{\|x^\prime-t^\prime\|^2}{2}) dx^\prime ~ \int_{y_i=-b_i}^{\infty} (y_i+b_i) \cdot \exp(-\frac{\|y_i-t_{w_i}\|^2}{2}) dy_i
\end{equation}
\begin{equation}
    = \frac{1}{\sqrt{2\pi}} \int_{y_i=-b_i}^{\infty} (y_i+b_i) \cdot \exp(-\frac{\|y_i-t_{w_i}\|^2}{2}) dy_i
\end{equation}
\begin{equation}
    = (t_{w_i}+b_i)\Phi(t_{w_i}+b_i) + \frac{\exp(-\frac{(t_{w_i}+b_i)^2}{2})}{\sqrt{2\pi}}
\end{equation}

where $\Phi(z)$ is the standard Gaussian c.d.f. of $z$. We then have
\begin{equation}
    \hat{f}_k = \sum_{i=1}^m a_i \cdot D^{(k)}_t [(t_{w_i}+b_i)\Phi(t_{w_i}+b_i) + \frac{\exp(-\frac{(t_{w_i}+b_i)^2}{2})}{\sqrt{2\pi}}]_{t = \mathbf{0}}
\end{equation}

Therefore we have $\hat{f}_0$ and $\hat{f}_1$ as
\begin{equation}
    \hat{f}_0 = \sum_{i=1}^m a_i[b_i\Phi(b_i) + \frac{\exp(-\frac{b_i^2}{2})}{\sqrt{2\pi}}]
    ~,~
    \hat{f}_1 = \sum_{i=1}^m a_i \Phi(b_i) w_i
\end{equation}

Since we are taking the derivative with respect to $t$, for some function $g(t_{w_i})$, by the chain rule we will have
\begin{equation}
D^{(k)}_t g = \frac{d^k g}{dt_{w_i}^k} \cdot w_i^{\otimes k}
\end{equation}

Finally, recall Fact \ref{lem:gaussian-derivative}, the derivatives of a Gaussian p.d.f. can be expressed in terms of Hermite polynomials, hence for $k \geq 2$
\begin{equation}
    \hat{f}_k = \sum_{i=1}^m a_i \cdot \frac{d^{k-2}}{dt_{w_i}^{k-2}} \frac{\exp(-\frac{(b_i+t_{w_i})^2}{2})}{\sqrt{2\pi}} \cdot w_i^{\otimes k}|_{t_{w_i} = 0}
\end{equation}
\begin{equation}
    = \sum_{i=1}^m (-1)^{k-2} \cdot a_i \cdot He_{k-2}(b_i + t_{w_i}) \cdot \frac{\exp(-\frac{(b_i+t_{w_i})^2}{2})}{\sqrt{2\pi}} \cdot w_i^{\otimes k}|_{t_{w_i} = 0}
\end{equation}
\begin{equation}
    = \sum_{i=1}^m (-1)^k \cdot a_i \cdot He_{k-2}(b_i) \cdot \frac{\exp(\frac{-b_i^2}{2})}{\sqrt{2\pi}} \cdot w_i^{\otimes k}
\end{equation}

which proves the lemma.
\end{proof}

\paragraph{Proposition \ref{prop:lossobj}}
\itshape
Let $\widetilde{f}(x)=\widetilde{a}^\top \sigma(\widetilde{W}^\top x +\widetilde{b})$ be the model trained using samples generated by the ground-truth ReLU network $f(x) = a^\top \sigma(W^\top x + b)$. 
    %In addition, define $\alpha \bowtie \alpha'$ as $\alpha$ being a permutation of $\alpha'$ and $c_k^{\alpha} = \sqrt{\mathbb{E}[He_k^{\alpha}(x)^2]}$. 
%    \anote{The permutation definition and the $c_k$ definition doesn't need to be in the statement, and can go into the proof. }
    Then the statistical risk with respect to the $\ell_2$ loss function can be expressed as follows

    \begin{align*}
        L(\widetilde{a}, \widetilde{b}, \widetilde{W}) &= 
        \sum_{k \in \mathbb{N}}
        \frac{1}{k!} \Big\| T_k - \hat{f}_k \Big\|_F^2 \\
        ~~\text{ where }~~ T_0 &= \sum_{i=1}^m \widetilde{a}_i(\widetilde{b}_i\Phi(\widetilde{b}_i) + \frac{\exp(-\widetilde{b}_i^2/2)}{\sqrt{2\pi}}), 
        ~\text{ and }~ T_1 = \sum_{i=1}^m \widetilde{a}_i\Phi(\widetilde{b}_i)\widetilde{w}_i\\
        \forall k \geq 2,~~ T_k &= \sum_{i=1}^m (-1)^k \cdot \widetilde{a}_i \cdot He_{k-2}(\widetilde{b_i}) \cdot \frac{\exp(-\widetilde{b_i}^2/2)}{\sqrt{2\pi}} \cdot \widetilde{w}_i^{\otimes k}
    \end{align*}
\upshape
\begin{proof}
Let $\alpha \bowtie \alpha'$ denote $\alpha$ being a permutation of $\alpha'$. Since $\bowtie$ is an equivalence relation, we can partition $[d]^*$ into equivalence classes ($*$ is the Kleene star operator) such that for some $\alpha \in [d]^*$, $[\alpha] = \set{\alpha' \in [d]^* ~|~ \alpha \bowtie \alpha'}$. Let $C$ be a subset of $[d]^*$ such that no pair of $\alpha, \alpha' \in C$ is in the same equivalence class. We can then directly decompose the statistical risk as
    \begin{equation}
        \mathbb{E}\Big[ |\widetilde{f}(x) - f(x)|^2 \Big]
        =\mathbb{E} \Big[ \Big| \sum_{\alpha \in C}\frac{T_{\alpha} He_{\alpha}(x)}{c_{\alpha}^2} - \sum_{\alpha \in C}\frac{\hat{f}_{\alpha} He_{\alpha}(x)}{c_{\alpha}^2} \Big|^2 \Big]
    \end{equation}
    
    where $c_{\alpha}^2 = \mathbb{E}[He_{\alpha}(x)^2]$. Note that we omit $k$ here and directly write the Hermite polynomial obtained by differentiating with respect to $x_{\alpha_1}, ..., x_{\alpha_k}$ as $He_{\alpha}(x) \in \mathbb{R}$. The above equation can thus be further simplified as
    \begin{equation}
        \mathbb{E} \Big[ \sum_{\alpha \in C} \Big(\frac{(T_{\alpha} - \hat{f}_{\alpha}) He_{\alpha}(x)}{c_{\alpha}^2} \Big)^2 +
        \sum_{\substack{\alpha \ne \alpha' \\ \alpha, \alpha' \in C}} \Big(\frac{(T_{\alpha} - \hat{f}_{\alpha}) He_{\alpha}(x)}{c_{\alpha}^2} \Big) \Big(\frac{(T_{\alpha'} - \hat{f}_{\alpha'}) He_{\alpha'}(x)}{c_{\alpha'}^2} \Big)
        \Big]
    \end{equation}
    \begin{equation}
        =
        \sum_{\alpha \in C} \frac{(T_{\alpha} - \hat{f}_{\alpha})^2}{c_{\alpha}^2} \mathbb{E} \Big[ \frac{He_{\alpha}(x)^2}{c_{\alpha}^2}
        \Big]
        =
        \sum_{\alpha \in C} \frac{(T_{\alpha} - \hat{f}_{\alpha})^2}{c_{\alpha}^2}
    \end{equation}
    
    since if both $\alpha, \alpha' \in C$ and $\alpha \ne \alpha'$, $\mathbb{E}[He_\alpha(x) He_{\alpha'}(x)] = 0$. Next, we rewrite the expression as
    \begin{equation}
        \sum_{k \in \mathbb{N}} \sum_{\substack{\alpha \in C \\ |\alpha|=k}} \frac{(T_{\alpha} - \hat{f}_{\alpha})^2}{c_{\alpha}^2}
        =
        \sum_{k \in \mathbb{N}} \frac{1}{k!} \sum_{\substack{\alpha \in C \\ |\alpha|=k}} \frac{k!}{c_{\alpha}^2} (T_{\alpha} - \hat{f}_{\alpha})^2
        =
        \sum_{k \in \mathbb{N}} \frac{1}{k!} \Big \| T_k - \hat{f}_k \Big \|_F^2
    \end{equation}
    
    The last equality is due to the fact that $c_{\alpha}^2 = \prod_{i=1}^d n_i!$, where $n_i$ is the number of times that $i$ occurs in the multi-index $\alpha$, and therefore $k!/c_{\alpha}^2$ is the number of possible permutations of the elements in $\alpha$ with $|\alpha|=k$ subjecting to 1 occurs $n_1$ times, 2 occurs $n_2$ times, ..., $d$ occurs $n_d$ times. Thus the proposition follows.
\end{proof}

% !TEX root=main.tex

\section{Robust Analysis for general $\ell$}\label{app:robust}

In this section, we prove that Algorithm~\ref{alg3} and its algorithmic guarantee in Theorem~\ref{thm:robust:uptosigns} (and Theorems~\ref{thm:full-rank} and \ref{thm:robust:higherorder}). 
% \paragraph{Theorem \ref{thm:nonoise:higherorder}} \label{app:thm:nonoise:higherorder}
% \itshape
% %\begin{theorem}[Same as Theorem~\ref{thm:nonoise:higherorder}]
%      Suppose $\ell \in \N$ be a constant, and $\epsilon>0$. If we are given $N$ i.i.d. samples as described above from a ReLU network $f(x)=a^\top \sigma(W^\top x+b)$ that is $B$-bounded 
%  %and satisfies 
% %  \begin{equation}
% % (a)~~     \sigma_m(W) \ge \frac{1}{\poly(m,d)}, ~\qquad(b)~~ |a_i| \ge a_{\min} ~\forall i \in [m], ~\qquad (c) ~b_i \ge -b_{\min} ~\forall i \in [m].
% %  \end{equation}
% and if $W$ satisfies $s_m(W^{\odot \ell})\ge 1/\poly(m,d)$, then there is an algorithm that given $N\ge \poly_\ell(m,d,1/\epsilon,B)$ runs in $\poly_\ell(N,m,d)$ time and with high probability finds a ReLU network $\widetilde{f}(x)=\widetilde{a}^\top \sigma(\widetilde{W}^\top x+\widetilde{b})$ that $\epsilon$-approximates $f$. 
% %i.e., $\forall x \in \R^d, ~ |f(x)-\widetilde{f}(x)| \le \epsilon \norm{x}_2$. 
%  Furthermore there are constants $c=c(\ell)>0, c'>0$ such that for all $i \in [m]$ with $|b_i| < c \sqrt{\log(1/(\epsilon\cdot mdB))}$ we have that $\norm{w_i - \widetilde{w}_i}_2+|a_i - \widetilde{a}_i|+|b_i - \widetilde{b}_i| < c'\epsilon/(mB)$. %\anote{as before, need to check the $c \sqrt{\log(1/\epsilon)}$ condition.}
% \upshape
% %\end{theorem}

% \vspace{10pt}
We break down the proof into multiple parts. 

\subsection{Estimating the Hermite coefficients} \label{app:samplingerrors}

%In the following analysis, we first give concentration bounds on $\xi_k$, which will be followed by error bounds of the recovered parameters $\widetilde{a}, \widetilde{b}, \widetilde{W}$ in terms of $\xi_k$.

To obtain the desired concentration bound, we first introduce an auxiliary claim we will make use of in the following analysis.

\begin{claim} \label{app:clm:young}
For $a_1, a_2, ..., a_n \in \R$ and $p \in \mathbb{N}$, $|\sum_{i \in [n]} a_i|^{2p} \leq n^{2p} \max_{i \in [n]} |a_i|^{2p} \leq n^{2p} \sum_{i \in [n]} |a_i|^{2p}$
\end{claim}
\begin{proof}
    By triangle inequality,
    \begin{equation*}
        |\sum_{i \in [n]} a_i| \leq \sum_{i \in [n]} |a_i| \leq n \max_{i \in [n]} |a_i|
        \leq n \sum_{i \in [n]} |a_i|
        \Rightarrow
        |\sum_{i \in [n]} a_i|^{2p} \leq n^{2p} \max_{i \in [n]} |a_i|^{2p} \leq n^{2p} \sum_{i \in [n]} |a_i|^{2p}
    \end{equation*}
\end{proof}

Equipped with the essential claim, we are now ready to prove Lemma \ref{thm:samplingerror}.

\paragraph{Lemma \ref{thm:samplingerror}}
\itshape
    For any $\eta>0$, if $T_k$ is estimated from $N \ge c_k d^{k} m^2 B^4 \poly(\log(mdB/\eta))/ \eta^2$ samples,  then     for some constant $c_k>0$ that depends only on $k$, we have with probability at least $1-(mdB)^{-\log(md)}$,
    \begin{equation}
        \|T_k - \hat{f}_k\|_F \leq \eta.
    \end{equation}
\upshape
%    for some constant $C$ such that $\mathbb{P}\set{\|x\| < 2\sqrt{d}} > 1 - o(1)$. 

\begin{proof}
%A simpler analysis using Chebyshev inequality suffices for Theorem~\ref{thm:nonoise:higherorder}. But one can in fact show a much stronger concentration bound using the Rosenthal inequality~\cite{Rosenthal, Pinelis}. 
Consider $p \in \mathbb{N}$, and a sum $S_Y= \sum_{j=1}^N Y_j$ of independent zero-mean r.v.s with $\tfrac{1}{N}\sum_{j=1}^N \E[Y_j^{2p}] \le  A_{2p}$ and $\tfrac{1}{N}\sum_{j=1}^N \E[Y_j^2] \le  A_2$. Then by Rosenthal's inequality (and Markov's inequality)  
\begin{align}\label{eq:rosenthal}
    \E\Big[ \Big(\sum_{j=1}^N Y_j \Big)^{2p} \Big] &\le 2^{p \log(p)+2p+p^2} \cdot \max\set{ NA_{2p}, (N A_2)^{p}}\\
    \text{And, }
    \Pr\Big[ \Big| \frac{1}{N}\sum_{j=1}^N Y_j \Big| > \eta \Big] &\le 2^{p \log(p)+2p+p^2} \cdot \max\Big\{ \frac{A_{2p}}{N^{2p-1} \eta^{2p}}, \Big(\frac{A_2}{N\eta^2}\Big)^{p} \Big\}.
\end{align}

Consider a fixed $\alpha \in [d]^{k}$ (an index of the tensor corresponding to the $k$th Hermite coefficient); $|\alpha|=k$. 
Given samples $\set{(x^{(j)}, f(x^{(j)}): j \in [N]}$, the random variables of interest are $Z_j, Y_j$ are
$$ Z_j = \sum_{i=1}^m a_i \sigma(w_i^\top x^{(j)}+b_i) He_{\alpha}(x^{(j)}),~~~~\text{ and }~~~ Y_j \coloneqq Z_j - \E[Z_j] .$$

We will apply the above concentration inequality with the random variables $Y_j$. We need bounds for $\E[Y_j^2]$ and $\E[Y_j^{2p}]$. For convenience let $Z \coloneqq \sum_{i=1}^m a_i \sigma(w_i^\top x+b_i) He_{\alpha}(x)$, and $Y \coloneqq Z - \E[Z]$. Note that by applying Claim \ref{app:clm:young}, we can bound these quantities as
\begin{align*}
    \E[Y^{2p}] &= \E[(Z - \E[Z])^{2p}] \le 2^{2p} (\E[Z^{2p}] + \E[Z]^{2p}), ~~\text{where}\\
    |\E[Z]|&= \Big| \sum_{i=1}^m a_i He_{k-2}(b_i) \cdot \frac{\exp(-b_i^2/2)}{\sqrt{2\pi}} \cdot \prod_{t=1}^{k} w_i(\alpha(t))  \Big| \le mB\sqrt{k!},  \\
    \E[Z^{2p}] &= \E\Big[ \Big(\sum_{i=1}^m a_i \sigma(w_i^\top x+b_i)\Big)^{2p} He_{\alpha}(x)^{2p}\Big],
\end{align*}

On the other hand, from Hölder's inequality, we have $(\sum_{i=1}^m |c_i||z_i|)^{2p} \le (\norm{c}_{q^*}^{q^*})^{2p/q^*} \cdot \norm{z}_{2p}^{2p}$ where $q^*$ is the dual norm of $2p$ i.e., $2p/q^* = 2p-1$. Hence, again combined with Claim \ref{app:clm:young}, we have
\begin{align*}
    \E[Z^{2p}] &\le \E\Big[ \Big(\sum_{i=1}^m a_i^{q^*} \Big)^{2p/q^*} \Big( \sum_{i=1}^m \sigma(w_i^\top x+b_i)^{2p}\Big) He_{\alpha}(x)^{2p}\Big] \le   (m^{2p-1}B^{2p}) \sum_{i=1}^m \E\Big[(w_i^\top x+b_i)^{2p} He_{\alpha}(x)^{2p}\Big]\\
    &\le  (2^{2p}m^{2p-1}B^{2p}) \sum_{i=1}^m \E\Big[((w_i^\top x)^{2p}+b_i^{2p}) He_{\alpha}(x)^{2p}\Big]\\
    & \le (2mB)^{2p} \sum_{i=1}^m \Big( \E\Big[(w_i^\top x)^{2p} He_{\alpha}(x)^{2p}\Big]+ B^{2p}\E\Big[He_{\alpha}(x)^{2p}\Big] \Big) 
\end{align*}

We note that $w_i^{\top}x $ is a standard Gaussian since $\norm{w_i}_2=1$. 
%\iffalse

Now, let $He_{\alpha}(x)$ involve different indices of $x$ up to $k_1, k_2, \dots, k_d$ times. Note that $\sum_{t \in [d]} k_t \le |\alpha| = k$. %(in particular at most $k$ are non-zero). 
Using properties of Hermite polynomials, we can bound $\E[He_{\alpha}(x)^{2p}]$ as %(since the degree $k$ univariate polynomial has coefficients at most $c'_k$ for some constant $c'_k>0$),  we have for some constants $c'_\ell>0$ that depend only on $t$ for all $\ell\in [k]$
%\begin{align*}
%\E[He_\alpha(x)^{2p}]&\le \prod_{t =1}^d (c_{k_t})^{2p} (2pk_t)^{pk_t} \le  (c_{k})^{2pk} (2p)^{pk} %\Big(2^{\sum_t \frac{k_t}{k} \log k_t}\Big)^{pk} \le (c'_k)^{2p} (2pk)^{pk},
%\end{align*}
\begin{equation*}
    \E[He_{\alpha}(x)^{2p}] = \E\Big[
    \Big( \sum_{\substack{\sum_{t \in [d]} k_t \leq k}} 
        c_{k_1...k_d} \prod_{t \in [d]} x_t^{k_t}
    \Big)^{2p}
    \Big]
\end{equation*}
\begin{equation*}
    \leq \binom{d+k}{d}^{2p} \max_{\substack{\sum_{t \in [d]} k_t \leq k}} c_{k_1...k_d}^{2p} \E\Big[
        \prod_{t \in [d]} x_t^{2pk_t}
    \Big]
    \leq \binom{d+k}{d}^{2p} (2pk-1)!! (k!)^{2p}
\end{equation*}
\begin{equation*}
    \leq \Big( \binom{d+k}{d} \cdot k! \Big)^{2p} (2pk)^{pk}
    =
    C_1^{2p} \cdot (2pk)^{pk}
\end{equation*}

%where the last inequality uses the  concavity of the logarithm function.  
by setting $C_1 = \binom{d+k}{d} \cdot k!$ and repetitively applying Claim \ref{app:clm:young}. A similar argument also holds for $\E[(w_i^\top x)^{2p} He_\alpha(x)^{2p}]$ by Cauchy–Schwarz inequality 
%(note without loss of generality $w=(1,0,0,\dots,0)$, and $w^\top x=x_1$ where $x_1$ is just a co-ordinate of $x$).

\begin{equation*}
    \E[(w_i^\top x)^{2p} He_{\alpha}(x)^{2p}] \leq \sqrt{\E[(w_i^\top x)^{4p}]\E[He_{\alpha}(x)^{4p}]}
    \leq
    \sqrt{\E[(w_i^\top x)^{4p}]} \binom{d+k}{d}^{2p} (k!)^{2p} (4pk)^{pk}
\end{equation*}
\begin{equation*}
    \leq
    \Big(\sqrt{4p} \cdot \binom{d+k}{d} \cdot k! \Big)^{2p} (4pk)^{pk}
    =
    (2^{k+2}p)^p \cdot C_1^{2p} \cdot (2pk)^{pk}
\end{equation*}

since $w_i^\top x$ follows a standard Gaussian distribution. Hence we have
%From standard properties of Hermite polynomials, 
%\begin{align*}
    %\E[Z^{2p}]  & \le (2mB)^{2p} (c_2(k)^{2p} (2pk+1)^{pk+1}+ c_3(k)^{2p}(2pk)^{pk} B^{2p}) ,\\
    %A_{2p}=\E[Y^{2p}] & \le (2mB)^{2p} (c_2(k)^{2p} (2pk+1)^{pk+1}+ c_3(k)^{2p}(2pk)^{pk} B^{2p} + c_1(k) (mB)^{2p} \le (c_4(k) (pk)^{k/2} mB^2)^{2p}, 
%\end{align*}
\begin{equation*}
    \E[Z^{2p}] \leq
    C_1^{2p} (2mB)^{2p} (2pk)^{pk} ((2^{k+2}p)^p + B^{2p}) m
\end{equation*}

\begin{equation*}
    A_{2p} = \E[Y^{2p}] \leq
    2^{2p} \Big( (2mBC_1)^{2p} (2pk)^{pk} ((2^{k+2}p)^p + B^{2p}) m + (mB\sqrt{k!})^{2p} \Big)
\end{equation*}

\begin{equation*}
    \Longrightarrow A_{2p} \leq (C_2 (pk)^{k/2} mB^2)^{2p}
\end{equation*}

where $C_2 = 8C_1^{2k}$. % is a constant that only depend on $k$. 
Note that $p=1$ also gives the required bounds for $\E[Y^2]$. 

Now, setting $p:=\tfrac{1}{2}(\log(1/\eta)+\log(mdB))$, and applying Rosenthal's inequality~\eqref{eq:rosenthal} with $N=\tfrac{c'(k)}{\eta^2}  m^2 B^4 \poly(\log(mdB/\eta))$, we have for an appropriate constant $c'(k) > 0$
\begin{align*}
    \Pr\Big[ \Big| \frac{1}{N}\sum_{j=1}^N Y_j \Big| > \eta \Big] &\le 2^{p \log(p)+2p+p^2} \cdot \max\Big\{ \frac{A_{2p}}{N^{2p-1} \eta^{2p}}, \Big(\frac{A_2}{N\eta^2}\Big)^{p} \Big\}\\
    &\le 2^{p \log(p)+2p+p^2} \cdot \max\Big\{ \frac{ (C_2 mB^2)^{2p} (pk)^{pk}}{N^{2p-1} \eta^{2p}}, \Big(\frac{(C_2 mB^2k^{k/2})^2}{N\eta^2}\Big)^{p} \Big\}\\
    &\le \Big(\frac{1}{mdB}\Big)^{\log(mdB)+\log(1/\eta)}, 
\end{align*}
as required. Finally by setting $\eta=\eta'/\sqrt{d^{k}}$ and a union bound over all entries, we get that w.h.p., $\norm{T_k - \hat{f}_k}_F \le \eta'$, as required. 

\end{proof}

%%%%%%%%%%%%%%%%%%%%%%%%%%%%%%%%%%%%%%%%%%%

\subsection{Recovering the parameters under errors} \label{app:robust:recovery}

\subsubsection{Recovery of weight vectors $w_i$ for the terms in $G$} \label{app:W_recovery}

We first prove the following important lemma that shows that Jennrich's algorithm run with an appropriate choice of rank $k$ will recover the large terms.  

\paragraph{Lemma \ref{app:lem:robust-recovery-jennrich}}
\itshape
Suppose $\epsilon_2 \in (0,  \tfrac{1}{4})$, and $\ell_1, \ell_2 \ge \ell, \ell_3>0$ be constants for some fixed $\ell$, and  $T = \mathsf{flatten}(\hat{f}_{\ell_1+\ell_2+\ell_3}, \ell_1, \ell_2, \ell_3)$ have decomposition $T=\sum_{i=1}^m \lambda_i (u_i \otimes v_i \otimes z_i)$ with $\lambda_i \in \mathbb{R}$ and unit vectors $u_i=w_i^{\otimes \ell_1} \in \mathbb{R}^{d^{\ell_1}}, v_i=w_i^{\otimes \ell_2} \in \mathbb{R}^{d^{\ell_2}}, z_i=w_i^{\otimes \ell_3} \in \mathbb{R}^{d^{\ell_3}}$. There exists $\eta_1=\poly(\eps_2, \sing_m(W^{\odot \ell}))/\poly(m,d,B)>0$ and $\eps_1 \coloneqq \max\set{2\eps_2, 1/\poly(1/\eps_2, 1/\sing_m(W^{\odot \ell}), d^{\ell_1+\ell_2+\ell_3},B)}$  such that 
    if %and constants $c_1, c_2>0$ such that given $\widetilde{T}$ satisfying %with $E=\widetilde{T}- T$ satisfying 
    \begin{align}
%        \sing_m(W) & \le \frac{1}{\kappa}, \qquad \text{and}\\
        \|T-\widetilde{T}\|_{F} &\leq \eta'_1 \coloneqq \min\Big\{\poly(\eps_2, \sing_m(W^{\odot \ell}))/\poly(m,d,B,1/\eta_1), \frac{\eta_1}{2}\Big\},  %\frac{\epsilon^{c_1} \cdot \norm{T}_F}{\mathsf{poly}(\sing_m(W^{\odot \ell}), (md)^{\ell_1+ \ell_2+ \ell_3},B)} 
    \end{align}
then Jennrich's algorithm runs with rank $k'\coloneqq \argmax_{r \le m} \sing_r(\flatten(\widetilde{T}, \ell_1, \ell_2+\ell_3, 0))>\eta_1$ and w.h.p. outputs\footnote{
Note that one can also choose to pad the output with zeros to output $m$ sets of parameters instead of $k'$ if required. 
} $\set{\widetilde{\lambda}_i, \widetilde{w}_i}_{i \in [k']}$ such that there exists a permutation $\pi:[m] \to [m]$ and signs $\xi_i \in \set{1,-1}~\forall i \in [m]$  %and $\eps_1>0$ with  $\eps_1 \coloneqq 1/\poly(1/\eps_2, 1/\sing_m(W^{\odot \ell}), d^{\ell_1+\ell_2+\ell_3},B)$  satisfying $\eps_1> 2 \eps_2$, 
satisfying:
\begin{align}
(i)~~ \forall i \in [m],& ~~ |\lambda_i-\widetilde{\lambda}_{\pi(i)}|\le \eps_2^2, \text{ and }\\
    (ii)~~\forall i \in [m],& \text{ s.t. } |\lambda_i|>\epsilon_1, ~\text{we have}~ \norm{w_i^{\otimes t} - \xi_{\pi(i)}^t \widetilde{w}_{\pi(i)}^{\otimes t}}_2 \le \eps_2,~~~ \forall t \in [2\ell]. 
\end{align}
\upshape

Before we proceed to the proof of this lemma, we first state and prove a couple of simple claims. 
We use the following simple claim about the assumptions of the theorem implying lower bounds on the least singular value of the submatrices given by two columns of $W$. 
\begin{claim}
Suppose the matrix $M_j \in \R^{d^j \times 2}$ formed by columns $u^{\otimes j}$ and $v^{\otimes j}$ for $u,v \in \mathbb{S}^{d-1}$. Suppose $\sing_2(M_\ell) \ge \kappa$, then $\sing_2(M_1) \ge \kappa/\sqrt{2\ell}$.
\end{claim}
\begin{proof}
    Suppose $v= \alpha u + \sqrt{1-\alpha^2} u^{\perp}$ for some $u^{\perp} \in \mathbb{S}^{d-1}$ that is perpendicular to $u$. It is easy to see that %$v^{\otimes \ell}$ is expressible as $v^{\otimes \ell} = \alpha^\ell u^{\otimes} + V^{\perp}$, where $V^{\perp}$ involves $u^{\perp}$ in at least one of the modes and hence orthogonal to $u^{\otimes \ell}$. Hence, 
    $$\iprod{v^{\otimes \ell}, u^{\otimes \ell}}=\alpha^{\ell} .$$

    For two unit vectors $u,v$, the least singular value of the matrix given by them as columns is
    \begin{align*}
    \min_{\substack{x,y \in \R\\ x^2+y^2=1}} \norm{x u + yv}_2 &= \min_{\substack{x,y \in \R\\ x^2+y^2=1}} \sqrt{x^2+y^2+2xy \iprod{u,v} } = \min_{\substack{x,y \in \R\\ x^2+y^2=1}} \sqrt{1+2xy \iprod{u,v}}
    = \sqrt{1 - |\iprod{u,v}|}.\\    
    \text{Hence, }  
        \kappa^2&=1-\alpha^{\ell} ~~~\implies \sing_2\big([u \; v ]\big) = \sqrt{1-\alpha} = \sqrt{1-(1-\kappa^{2})^{1/\ell}} \ge  \frac{\kappa}{\sqrt{2\ell}}.
    \end{align*}
\end{proof}

We use the following simple claim shows that if we obtain a rank-1 term which is close, then the corresponding vectors are also close. 
\begin{claim}\label{claim:rank1vectors}
For any $\eps>0, \ell \in \N$ with $\ell\ge 2$, suppose $\alpha, \beta>0$, and $u, v \in \mathbb{S}^{d-1}$ satisfy  $\norm{\alpha u^{\otimes \ell} - \beta v^{\otimes \ell}}_F \le \eps$, for some $\eps\in [0, \alpha/2)$. Then there exists $\xi \in \set{+1, -1}$ such that for any $t \in \set{1,2,\dots,\ell}$, $\norm{u^{\otimes t} - \xi^t v^{\otimes t}}_F \le \sqrt{2}\eps/\alpha$. Also $|\alpha - \beta| \le 3\eps$. 
\end{claim}
We remark that if $\ell$ is odd, we can additionally conclude that $\xi=+1$, but this is not used in the arguments, so we skip its proof. 
%\anote{This is where we missed the sign.} \anote{Do we need the additional claim for odd $\ell$? We don't use it I think. }

\begin{proof}
Suppose $A_1=u^{\otimes t}, B_1 = v^{\otimes t}$ and $A_2=u^{\otimes \ell-t}, B_2 = v^{\otimes \ell -t}$. Note that they all have unit norm. Let $\eta=\min\set{\norm{A_1 - B_1}_F, \norm{A_1 +  B_1}_F}$. Then $A_1 = \sqrt{1-\eta^2/2} B_1 + \tfrac{\eta}{\sqrt{2}} B_1^{\perp}$ for some $B_1^{\perp}$ with unit norm orthogonal to $B_1$. We have
\begin{align*}
\alpha u^{\otimes \ell} - \beta v^{\otimes \ell }&= \alpha A_1 \otimes A_2 - \beta B_1 \otimes B_2 = \alpha  \sqrt{1 -\tfrac{\eta^2}{2}} B_1 \otimes A_2 + \alpha \cdot\tfrac{\eta}{\sqrt{2}} B_1^{\perp} \otimes A_2 - \beta B_1 \otimes B_2 \\
\text{Hence } \eps^2 &= \norm{\alpha u^{\otimes \ell} - \beta v^{\otimes \ell }}_F^2 \ge   \frac{\alpha^2 \eta^2}{2} \norm{B_1^{\perp} \otimes A_2}_F^2 \ge \frac{\alpha^2 \eta^2}{2} ~~(\text{ since } B_1 \perp B_1^{\perp}).
\end{align*}
Hence $ \eta^2  \le 2\eps^2/\alpha^2$. 
For even $t$, it is easy to see that $\norm{u^{\otimes t} - v^{\otimes t}}_F \le \norm{u^{\otimes t}+ v^{\otimes t}}_F$; hence $\norm{u^{\otimes t} - v^{\otimes t}}_F \le \sqrt{2} \eps/\alpha$. For odd $t$, it could be either $\norm{u^{\otimes t} - v^{\otimes t}}_F = \eta$ or $\norm{u^{\otimes t} - v^{\otimes t}}_F=\eta$; moreover the sign (in front of $v^{\otimes t}$) is coordinated across the different $t$ since $\norm{u^{\otimes t} - v^{\otimes t}}_F^2+\norm{u^{\otimes t} - v^{\otimes t}}_F^2=2$. Hence for an appropriate sign $\xi \in \set{+1,1}$ we have $\norm{u^{\otimes t} - \xi^t v^{\otimes t}}_F = \eta$.

Finally, to give an upper bound on $|\alpha-\beta|$, we use the conclusion from the above bound with $t=1$, to argue that for some $v^{\perp} \in \mathbb{S}^{d-1}$ that is orthogonal to $v$ 
\begin{align*}
    \eps^2 &\ge \Bignorm{\alpha \big(\sqrt{1-\tfrac{\eta^2}{2}} \cdot v + \tfrac{\eta}{\sqrt{2}} v^{\perp}\big)^{\otimes \ell}  - \beta v^{\otimes \ell}}_F^2  \\
    &= \Big|\alpha \big(1- \frac{\eta^2}{2} \big)^{\ell/2} - \beta \Big|^2 +  \alpha^2 \Big( 1- \big(1- \frac{\eta^2}{2} \big)^{\ell}\Big).
\end{align*}
Since $\eta^2 \in (0,1)$, we can use a simple linear approximation to claim that $t \eta^2/4 \le |1-(1-\eta^2/2)^t| \le t \eta^2$ for any $t>0$. Hence, we get that \begin{align*}
\eps^2& \ge (|\alpha - \beta| - |\alpha|\ell \eta^2)^2 +  \alpha^2 \big( \tfrac{1}{4}\ell \eta^2 \big).\\
\text{Hence } |\alpha - \beta| &\le  \eps+|\alpha| \sqrt{\ell} \eta, \text{ and } |\alpha| \sqrt{\ell} \eta  \le 2 \eps. 
\end{align*}
Hence the claim follows. 
 
\end{proof}

We now proceed to the proof of Lemma~\ref{app:lem:robust-recovery-jennrich}.

\begin{proof}[Proof of Lemma~\ref{app:lem:robust-recovery-jennrich}]

The proof proceeds by first identifying a tensor $\widetilde{T}$ which we show satisfies all the conditions for Jennrich's robust algorithmic guarantee (Theorem~\ref{thm:jennrich}) with rank $k'$, which corresponds to the $k'$-th largest $|\lambda_i|$. Note that the recovery error in the rank-$1$ terms may be larger than some of the $k'$ terms of $\widetilde{T}$ (for example if there is not much separation between the $k'$ largest and $(k'+1)$th largest of the $\set{|\lambda_i|}$). Therefore, we argue that if $|\lambda_i|$ is sufficiently large, it will be recovered up to small error.  

We start with some notation.     Suppose $\smin \coloneqq \sing_m(W^{\odot \ell})$.  
    Let $\widetilde{M}=\flatten(\widetilde{T}, \ell_1, \ell_2+\ell_3, 0)$ and $M=\flatten(T, \ell_1, \ell_2+\ell_3,0)$. Set $U=W^{\odot \ell_1}, V = W^{\odot \ell_2},  Y=\diag(\lambda) Z = \diag(\lambda) W^{\odot \ell_3}$.      Recall that $k'\coloneqq \argmax_{r \le m} \sing_r(\widetilde{M})>\eta_1$. 
    Note that $M= U \diag(\lambda) (V \odot Y)^\top$, where $\lambda=(\lambda_1, \dots, \lambda_m)$.
    We remark that by Claim~\ref{claim:khatri-rao} 
    \begin{equation}\label{eq:jenn:factor:smin}
        \sing_{m}(U) \ge \frac{\sing_m(W^{\odot \ell})}{(2m)^{\ell_1-\ell}} \ge  \frac{\smin}{(2m)^{\ell_1}}, ~~\text{ and similarly }~  \sing_{m}(V) %\ge \frac{\sing_m(W^{\otimes \ell})}{(2m)^{\ell_2-\ell}} 
        \ge \frac{\smin}{(2m)^{\ell_2}}, ~~ \sing_{m}(V \odot Z) %\ge \frac{\sing_m(W^{\otimes \ell})}{(2m)^{\ell_2+\ell_3-\ell}} 
        \ge \frac{\smin}{(2m)^{\ell_2+\ell_3}}.
    \end{equation}

    We first argue that there are at least $k'$ values of $|\lambda_i|$ that are non-negligible. 
    Since $\norm{T-\widetilde{T}}_F \le \eta'_1 < \eta_1/2$, we have from Weyl's inequality that $\sing_{k'}(M)>\eta_1/2$.  
    Let $S \subset [m]$ denote the indices corresponding to the largest $k'$ values of $|\lambda_i|$ (this is for analysis). The rank-1 terms restricted to $S$ will constitute the ``ground-truth'' decomposition $\widetilde{T}$.   
    We first observe that 
    \begin{align}
    \min_{i \in S} |\lambda_i| &> \frac{\eta_1}{2m}, ~~~\text{ and } \label{eq:jenn:S1}\\
    \forall i \in [m]  ~\text{ s.t. }~  |\lambda_i| &\ge \frac{\eta_1 (2m)^{\ell_1+ \ell_2+\ell_3}}{\smin^2}, ~~\text{ we have }~~ i \in S. \label{eq:jenn:S2}  
    \end{align}
    To see why \eqref{eq:jenn:S1} holds, note that 
    $$\frac{\eta_1}{2}< \sing_{k'}(M) = \sing_{k'}\Big( U \diag(\lambda) (V \odot Y)^{\top}\Big) \le \sing_{k'}(\diag(\lambda)) \cdot \sing_1(U) \cdot \sing_1(V \odot Y) \le m\cdot \min_{i \in S} |\lambda_i|, $$
    where we used the fact that all the columns of $U$ and $V \odot Z$ are unit vectors. 
    To show \eqref{eq:jenn:S2}, suppose we assume for contradiction that $|\lambda_i|> 2(2 m)^{\ell_1+\ell_2+\ell_3} \eta_1/\smin^2$, but $i \notin S$. Let $S'=S \cup \set{i}$. Then we can see that at least $k'+1$ singular values of $\widetilde{M}$ are greater than $\eta_1$ since by Weyl's inequality,
    \begin{align*}
        \sing_{k'+1}(\widetilde{M}) &\ge \sing_{k'+1}(M) - \eta_1' = \sing_{k'+1}\Big( U \diag(\lambda) (V \odot W)^{\top} \Big) \eta'_1\\
        &\ge \sing_m(U) \cdot \sing_{k'+1}(\diag(\lambda))  \cdot \sing_m( V \odot W) \ge \frac{\smin}{(2m)^{\ell_1}} \cdot |\lambda_i| \cdot \frac{\smin}{(2m)^{\ell_2}} \\
        &\ge 2\eta_1 - \eta_1'> \eta_1. 
    \end{align*}
    Hence \eqref{eq:jenn:S1} and \eqref{eq:jenn:S2} are both true.

    We now argue that we satisfy the requirements of Theorem~\ref{thm:jennrich} (the robust guarantee).
    Let $U_S, V_S$ and $Y_S$ denote the restriction of the factor matrices $U,V,Y$ to the columns corresponding to $S$. Then by Claim~\ref{claim:khatri-rao} 
    $$\sing_{k'}(U) \ge \sing_{m}(U) \ge \frac{\sing_m(W^{\odot \ell})}{(2m)^{\ell_1-\ell}} \ge  \frac{\smin}{(2m)^{\ell_1}}, \text{ and } \sing_{k'}(U) \ge \sing_{m}(U) \ge \frac{\sing_m(W^{\odot \ell})}{(2m)^{\ell_2-\ell}} \ge \frac{\smin}{(2m)^{\ell_2}}.$$ 
    Moreover for any two columns $i,j \in S$, we have that the restriction of $Z$ to these two columns $Y_{\set{i,j}}$ satisfies 
    $$ \sing_2(Y_{\set{i,j}}) \ge \min\set{|\lambda_i|, |\lambda_j| } \cdot \sing_2(W_{\set{i,j}}) \ge \frac{\eta_1 \cdot \smin}{\sqrt{2\ell} \cdot m}.$$
    Moreover the maximum singular values of the factor matrices $U,V$ are all upper bounded by $\sqrt{m}$.
    
    Finally, suppose $T_S=\sum_{i \in S} \lambda_i u_i \otimes v_i \otimes z_i$, then the error between the input tensor and $T_S$
    \begin{align*}
        \norm{\widetilde{T}-T_S}_F &\le \norm{\widetilde{T} - T}_F + \norm{T - T_S}_F \le \eta'_1 + \Bignorm{\sum_{i \notin S} \lambda_i u_i \otimes v_i \otimes z_i }_F\\
        &=  \eta'_1 + \Bignorm{\flatten\Big(\sum_{i \notin S} \lambda_i u_i \otimes v_i \otimes z_i, \ell_1, \ell_2+\ell_3,0\Big) }_F \\
        &\le \eta'_1 + \sqrt{m} \sing_1\Big(\flatten\big(\sum_{i \notin S} \lambda_i u_i \otimes v_i \otimes z_i, \ell_1, \ell_2+\ell_3, 0\big) \Big) \\
        &\le \eta'_1 + \sqrt{m} \eta_1 \le (\sqrt{m}+1) \eta_1. 
    \end{align*}
    
%\anote{Perhaps clarify the setting of $\eta_1$ better.}
%\tnote{Maybe we can remove the dependency on $s_{\min}$, since we will still make $\eta'_1 = \eta_1/2$ even if $s_{\min}$ is too small.}

    Now applying Theorem~\ref{thm:jennrich} with $\epsjenn=\eps_2^2$, and setting $\eta'_1$ such that  $\eta'_1 = \min \Big \{ \etajenn\big(\epsjenn = \eps_2^2, \kappa = \tfrac{\sqrt{m} (2m)^{\ell_1+\ell_2}}{\smin}, d^{\ell_1+\ell_2+\ell_3}, \delta = \tfrac{\eta_1\cdot \smin}{\sqrt{2\ell} \cdot m}\big)/(\sqrt{m}+1), \eta_1/2 \Big \}$, we have that the rank-1 terms can be recovered up to accuracy $\eps_2^2$ (up to renaming the indices $i \in [m]$):
    \begin{align}
        \forall i \in S, ~~ \norm{\lambda_i u_i \otimes v_i \otimes z_i - \widetilde{\lambda}_i \widetilde{u}_i \otimes \widetilde{v}_i \otimes \widetilde{z}_i}_F &\le \eps_2^2 .\label{eq:rank1error} \\
    \text{From Claim~\ref{claim:rank1vectors} },         \forall i \in S, ~~ \norm{w_i- \xi_i \widetilde{w}_i}_2 &\le \frac{\eps_2^2}{|\lambda_i|} , \label{eq:jennrich:robust:gent}
    \end{align}
    for appropriate signs $\xi_i \in \set{+1, -1}$. 
    We remark that the choice of $\eta'_1$ is consistent with both Theorem \ref{thm:jennrich} and this lemma, since in our case $\kappa \leq \poly_{\ell}(m)/s_{\min}$ and $1/\delta \leq \poly_{\ell}(m, 1/\eta_1)/s_{\min}$. If $s_{\min} > 0$ becomes too small, we will directly set $\eta'_1$ as $\eta_1/2$ instead.
    
    Note that from \eqref{eq:rank1error} and triangle inequality, we already have for $i \in S$ 
    $$|\lambda_i - \widetilde{\lambda}_i| = \Big|\norm{\lambda_i u_i \otimes v_i \otimes z_i}_F - \norm{\widetilde{\lambda}_i \widetilde{u}_i \otimes \widetilde{v}_i \otimes \widetilde{z}_i}_F \Big| \le \norm{\lambda_i u_i \otimes v_i \otimes z_i - \widetilde{\lambda}_i \widetilde{u}_i \otimes \widetilde{v}_i \otimes \widetilde{z}_i}_F \le \eps_2^2.$$ (For terms not in $S$, the output $\widetilde{\lambda_i}=0$ and $|\lambda_i|<\eta_1/(2m)$, hence it is still satisfied). 
    For any $i \in [m]$ s.t., $|\lambda_i| > \eps_1$, we have from \eqref{eq:jenn:S2} that $i \in S$; hence
    \begin{align*}
    \text{For all } i ~\text{ s.t. } |\lambda_i| > \eps_1, ~~~ \norm{w_i- \xi_i \widetilde{w}_i}_2 &\le \frac{\eps_2^2}{\eps_1} \le \eps_2,
    \end{align*}
   as long as $|\eps_1| \ge \sqrt{\eps_2}$.  A similar proof also holds for $\norm{w_i^{\otimes t} - \xi_i^t \widetilde{w}_i^{\otimes t}}_F$ by using Claim~\ref{claim:khatri-rao} with general $t \ge 1$ in \eqref{eq:jennrich:robust:gent}. This completes the proof.  
    
\end{proof}

%\anote{Writing down a bunch of lemmas. The error terms may need to be adjusted.}
%The following claim shows that we can recover all the weight vectors $w_i$ for all the terms $i \in G$.  

\paragraph{Lemma \ref{lem:robust:W_recovery}}
\itshape
%\begin{lemma}\label{lem:robust:W_recovery}
For any $\eps_2>0$, there exists an $\eta_2' = \frac{\poly(\eps_2, \sing_m(W^{\odot \ell}))}{\poly_\ell(m,d,B)}>0 $ such that if the estimates $\norm{T_k - \hat{f}_k}_F \le \eta_2'$ for all $k \in \set{0,1,\dots, 2\ell+2}$, then steps 1-4 of Algorithm \ref{alg3} finds a set $\set{\widetilde{w}_i : i \in [m']}$ such that  there exists a one-to-one map $\pi: [m'] \to [m]$ satisfying {\em (i)} every $i \in G$ has a pre-image in $\pi$ (i.e., every term in $G$ is recovered), and for appropriate signs $\set{\xi_i \in \set{1,-1}: i \in [m']}$, 
\begin{equation}\label{eq:wrecovery}
    \forall i \in [m'], \forall t \in [2\ell], ~~ \norm{\xi_i^t \widetilde{w}_{i}^{\otimes t} - w_{\pi(i)}^{\otimes t}}_F \le \eps_2.  
\end{equation}
In particular $\forall i \in [m']$, we have $\norm{\xi_i \widetilde{w}_{i} - w_{\pi(i)}}_2 \le \eps_2$. 
\upshape
%\end{lemma}
%Note that the above lemma applied with $t=1$ shows that $\tilde{w}_i \approx w_i$ for all $i \in G$. 

The stronger guarantee for all $t \in [2\ell]$ will be useful in bounding the recovery error of the $a_i, b_i$ in later steps. 

\begin{proof}[Proof of Lemma~\ref{lem:robust:W_recovery}]
The proof uses the robust guarantees of Jennrich's algorithm in Lemma~\ref{app:lem:robust-recovery-jennrich} along with the crucial property of separation of roots in Lemma~\ref{lem:hermite_roots}.

Set $\eps'_1 \coloneqq \eps_1/\sqrt{k!/2}$ and  $\eps_1=\poly(\eps_2, m,d,\sing_m(W^{\odot \ell}))$ be given by Lemma~\ref{app:lem:robust-recovery-jennrich}. Similarly $\eta'_2$ is specified by the requirement of Lemma~\ref{app:lem:robust-recovery-jennrich}. Set also $\eta_2=4\eps_1$.

Consider a $\widetilde{w}_i$ output by the algorithm in step 4; and suppose w.l.o.g. it was output in step 2. Then we have that $|\widetilde{\lambda}_i| \ge \eta_2$. Further, $|\widetilde{\lambda}_i - \lambda_i| \le \eps_2^2 < \eta_2/4$. Hence, for every term $i\in [m']$ that is output after step 4, we have $|\lambda_i|> \eta_2/2 \ge  \eps_1$.

We first argue that every term in $G$ is one of the $m'$ terms output by the algorithm in step 4. Consider the decompositions of the two tensors obtained from the Hermite coefficients of $f$ i.e., 
\begin{align}
    \hat{f}_{2\ell+1} &= \sum_{i=1}^m  (-a_i) \cdot He_{2\ell-1}(b_i) \cdot \frac{1}{\sqrt{2\pi}}\exp(-b_i^2/2) \cdot w_i^{\otimes 2\ell+1}= \sum_{i=1}^m \lambda_i (w_i^{\otimes \ell}) \otimes (w_i^{\otimes \ell}) \otimes w_i\\
    \hat{f}_{2\ell+2} &= \sum_{i=1}^m  a_i \cdot He_{2\ell}(b_i) \cdot \frac{1}{\sqrt{2\pi}}\exp(-b_i^2/2) \cdot w_i^{\otimes 2\ell+2} = \sum_{i=1}^m \lambda'_i (w_i^{\otimes \ell}) \otimes (w_i^{\otimes \ell}) \otimes w_i^{\otimes 2}.
\end{align}

Note that from Lemma~\ref{lem:hermite_roots} we have that for every $x \in \R$, at least one of $|He_{2\ell+2}(x)|, |He_{2\ell+1}(x)|$ is at least $\sqrt{k!/2}$. Moveover since $i \in G$ for our choice of $c_k$ in \eqref{eq:def:G}, we have that $e^{-b_i^2/2}/\sqrt{2\pi}>\eps'_1 $. Hence for each $i \in G$, we have that $\max\{|\lambda_i|, |\lambda'_i|\}> \eps_1$.

Finally, we now prove that when $|\lambda_i|\ge \eps_1$, the corresponding $w_i$ is recovered up to error $\eps_2$.  
From Lemma~\ref{app:lem:robust-recovery-jennrich}, if $\widetilde{w}_i$ is the vector output by one of the decompositions for $w_i$, we have for all $t \in [2\ell]$ that  $\norm{w_i^{\otimes t} -\xi_i^t \widetilde{w}_i^{\otimes t}}_F \le \eps_2$ for some sign $\xi_i \in \set{1,-1}$ as required.   Moreover since $\eta_2 \coloneqq \eps_1/2$ and the error in each rank-1 term is at most $\eps_{\ref{thm:jennrich}}< \eta_1/2$, we have that none of these terms are removed as duplicates of other terms.  On the other hand, since $\eta_3 \coloneqq 2 \eps_2$, we have that duplicates are correctly removed. Hence we have that for every $i \in [m']$, $\norm{w_i^{\otimes t} - \xi_i^t \widetilde{w}_i^{\otimes t}}_F \le \eps$ for appropriate signs $\xi_i \in \set{\pm 1}$.

%Finally, consider a term $\set{\tilde{w}_i: i \in [m']}$ output by the algorithm in step 4; and suppose w.l.o.g. it was output in step 2. Then we have that $|\tilde{\lambda}_i| \ge \eta_2$. Further, $|\tilde{\lambda}_i - \lambda_i| \le \eps_2^2 < \eta_2/4$. Hence, for every term $i\in [m']$ that is output after step 4, we have $|\lambda_i|> \eta_2/2 \ge  \eps_1/4$. Since the recovery error $\eps_1$ has a polynomial dependence on $\eps_2$, we have that these terms are also recovered up to an error $c'_2 \eps_2$ for some constant $c'_2>0$.  

%\anote{To be filled in. }
\end{proof}

\subsubsection{Recovering error for the parameters $a_i, b_i$ for terms $i \in G$.  }\label{app:ab_recovery}

In this section we prove the following claim that shows the recovery of all the $\set{a_i, b_i : i \in G}$ (and in fact, all the terms output in steps 1-5 of Algorithm \ref{alg3}).

Before we start the proof of the main lemmas, we first show a key property of Hermite polynomials we will utilize later.

\begin{claim} \label{app:clm:cramer}
$\forall x \in \R$, $|He_k(x)| \exp(-x^2/2) \leq \sqrt{k!}$
\end{claim}

\begin{proof}
    We utilize Cramér's inequality for Hermite functions that for all $x \in \R$, $|\psi_k(x)| \leq \pi^{-1/4}$, where $\psi_k(x)$ is the $k$'th Hermite function given by 
    \begin{equation*}
        |\psi_k(x)| = (2^k k! \sqrt{\pi})^{-1/2} \exp(-x^2/2) |H_k(x)|
    \end{equation*}
    
    with $H_k(x)$ denoting the $k$'th physicist's Hermite polynomial\footnote{The physicist's Hermite polynomials are defined as $H_k(x) = \frac{(-1)^k}{\exp(-x^2)} \cdot \frac{d^k}{dx^k} \exp(-x^2)$}. Now, substituting $x$ with $x/\sqrt{2}$ yields
    \begin{equation*}
        |\psi_k(\frac{x}{\sqrt{2}})| = (k! \sqrt{\pi})^{-1/2} \exp(-x^2/4) |He_k(x)| \leq \pi^{-1/4}
    \end{equation*}
    \begin{equation*}
        \Longrightarrow ~~
        \exp(-x^2/2) |He_k(x)| \leq \sqrt{k!} \exp(-x^2/4) \leq \sqrt{k!}
    \end{equation*}
\end{proof}

With this claim, we are now ready to proceed.

\paragraph{Lemma \ref{lem:robust:ab_recovery}}
\itshape
%\begin{lemma}\label{lem:robust:ab_recovery}
For $\eps>0$ in the definition of $G$ in \eqref{eq:def:G}, there exists $\eta_3'=\frac{\poly(\eps, \sing_m(W^{\otimes \ell}))}{\poly_\ell(m,d,B)}>0$, and $\eps_3'= \frac{\poly(\eps, \sing_m(W^{\otimes \ell}))}{\poly_\ell(m,d,B)}>0$ 
such that for some $\xi_i \in \set{\pm 1}~\forall i \in [m']$
$$ \text{if}~~\norm{T_k - \hat{f}_k}_F \le \eta_3'~~ \forall k \le 2\ell+2, ~~~~\text{ and }   ~~~\norm{\widetilde{w}_i^{\otimes t} - \xi^t w_i^{\otimes t}}_F \le \eps_3', \forall i \in [m'], \forall \ell \le t \le \ell+3.$$ 
then steps 5-6 of the algorithm finds $(\widetilde{a}_i, \widetilde{b}_i: i \in [m'])$ such that  
\begin{equation}\label{eq:wrecovery}
|\widetilde{a}_i - a_i | \le \eps, \text{ and } |\widetilde{b}_i - \xi_i b_i| \le \eps.  
\end{equation}
\upshape
%\end{lemma}
%Note that in Lemma~\ref{lem:robust:W_recovery} we showed that $G$ is contained in the $m'$ terms output in steps 1-4 (and hence step 5 as well). 
%%%% Robust version of Lemmas
%\newcommand{\tempa}{\alpha}

The above uses Lemma \ref{lem:robust:recovscalars}, a robust version of Lemma~\ref{lem:recovscalars} when there are errors in the estimates. 
Also $\beta=\tempa e^{-z^2/2}$ in the notation of Lemma~\ref{lem:recovscalars}. %This lemma is used to estimate the $a_i, b_i$ once we have recovered the $w_i$.  

\anote{Need to copy paste these back into Section 5.}

\paragraph{Lemma \ref{lem:robust:recovscalars}}(Robust version of Lemma~\ref{lem:recovscalars} for $k \ge 2$)~
\itshape
%\begin{lemma}[Robust version of Lemma~\ref{lem:recovscalars} for $k \ge 2$]\label{lem:robust:recovscalars}
Suppose $k \in \N, k \ge 2, B \ge 1$, and $\tempa,z \in \R$ be unknown parameters.  There exists a constant $c_k=c(k)\ge 1$ such that for any $\eps \in (0,\tfrac{1}{4})$ satisfying (i) $|\tempa| \in [\tfrac{1}{B}, B]$ and $|z| \le B$, and (ii) $|z| <  2\sqrt{\log(c_k/ \eps^{1/4} (1+B)^3))}$ \anote{simplify this?}, if we are given values $\gamma_k,\gamma_{k+1},\gamma_{k+2}, \gamma_{k+3}$  s.t. for some $\xi \in \set{\pm 1}$, $$ \Big| \gamma_j-  \tempa \cdot \frac{(-1)^j}{\sqrt{2\pi}} e^{-z^2/2}  He_j(\xi z) \Big| \le \eps'=\frac{\eps^{4}}{2(1+B)^2} ~~~ \forall j \in \set{k, k+1, k+2, k+3},$$
then the estimates $\widetilde{z}, \widetilde{\tempa}$ obtained as:   
\begin{align}
\widetilde{z}=  -\frac{\gamma_{q+1}+q 
  \cdot \gamma_{q-1}}{\gamma_{q}} ~~\text{where } q:= \argmax_{\substack{j \in \set{k+1, k+2}}} |\gamma_{j}|, ~~\text{and}~~\widetilde{\tempa}&= (-1)^{q} \frac{\sqrt{2\pi} \gamma_{q}}{e^{-\frac{\widetilde{z}^2}{2}}He_{q}(\widetilde{z})}  \nonumber\\
%\end{align}
%Then $\widetilde{z}$ and $\widetilde{\alpha}$ satisfy
%\begin{align}
\text{satisfy}~~~
\big| \widetilde{z}- \xi z \big| \le \frac{\eps|z|}{B+1}  \le \eps, ~~\text{ and }~~\big| \widetilde{\tempa}-\tempa \big| &\le \eps. %\label{eq:ab_assgt}
\end{align}
\upshape
%\end{lemma}
%\vspace{-5pt}
\begin{proof}
Set $\eps'\coloneqq \eps^4/(2(1+B)^2)$, and let $\beta:= \tempa e^{-z^2/2}/\sqrt{2\pi}$. Under our assumptions $|\beta| > c_k (\eps')^{1/4} (1+B)^2$. 
    We use the following fact about Hermite polynomials:
    \begin{equation} \label{eq:recurrence:rob}
        He_{r+1}(z)=z He_{r}(z)-r \cdot He_{r-1}(z).
    \end{equation}
    For convenience, for a scalar quantity $v$ we denote by $v = a \pm \epsilon$ iff $|v-a| \le \epsilon$. Recall that $q=\argmax_{\substack{j \in \set{k+1, k+2}}} |\gamma_{j}|$. Since $\beta \sqrt{k!/2}> 4\eps'$ by the conditions, we have that $|\gamma_q| > \beta \sqrt{k!}/2 $.

    Setting $r=q$ in \eqref{eq:recurrence:rob} and dividing by $He_{q}(z)$ we get using its odd or even function property depending on parity of $q$,
    \begin{align*}
        \xi z&=\xi \cdot \frac{ He_{q+1}(z) + q He_{q-1}(z)}{ He_{q}(z)}=\frac{\beta He_{q+1}(\xi z) + q\cdot \beta He_{q-1}(\xi z)}{\beta He_{q}(\xi z)}\\
       &= -\frac{(\gamma_{q+1} \pm \eps') +q ( \gamma_{q-1} \pm \eps')}{  \gamma_{q} \pm \eps'} = -\frac{\gamma_{q+1} + q \gamma_{q-1}}{ \gamma_{q} (1\pm \tfrac{\eps'}{\gamma_{q}})} \pm \frac{ (k+3)\eps'}{ \gamma_{q} (1\pm \tfrac{\eps'}{\gamma_{q}})}\\
        &=-\frac{\gamma_{q+1}+q 
  \cdot \gamma_{q-1}}{\gamma_{q}} \cdot  \Big( 1 \pm \frac{2 \eps'}{|\gamma_q|}\Big) \pm \frac{(k+3) \eps'}{\gamma_q (1\pm \tfrac{\eps'}{|\gamma_q|})}\\
&  =  \widetilde{z} (1\pm \sqrt{\eps'}) \pm \sqrt{\eps'}, ~~~~~~ \text{ since } \beta \sqrt{k!}/2 \ge 2(k+3) \sqrt{\eps'}. 
    \end{align*}
    
    Let $g(z)=e^{-z^2/2} He_q(z)/\sqrt{2\pi}$. Hermite polynomials satisfy $He'_{q}(x) = q He_{q-1}(x)$. Hence
    $$g'(z)=e^{-z^2/2}\big(q He_{q-1}(z) - z He_q(z)\big)/\sqrt{2\pi}= - \frac{e^{-z^2/2}He_{q+1}(z)}{\sqrt{2\pi}},$$
    by applying \eqref{eq:recurrence:rob}. Also, by Claim \ref{app:clm:cramer}, $\max_{z'} |g'(z')| \le \max_{z'} e^{-z'^2/2}|He_{q+1}(z')|/\sqrt{2\pi} \le (q+1)!$. Hence,
    %for some constant $c'_q>0$ (note $q \le k+2$)
    
    \begin{equation*}
        |g(\widetilde{z}) - g(\xi z)| \le \max_{z' \in [z,\widetilde{z}] \cup z' \in [\widetilde{z}, z]} |g'(z)| |\widetilde{z} - \xi z| \le 4 (q+1)!  (1+B) \sqrt{\eps'}  
    \end{equation*}
    
    Plugging these error bounds into $\tempa$, and using $He_q(\xi z)= \xi^q He_q(x)$ we have
    \begin{align*}
       \widetilde{\tempa} & = \frac{\gamma_q}{g(\widetilde{z})} = \frac{\tempa \xi^q \cdot \tfrac{e^{-z^2/2}}{\sqrt{2\pi}} He_q(z) \pm \eps'}{ \tfrac{e^{-z^2/2} He_q(\xi z)}{\sqrt{2\pi}}  \pm |g(\widetilde{z})-g(\xi z)|}\\
       &=\frac{\tempa g(z) \pm \eps'}{g(z) \pm 4 (q+1)! (1+B)\sqrt{\eps'} } = \tempa \Big(1 \pm \frac{8 (q+1)! (1+B)\sqrt{\eps'}}{|g(z)|} \Big) \pm \frac{2\eps'}{ |g(z)|}\\
      |g(z)|&= \frac{|\beta He_q(z)|}{|\tempa|} \ge \frac{|\beta| \sqrt{k!}}{2 B} > c_k (1+B)^2 (\eps')^{1/4} \\ 
       \text{Hence }\big|\widetilde{\tempa}-\tempa \big| & \le 8 |\tempa|\times \frac{ (q+1)! (1+B) \sqrt{\eps'}}{c_k (1+B)^2 (\eps')^{1/4}} + \frac{2\eps'}{c_k (1+B)^2 \eps'^{1/4}} \le (\eps')^{1/4} \le \eps, 
   \end{align*}
    because of our choice of $c_k = 16(k+3)!$. %\ge 16\max\{c_q,1\}$.
    
\end{proof}

The simpler variant of the above lemma (Lemma~\ref{lem:robust:recovscalars}) for $k=1$ follows a very similar analysis. 
\paragraph{Lemma \ref{lem:robust:recovscalars:1}}(Robust version of Lemma~\ref{lem:recovscalars} for $k=1$)~
\itshape
%\begin{lemma}[Robust version of Lemma~\ref{lem:recovscalars} for $k \ge 2$]\label{lem:robust:recovscalars}
Suppose $B \ge 1$, and $\tempa,z \in \R$ be unknowns.  There exists a constant $c\ge 1$ such that for any $\eps \in (0,\tfrac{1}{4})$ satisfying (i) $|\tempa| \in [\tfrac{1}{B}, B]$ and $|z| \le B$, and (ii) $|z| <  2\sqrt{\log(c/ \eps^{1/4} (1+B)^3))}$, if we are given values $\gamma_0,\gamma_1$  s.t. for some $\xi \in \set{\pm 1}$, 
$$\Big| \gamma_j-  \tempa \cdot \frac{(-1)^j}{\sqrt{2\pi}} e^{-z^2/2}  He_j(\xi z) \Big| \le \eps'=\frac{\eps^{4}}{2(1+B)^2} ~~~ \forall j \in \set{0,1},$$
%$$ \Big| \gamma_0-  \tempa \cdot \frac{e^{-z^2/2}}{\sqrt{2\pi}}   \Big| + \Big| \gamma_1 +  \tempa \cdot \frac{e^{-z^2/2}}{\sqrt{2\pi}}  He_j(\xi z) \Big| \le \eps'=\frac{\eps^{4}}{2(1+B)^2} \le \eps'=\frac{\eps^{4}}{2(1+B)^2},$$
then the estimates $\widetilde{z}, \widetilde{\tempa}$ obtained as:   
\begin{align}
\widetilde{z}&=  -\frac{\gamma_{1}}{\gamma_{0}} , ~~\text{ and }~~\widetilde{\tempa}=  \frac{\sqrt{2\pi} \gamma_{0}}{e^{-\frac{\widetilde{z}^2}{2}}} \nonumber \\
\text{satisfy}~~~    \big| \widetilde{z}- \xi z \big| \le \frac{\eps|z|}{B+1}  &\le \eps, ~~\text{ and }~~\big| \widetilde{\tempa}-\tempa \big| \le \eps. %\label{eq:ab_assgt}
\end{align}
\upshape
%\end{lemma}
%\vspace{-5pt}
Note that $He_0(z)=1$ and $He_1(z)=z$ to see the similarities between Lemma~\ref{lem:robust:recovscalars} and Lemma~\ref{lem:robust:recovscalars:1}
\begin{proof}
Set $\eps'\coloneqq \eps^4/(2(1+B)^2)$, and let $\beta:= \tempa e^{-z^2/2}/\sqrt{2\pi}$. Under our assumptions $|\beta| > c (\eps')^{1/4} (1+B)^2$. 
    % We use the following fact about Hermite polynomials:
    % \begin{equation} \label{eq:recurrence:rob}
    %     He_{r+1}(z)=z He_{r}(z)-r \cdot He_{r-1}(z).
    % \end{equation}
    For convenience we denote for a scalar $v$, $v = a \pm \epsilon$ iff $|v-a| \le \epsilon$. 
    
%    Recall that $q=\argmax_{\substack{j \in \set{k+1, k+2}}} |\gamma_{j}|$. 
Since $\beta > 4\eps'$ by the conditions, we have that $|\gamma_0| > \beta $.      Recall that $He_0(z)=1$ and $He_1(z)=z$. Hence, 
    \begin{align*}
        \xi z&=\xi \cdot \frac{ He_{1}(z)}{ He_{0}(z)}= \frac{ \beta He_{1}(\xi z)}{ \beta He_{0}(z)} \\ 
       &= -\frac{(\gamma_{1} \pm \eps') }{  \gamma_{0} \pm \eps'} = -\frac{\gamma_{1} \pm \eps'}{ \gamma_{0} (1\pm \tfrac{\eps'}{\gamma_{0}})} 
        =-\frac{\gamma_{1}}{\gamma_{0}} \cdot  \Big( 1 \pm \frac{2 \eps'}{|\gamma_0|}\Big) \pm \frac{\eps'}{\gamma_0 (1\pm \tfrac{\eps'}{|\gamma_0|})}\\
&  =  \widetilde{z} (1\pm \sqrt{\eps'}) \pm \sqrt{\eps'}, ~~~~~~ \text{ since } \beta \ge 8 \sqrt{\eps'}. 
    \end{align*}
    
To argue about $|\widetilde{\tempa}-\tempa|$,    let $g(z)=e^{-z^2/2} /\sqrt{2\pi}$. %Hermite polynomials satisfy $He'_{q}(x) = q He_{q-1}(x)$. Hence
%    $$g'(z)=e^{-z^2/2}\big(q He_{q-1}(z) - z He_q(z)\big)/\sqrt{2\pi}= - \frac{e^{-z^2/2}He_{q+1}(z)}{\sqrt{2\pi}},$$
 %   by applying \eqref{eq:recurrence:rob}. Also, 
 Its derivative $g'(z)$ satisfies
 by Claim \ref{app:clm:cramer}, $\max_{z'} |g'(z')| \le \max_{z'} |z'| e^{-z'^2/2}/\sqrt{2\pi} \le 1$. Hence,
    %for some constant $c'_q>0$ (note $q \le k+2$)
    
    \begin{equation*}
        |g(\widetilde{z}) - g(\xi z)| \le \max_{z' \in [z,\widetilde{z}] \cup z' \in [\widetilde{z}, z]} |g'(z)| |\widetilde{z} - \xi z| \le 4   (1+B) \sqrt{\eps'}  
    \end{equation*}
    
    Plugging these error bounds into $\tempa$ we have
    \begin{align*}
       \widetilde{\tempa} & = \frac{\gamma_0}{g(\widetilde{z})} = \frac{\tempa  \cdot \tfrac{e^{-z^2/2}}{\sqrt{2\pi}}  \pm \eps'}{ \tfrac{e^{-z^2/2} }{\sqrt{2\pi}}  \pm |g(\widetilde{z})-g(\xi z)|}\\
       &=\frac{\tempa g(z) \pm \eps'}{g(z) \pm 4 (1+B)\sqrt{\eps'} } = \tempa \Big(1 \pm \frac{8  (1+B)\sqrt{\eps'}}{g(z)} \Big) \pm \frac{2\eps'}{ g(z)}\\
      g(z)&= \frac{\beta}{\tempa} \ge \frac{|\beta|}{B} > c (1+B)^2 (\eps')^{1/4} \\ 
       \text{Hence }\big|\widetilde{\tempa}-\tempa \big| & \le 8 |\tempa|\times \frac{ (1+B) \sqrt{\eps'}}{c (1+B)^2 (\eps')^{1/4}} + \frac{2\eps'}{c (1+B)^2 \eps'^{1/4}} \le (\eps')^{1/4} \le \eps, 
   \end{align*}
    because of our choice of $c = 16$. %\ge 16\max\{c_q,1\}$.
    
\end{proof}

We now prove Lemma~\ref{lem:robust:ab_recovery}.
\begin{proof}[Proof of Lemma~\ref{lem:robust:ab_recovery}]
    Set $\eps_3 = \eps^4 \sing_m(W^{\otimes \ell}/(16m^{3/2} (1+B)^2)$.
    For each of the $j \in \set{\ell, \ell+1, \ell+2, \ell+2}$,  we have from Lemma~\ref{lem:hermitecoeffs} that 
    $$\hat{f}_j = \sum_{i=1}^m (-1)^j \cdot a_i \cdot He_{j-2}(b_i) \cdot \frac{\exp(\frac{-b_i^2}{2})}{\sqrt{2\pi}} \cdot w_i^{\otimes j}.$$
    Moreover $\norm{T_j - \hat{f}_j}_F \le \eta'_3$.  Also for the terms $i \notin G$, we have for each $\ell \le j \le \ell+3$, the corresponding term
    \begin{align*}
        \Big|a_i \cdot \frac{\exp(-b_i^2/2)}{2\pi} \cdot He_{j-2}(b_i)\Big| \le B \cdot ( \eps m d B )^{c_\ell^2/2} &< \frac{\eps_3}{2m} . \\
        \text{ Hence }  \Bignorm{ T_j - \sum_{i \in G} (-1)^j \cdot a_i \cdot He_{j-2}(b_i) \cdot \frac{\exp(\frac{-b_i^2}{2})}{\sqrt{2\pi}} \cdot w_i^{\otimes j}}_F &\le \eps'_3+ m \cdot \frac{\eps_3}{2m} \le \eps_3, 
    \end{align*}
    where the first line follows from our choice of $\eps_3$ and our choice of $c_\ell$. 
    
    Let $m'\coloneqq |G|$. Next we establish that the linear system is well-conditioned for each $\ell \le j \le \ell+3$. For any signs $\xi_i \in \set{\pm 1}~\forall i \in [m]$,  the matrix formed by the vectors $\set{\xi_i^{\ell} w_i^{\otimes \ell}: i \in G}$ has non-negligible least singular value. Moreover from Claim~\ref{claim:khatri-rao} (applied three times), we have for $\ell \le j \le  \ell+3$, we have the matrix formed by columns $\set{ \xi_i^j  w_i^{\otimes j} : i \in G}$ has least singular value $\sing_{m'}\big( (w_i^{\otimes j} : i \in G )\big) \ge \sing_{m}(W^{\otimes \ell})/(2m)^{3/2}$. Suppose $M_j, \widetilde{M}_j \in \R^{d^{j} \times m'}$ be the matrices with the $i$th columns $(\xi_i w_i)^{\otimes j}$ and $\widetilde{w}_i^{\otimes j}$ respectively for $i \in G$. %Moreover since the columns has unit norm, the condition number $\kappa(M_j) \le \sqrt{m}/ \sing_{m'}(M_j)$. 
    Then by Weyl's inequality, we have 
    \begin{align*}
\sing_{m'}(\widetilde{M}_j) &\ge \sing_{m'}(M_j) - \norm{\widetilde{M}_j - M_j}_F \ge \frac{\sing_m(W^{\otimes \ell})}{(2m)^{3/2}}- \sum_{i \in G} \norm{\widetilde{w}_i^{\otimes j}- \xi_i w_i^{\otimes j}}_F \\
&\ge \frac{\sing_m(W^{\otimes \ell})}{(2m)^{3/2}} - m \cdot \eps_3' \ge \frac{\sing_m(W^{\otimes \ell})}{4m^{3/2}}, ~~~\text{ since } \eps_3<\frac{\sing_m(W^{\otimes \ell})}{4m^2} .
    \end{align*}
    
    The target solution to the linear system for each $\ell \le j \le \ell+3$ 
    $$\forall i \in G, ~~\zeta^*_j(i) \coloneqq a_i \cdot (-1)^{j} \xi_i^j\cdot \frac{\exp(\frac{-b_i^2}{2})}{\sqrt{2\pi}} ~He_{j-2}(b_i).$$
    Note that since for any $j$, $\sup_{z \in \R} e^{-z^2/2} He_j(z) \le c'_j$ for some bounded constant $c'_j< \infty$. Now  a standard error analysis of the linear system yields (see e.g., \cite{Bhatia}) we have for all $\ell \le j \le \ell+3$,
    \begin{align*}
        \norm{\zeta_j - \zeta^*_j}_2 & \le \Big(\sing_{m'}(\widetilde{M}_j) \Big)^{-1} \Big( \eps_3 + \norm{M_j - \widetilde{M}_j}_F \norm{\zeta^*_j}_2\Big)\\
        &\le \frac{4m^{3/2}}{\sing_m(W^{\otimes \ell})} \big( \eps_3 + c'_j \cdot B \cdot \sqrt{m'} \eps'_3\big).\\
        \text{Hence, } \forall i \in G,~~ \big| \zeta_j(i) - \zeta^*_j(i) \big| & \le \frac{8m^{3/2}}{\sing_m(W^{\otimes \ell})} \cdot  \eps_3 \le \frac{\eps^4}{2(1+B)^2}.
    \end{align*} 
    since $\eps'_3 \le \tfrac{1}{2}(c'_{\ell+3} B \sqrt{m} ) \eps_3$, and for our choice of $\eps$.  
    
    Finally we can now apply Lemma~\ref{lem:robust:recovscalars} for $\ell \ge 2$ or Lemma~\ref{lem:robust:recovscalars:1} for $\ell=1$ for each of the $i \in G$ separately with $\gamma_j = \zeta_j(i)$ (note that the error is at most $\eps'$ as in Lemmas~\ref{lem:robust:recovscalars} and ~\ref{lem:robust:recovscalars:1}). The output is $\widetilde{a}_i=\tempa, \widetilde{b}_i= z$ and conclude that $|\widetilde{a}_i - a_i|\le \eps$, and $|\widetilde{b}_i-\xi_i b_i| \le \eps$.

    % the linear system has a unique solution for each $j$ with 
    % $$\forall j \in \set{\ell, \ell+1, \ell+2, \ell+3} ~\forall i \in [m], ~\text{ we have } \zeta_j(i)=a_i \cdot (-1)^{j} \cdot \frac{\exp(\frac{-b_i^2}{2})}{\sqrt{2\pi}} ~He_{j-2}(b_i). $$ 
    % Now for each $i \in [m]$, we will use Lemma~\ref{lem:recovscalars} with $\beta=a_i e^{-b_i^2/2}/\sqrt{2\pi}$ and $\gamma_{j-2}=(-1)^j He_{j-2}(b_i)$. Note that $\beta \ne 0$.  Lemma~\ref{lem:recovscalars} shows that $a_i, b_i$ can be recovered for each $i \in [m]$.
\end{proof}

\subsection{Other claims for the robust analysis}

The following claim shows how one can combine some of the activation units output by the regression step to get a ReLU network with at most $|G|+2$ units. 

\begin{claim}\label{claim:consolidate}
Given a function $g(x)$ of the form 
\begin{equation}\label{eq:consol:1}
    g(x)= v^\top x+ c+ \sum_{i \in m'} \alpha_i \sigma(w_i^\top x + b_i) + \alpha'_i \sigma(-w_i^\top x - b_i ),
\end{equation} 
then $g(x)$ can be expressed as a ReLU network with at most $m'+2$ activation units as
\begin{equation}\label{eq:consol:2}
g(x) = \beta_0 \sigma({w_0}^\top x+b_0) - \beta_0 \sigma(-w_0^\top x - b_0)+ \sum_{i=1}^{m'} \beta_i \sigma( w_i^\top x +  b_i  ),\end{equation}
where for each $i \in [m']$, 
%$\xi_{i} = \text{sign}(\alpha_i - \alpha'_i)$, 
$\beta_i = \alpha_i + \alpha'_i$ and $w_0 \in \mathbb{S}^{d-1}, b_0 \in \R, \beta_0 \in \R$ chosen to satisfy 
$\beta_0 w_0 = v-\sum_{i=1}^{m'} \alpha'_i w_i$  and $\beta b_0 = c -\sum_{i-1}^{m'}\alpha'_i b_i$. 
\end{claim}
\begin{proof}
First we note that for any $z \in \R$, $\sigma(z)=\tfrac{1}{2}(|z|+z)$ and $\sigma(-z)=\tfrac{1}{2}(|z|-z)$. Hence we have
\begin{equation}\label{eq:consol:temp}
\gamma \sigma(z)+\gamma' \sigma(-z) = \tfrac{1}{2}(\gamma+\gamma') |z| + \tfrac{1}{2}(\gamma - \gamma') z, \text{ and } z = \sigma(z) - \sigma(-z).    
\end{equation} 
Hence the terms are consolidated by replacing terms of the form $\sigma(w_i^\top x+ b_i)$ and $\sigma(-w_i^\top x - b_i)$ by one ReLU unit so that the coefficient of $|w_i^\top x + b_i|$ match, along with a linear term. All the linear terms are themselves consolidated together, and replaced by a sum of two ReLU units. 
Now, substituing the setting of $\beta_i$ $w_0, b_0$ in \eqref{eq:consol:2} and simplifying, we have
\begin{align*}
    & \beta_0 \sigma({w_0}^\top x+b_0) - \beta_0 \sigma(-w_0^\top x - b_0)+ \sum_{i=1}^{m'} \beta_i \sigma( w_i^\top x +  b_i  ) \\
    =& \beta_0 (w_0^\top x+ b_0) + \sum_{i=1}^{m'} (\alpha_i+\alpha'_i)\cdot \frac{1}{2} \big(|w_i^\top x+b_i|+ (w_i^\top x+b_i) \big)\\
%    =& v^\top x + c  + \sum_{i=1}^{m'} - \alpha_i' (w_i^\top x + b_i) + \big(\frac{\alpha_i + \alpha'_i}{2}\big)  (w_i^\top x+ b_i)+ \big(\frac{\alpha_i + \alpha'_i}{2}\big) |w_i^\top x+ b_i| \\
    =& v^\top x+c+ \sum_{i=1}^{m'} \frac{1}{2}(\alpha_i + \alpha'_i) |w_i^\top x+ b_i|+ \frac{1}{2}(\alpha_i - \alpha'_i) (w_i^\top x + b_i)\\
    =&v^\top x+c+ \sum_{i=1}^{m'} \alpha_i \sigma(w_i^\top x+ b_i)+ \alpha'_i \sigma(-w_i^\top x - b_i) = g(x),
\end{align*}
where the last line follows from \eqref{eq:consol:temp} and \eqref{eq:consol:1}. 
\end{proof}

\end{document}